\documentclass{article} 
\usepackage{iclr2024_conference,times}

\usepackage{hyperref}
\usepackage{url}

\usepackage[noend]{algpseudocode}
\usepackage{algorithm}

\usepackage{todonotes}

\usepackage{microtype}
\usepackage{graphicx}
\usepackage{booktabs} 

\usepackage{caption}
\usepackage{subcaption}

\usepackage{hyperref}

\usepackage{multirow}

\usepackage{amsmath}
\usepackage{amssymb}
\usepackage{mathtools}
\usepackage{amsthm}

\usepackage[capitalize,noabbrev]{cleveref}

 \usepackage{adjustbox}
 \usepackage{amsmath}
 
 \usepackage{listings}
 \usepackage{xcolor}

  \usepackage{graphicx}

\usepackage{caption}

 \usepackage{scalerel,stackengine}
\stackMath
\newcommand\reallywidehat[1]{%
\savestack{\tmpbox}{\stretchto{%
  \scaleto{%
    \scalerel*[\widthof{\ensuremath{#1}}]{\kern.1pt\mathchar"0362\kern.1pt}%
    {\rule{0ex}{\textheight}}
  }{\textheight}%
}{2.4ex}}%
\stackon[-6.9pt]{#1}{\tmpbox}%
}
\makeatletter

\newcommand*\iftodonotes{\if@todonotes@disabled\expandafter\@secondoftwo\else\expandafter\@firstoftwo\fi}   
\makeatother

\newcommand{\nik}[2][]{}
\newcommand{\Nik}[2][]{}

\newcommand{\fran}[2][]{}
\newcommand{\Fran}[2][]{}

\makeatletter
\newtheorem*{rep@theorem}{\rep@title}
\newcommand{\newreptheorem}[2]{%
\newenvironment{rep#1}[1]{%
 \def\rep@title{#2 \ref{##1}}%
 \begin{rep@theorem}}%
 {\end{rep@theorem}}}
\makeatother

\renewcommand{\textcolor}[2]{#2}

\usepackage{amsmath,amsfonts,bm}

\def\eqref#1{equation~\ref{#1}}

\def\1{\bm{1}}

\def\d{\mathrm{d}}

\def\ora{\overrightarrow}
\def\ola{\overleftarrow}

\def\fP{\overrightarrow{\mathbb{P}}}
\def\bP{\overleftarrow{\mathbb{P}}}

\def\fX{{\bm{X}}}

\def\fY{{\bm{Y}}}
\def\bY{{\bm{Y}}}

\def\fZ{{\bm{Z}}}

\def\fW{{\bm{W}}}
\def\bW{{\bm{W}}}

\def\ddb{\overleftarrow{\mathrm{d}}}
\def\ddf{\overrightarrow{\mathrm{d}}}

\def\dd{{\mathrm{d}}}

\def\vx{{\bm{x}}}
\def\vy{{\bm{y}}}
\def\vz{{\bm{z}}}

\def\mW{{\bm{W}}}

\def\mY{{\bm{Y}}}

\DeclareMathAlphabet{\mathsfit}{\encodingdefault}{\sfdefault}{m}{sl}
\SetMathAlphabet{\mathsfit}{bold}{\encodingdefault}{\sfdefault}{bx}{n}

\def\gF{{\mathcal{F}}}

\def\gN{{\mathcal{N}}}

\def\gW{{\mathcal{W}}}

\def\gZ{{\mathcal{Z}}}

\def\sR{{\mathbb{R}}}

\newcommand{\E}{\mathbb{E}}

\def\calN{{\mathcal{N}}}

\let\log\relax
\DeclareMathOperator{\log}{ln}

\def\P{{\mathbb{P}}}

\def\Q{{\mathbb{Q}}}

\def\Rc{{\mathcal{R}}}

\newcommand{\KL}{D_{\mathrm{KL}}}
\newcommand{\Var}{\mathrm{Var}}

\DeclareMathOperator*{\argmin}{arg\,min}

\theoremstyle{plain}

\newtheorem{theorem}{Theorem}[section]
\newtheorem{proposition}[theorem]{Proposition}

\newtheorem{corollary}[theorem]{Corollary}
\theoremstyle{definition}

\newtheorem{framework}{Framework}
\newtheorem{assumption}[theorem]{Assumption}
\newtheorem{remark}{Remark}

\newreptheorem{proposition}{Proposition}

\newreptheorem{corollary}{Corollary}
\newreptheorem{theorem}{Theorem}
\newreptheorem{lemma}{Lemma}
\newreptheorem{observation}{Observation}
\newreptheorem{remark}{Remark}

\newenvironment{frmbis}[1]
  {\renewcommand{\theframework}{\ref*{#1}$'$}%
   \begin{framework}}
  {\end{framework}}

\newenvironment{algocolor}{%
   \setlength{\parindent}{0pt}
   \itshape
    
}{}

\title{Transport meets Variational Inference: \\ Controlled Monte Carlo Diffusions}

\author{Francisco Vargas\textsuperscript{\ensuremath{*}},  Shreyas Padhy\textsuperscript{\ensuremath{*}} \\
University of Cambridge\\
Cambridge, UK \\
\texttt{\{fav25,sp2058\}@cam.ac.uk} 
\And
Denis Blessing \\
KIT \\
Karlsruhe, Germany \\
 \texttt{jl8142@kit.edu}
\And
Nikolas Nüsken\textsuperscript{\ensuremath{*}}  \\
Kings College London \\
London, UK \\ 
\texttt{nik.nuesken@gmx.de}
}

\iclrfinalcopy
\begin{document}

\maketitle

\begin{NoHyper}
\renewcommand{\thefootnote}{\ensuremath{*}}
\footnotetext{Equal contribution.}
\end{NoHyper}

\begin{abstract}
Connecting optimal transport and variational inference, we present a principled and systematic framework for sampling and generative modelling centred around divergences on path space. Our work culminates in the development of \emph{Controlled Monte Carlo Diffusions} for \textcolor{magenta}{sampling and inference}, a score-based annealing technique that crucially adapts both forward and backward dynamics in a diffusion model. \textcolor{magenta}{On the way, we clarify the relationship between the EM-algorithm and iterative proportional fitting (IPF) for Schr{\"o}dinger bridges, providing a conceptual link between fields.} Finally, we show that CMCD has a strong foundation in the Jarzinsky and Crooks identities from statistical physics, and that it convincingly outperforms competing approaches across a wide array of experiments.

\end{abstract}



\section{Introduction}

Optimal transport \citep{villani2009optimal} and
variational inference \citep{blei2017variational} have for a long time been separate fields of research. 
In recent years, many fruitful connections have been established \citep{liu2019understanding}, in particular based on dynamical formulations \citep{tzen2019neural}, and in conjunction with time reversals \citep{huang2021variational,song2020score}.
\textcolor{magenta}{The goal of this paper is twofold: In the first part, we enhance those relationships based on forward and reverse time diffusions, and associated Girsanov transformations, arriving at a unifying framework for generative modeling and sampling. In the second part, we build on this and develop a novel score-based scheme for sampling from unnormalised densities.} 
 To set the stage, we recall a classical approach \citep{kingma2013auto,rezende2015variational} towards generating samples from a target distribution $\mu(\vx)$, which is the goal both in generative modelling and sampling:

\textbf{Generative processes, encoders and decoders.} We consider methodologies which can be implemented via the following generative process,  \vspace{-0.2cm}
\begin{align}
 \vz \sim \nu(\vz), \qquad 
 \vx|\vz \sim p^{\theta}(\vx | \vz), \label{eq:gen_model} \vspace{-0.1cm}
 \end{align}
transforming a sample $\vz \sim \nu(\vz)$ into a sample $\vx \sim \int p^\theta(\vx|\vz) \nu(\mathrm{d}\vz)$. Traditionally,  $\nu(\vz)$ is a simple auxiliary distribution, and the family of transitions $p^{\theta}(\vx | \vz)$ is parameterised flexibly and in such a way that sampling according to (\ref{eq:gen_model}) is tractable. Then we can frame the tasks of generative modelling and sampling as finding transition densities such that the marginal in $\vx$ matches the target distribution, \vspace{-0.29cm}
 \begin{equation}
 \label{eq:x marginal}
 \mu(\vx) = \int p^{\theta}(\vx | \vz) \nu(\mathrm{d}\vz).  \vspace{-0.1cm}
 \end{equation} 
\noindent To learn such a transition, it is helpful to introduce a reversed process \vspace{-0.18cm}
\begin{align} 
\vx \sim \mu(\vx), \qquad 
\vz|\vx \sim q^{\phi}(\vz | \vx), \label{eq:encoder} 
\end{align} 
relying on an appropriately parameterised backward transition $q^\phi(\vz|\vx)$. We will say that (\ref{eq:gen_model}) and  (\ref{eq:encoder}) are \emph{reversals of each other} in the case when their joint distributions coincide, that is, when \vspace{-0.1cm}
 \begin{equation}
 \label{eq:reversal}
q^\phi(\vz|\vx)\mu(\vx) = p^\theta(\vx|\vz) \nu(\vz). \vspace{-0.1cm}
\end{equation} To appreciate the significance of (\ref{eq:encoder}), notice that if (\ref{eq:reversal}) holds, then (\ref{eq:x marginal}) is implied by integrating both sides with respect to $\vz$. Building on this observation, it is natural to define the loss function \vspace{-0.2cm}
 \begin{equation} \label{eq:abstract loss} 	\mathcal{L}_D(\phi,\theta) := D \left( q^\phi (\vz|\vx)\mu(\vx)\big|\big| p^{\theta}(\vx | \vz)\nu(\vz) \right),
 \end{equation}
where $D$ is a divergence\footnote{As usual, divergences are characterised by the requirement that $D(\alpha \big|\big|\beta) \ge  0$, with equality iff $\alpha = \beta$.} between distributions yet to be specified. Along the lines of \citet{bengio2021gflownet,sohl2015deep,wu2020stochastic,liu2022let}, we have now laid the foundations for algorithmic approaches that aim at sampling from $\mu(\vx)$ by minimising $\mathcal{L}_D(\phi,\theta)$:
 \begin{framework}\label{prop:framework}
 Let $D$ be an arbitrary divergence, and assume that $\mathcal{L}_D(\phi,\theta) = 0$. Then we have \vspace{-0.2cm}
 \begin{equation}
 \label{eq:coupling}
\mu(\vx) \!= \!\!\int \!\!p^{\theta}(\vx | \vz) \nu(\mathrm{d}\vz)\; \quad  \text{and} \quad \; \nu(\vz)\!\! = \!\!\int \!\!q^{\phi}(\vz | \vx) \mu(\mathrm{d}\vx),
\end{equation}
 that is, $\nu(\vz)$ is transformed into $\mu(\vx)$ by $p^\theta(\vx|\vz)$, and $\mu(\vx)$ is transformed into $\nu(\vz)$ by $q^\phi(\vz|\vx)$. 
 \end{framework}
 \textcolor{magenta}{ \textbf{The sampling problem.} Let $\nu$ denote a probability density function on $\sR^d$ of the form ${\nu(\vz)=\frac{\hat{\nu}(\vz)}{Z}, \quad Z=\int_{\mathbb{R}^d} \hat{\nu}(\vz) \mathrm{d} \vz, } $}
\textcolor{magenta}{where $\hat{\nu}: \mathbb{R}^d \rightarrow \mathbb{R}^{+}$can be differentiated and evaluated pointwise but the normalizing constant $Z$ is intractable. We are interested in both estimating $Z$ and obtaining approximate samples from $\nu$ given we can sample from a more tractable density $\mu$.} 
\textcolor{magenta}{Framework \ref{prop:framework} provides us with an objective to tackle the sampling problem as once $\mathcal{L}_D(\phi,\theta)=0$,   we can generate samples from $\nu(\vz)$ via the variational distribution $q^\phi(\vz|\vx)$.}
 \textcolor{magenta}{Through variational inference and optimal transport, we discuss relationships to classical methods as well as shortcomings:}
 
 \textbf{KL-divergence, ELBO and variational inference.} Choosing $D = \KL$ in (\ref{eq:abstract loss}), variational inference  (VI) and latent variable model based approaches  \citep{dempster1977maximum,blei2017variational,kingma2013auto} can elegantly be placed within Framework \ref{prop:framework}. Indeed, direct computation (see Appendix \ref{app:vi}) shows that $\mathcal{L}_{\KL}(\phi,\theta) = - \mathbb{E}_{\vx \sim \mu(\vx) }[\mathrm{ELBO}_x(\phi,\theta)] + \mathbb{E}_{\vx \sim \mu(\vx)}[\log \mu (x)]$, 
 so that minimising $\mathcal{L}_{\KL}(\phi,\theta)$ is equivalent to maximising the expected evidence lower bound (ELBO), also known as the negative free energy \citep{blei2017variational}. This derivation is alternative to the standard approach via maximum likelihood and convex duality (or
 Jensen's inequality) \citep[Section 2.2]{kingma2021variational}, and directly accomodates various modifications by  replacing the $\KL$-divergence (see Appendix \ref{app:vi}).

 	
\textbf{Couplings, (optimal) transport and nonuniqueness.}
Assuming (\ref{eq:reversal}) holds, it is natural to define the joint distribution $\pi(\vx,\vz) :=  q^\phi(\vz|\vx)\mu(\vx) = p^\theta(\vx|\vz) \nu(\vz)$, which is a coupling between $\mu(\vx)$ and $\nu(\vz)$. Viewed from this angle, the set of minimisers of $\mathcal{L}(\phi,\theta)$ stands in one-to-one correspondence with the set of couplings between $\mu(\vx)$ and $\nu(\vz)$, provided that the parameterisations are chosen flexibly enough. Under the latter  assumption, the objective in (\ref{eq:abstract loss}) admits an infinite number of minimisers, rendering algorithmic approaches solely based on Framework \ref{prop:framework} potentially unstable and their output hard to interpret. In the language of \emph{optimal transport} \citep{villani2003topics}, minimising $\mathcal{L}(\phi,\theta)$  enforces the marginal (\emph{`transport'}) constraints in (\ref{eq:coupling})  without a selection principle based on an appropriate cost function (\emph{`optimal'}). 

Methods such as VAEs \citep{kingma2013auto} parameterise $p^\theta(\vx|\vz)$ and $q^\phi(\vz|\vx)$ with a restricted family of distributions (such as Gaussians), thus restricting the set of couplings. Expectation maximisation (EM) minimises $\mathcal{L}(\phi,\theta)$ in a component-wise fashion, resolving nonquniqueness in a procedural manner (see Section \ref{sec:IPF}). Common diffusion models fix either $p^\theta(\vx|\vz)$ or $q^\phi(\vz|\vx)$, and thus select a coupling (Section \ref{sec:SDEs}). In this paper, we argue that the full potential of diffusion models can be unleashed by training the  forward and backward processes at the same time, but appropriate modifications that resolve the nonuniqueness inherent in Framework \ref{prop:framework} need to be imposed. To develop principled approaches towards this, we proceed as follows:

\textbf{Outline and contributions.} In Section \ref{sec:VAE-SDE} we recall hierarchical VAEs \citep{rezende2014stochastic} and, following \citet{tzen2019neural}, proceed to the infinite-depth limit described by the SDEs in (\ref{eq:SDEs}). Readers more familiar with VI and discrete time might want to take the development in Section \ref{sec:hierVAE} as an explanation of (\ref{eq:SDEs}); readers with background in stochastic analysis might take Framework \ref{frm2} as their starting point. In Proposition \ref{prop:RND} we provide a generalised form of the Girsanov theorem for forward-reverse time SDEs, crucially incorporating the choice of a reference process that allows us to reason about sampling and generation in a systematic and principled way. We demonstrate that a range of widely used approaches, such as score-based diffusions and path integral samplers, among others, are special cases of our unifying framework (Section \ref{sec:SDEs}). \textcolor{magenta}{Similarly in Section \ref{sec:schr} we unify optimal transport (OT) and VI under our framework by establishing a correspondence between expectation-maximisation (EM) and iterative proportional fitting (IPF)}. Going further, we show that this framework allows us to derive new methods:

\textcolor{magenta}{ In Section \ref{sec:annealing}, we derive a novel score-based annealed flow technique, the Controlled Monte Carlo Diffusion (CMCD) sampler, and show that it may be viewed as an infinitesimal analogue of the method from Section \ref{sec:schr}. Finally, we connect CMCD to the foundational identities by Crooks and Jarzynki in statistical physics, and show that it empirically outperforms a range of state-of-the-art inference methods in sampling and estimating normalizing constants (Section \ref{sec:experiments}).
}



\section{From hierarchical VAEs to forward-reverse time diffusions}
\label{sec:VAE-SDE}
  
\subsection{Hierarchical VAEs \citep{rezende2014stochastic}}
\label{sec:hierVAE} 
A particularly flexible choice of implicitly parameterising $p^\theta(\vx|\vz)$ and $q^\phi(\vz|\vx)$ can be achieved via a hierarchical model with intermediate latents: We identify $\vx =: \vy_0$ and $\vz =: \vy_L$ with the `endpoints' of the layered augmentation  $(\vy_0, \vy_1 ,\ldots,\vy_{L-1},\vy_L) =: \vy_{0:L}$, and define\vspace{-0.27cm}
 \begin{align}
 \label{eq:full conditionals}
 q^\phi(\vy_L, \ldots, \vy_1 | \vy_0) := \prod_{l=1}^{L} q^{\phi_{l-1}}(\vy_l | \vy_{l-1}), \qquad p^\theta(\vy_0,\ldots, \vy_{L-1} | \vy_L) := \prod_{l=1}^{L} p^{\theta_l}(\vy_{l-1} | \vy_{l}),\vspace{-0.2cm}
 \end{align}
so that $q^\phi(\vz|\vx)$ and $p^\theta(\vx|\vz)$ can be obtained from (\ref{eq:full conditionals}) by marginalising over the auxiliary variables $\vy_1,\ldots,\vy_{L-1}$. Here, $\phi = (\phi_0,\ldots, \phi_{L-1})$ and $ \theta = (\theta_1,\ldots,\theta_{L})$ refer to sets of parameters to be specified in more detail below. Further introducing notation, we write $q^{\mu,\phi}(\vy_{0:L}) := q^\phi(\vy_{1:L}|\vy_0) \mu(\vy_0)$ as well as $p^{\nu,\theta}(\vy_{0:L}) := p^\theta(\vy_{0:L-1}|\vy_L) \nu(\vy_L)$ and think of those implied joint distributions as emanating from $\mu(\vx) = \mu(\vy_0)$ and $\nu(\vz) = \nu(\vy_L)$, respectively, moving `forwards' or `backwards' according to the specific choices for $\phi$ and $\theta$. In the regime when $L$ is large, the models in (\ref{eq:full conditionals}) are very expressive, even if the intermediate transition kernels are parameterised in a simple manner.
 We hence proceed by assuming Gaussian distributions,\vspace{-0.1cm}
\begin{align}
\label{eq:interm trans}
&q^{\phi_{l-1}}\!(\vy_l |\vy_{l-1}\!) \!= \!\gN(\vy_l | \vy_{l-1} \!\!+\!  \delta a^\phi_{l-1}(\vy_{l-1}),  \delta \sigma^2 \!I), \;\;p^{\theta_l}\!(\vy_{l-1} |\vy_{l}\!) \!= \!\gN(\vy_{l-1}| \vy_{l} \!\!+\! \delta b^\theta_l(\vy_{l}), \delta \sigma^2 \!I),
\end{align}
where $\sigma > 0$ controls the standard deviation, and $\delta>0$ is a small parameter, anticipating the limits $L \rightarrow \infty$, $\delta \rightarrow 0$ to be taken in Section 
\ref{sec:SDEs} below. The vector fields $a^\phi_l(\vy_l)$ and $b_l^\theta(\vy_l)$ introduced in (\ref{eq:interm trans}) should be thought of as parameterised by $\phi$ and $\theta$, but we will henceforth suppress this for brevity. 

The models (\ref{eq:full conditionals})-(\ref{eq:interm trans}) could equivalently be defined via the Markov chains
\vspace{-0.2cm}
 \begin{subequations}
 \label{eq:Markov chains}
 \begin{align}
 \label{eq:forward chain}
\vy_{l+1} &= \vy_l + \delta a_l(\vy_l) + \sqrt{\delta} \sigma \xi_l, \qquad \vy_0 \sim \mu \implies \vy_{0:L} \sim q^{\mu,\phi}(\vy_{0:L}),
  \end{align} \vspace{-0.7cm}
   \begin{align}
 \label{eq:backward chain}
\vy_{l-1} = \vy_l + \delta b_l(\vy_l) + \sqrt{\delta} \sigma \xi_l, \qquad \vy_L \sim \nu \implies \vy_{0:L} \sim p^{\nu,\theta}(\vy_{0:L}),
 \end{align}
 \end{subequations}
where $(\xi_l)_{l=1}^L$ is an iid sequence of standard Gaussian random variables. As indicated, the forward process in (\ref{eq:forward chain}) may serve to define the distribution $q^{\mu,\phi}(\vy_{0:L})$, whilst the backward process in (\ref{eq:backward chain}) induces $p^{\nu,\theta}(\vy_{0:L})$. Note that the transition densities $p^\theta(\vx|\vz)$ and $q^\phi(\vz|\vx)$ obtained as the marginals of (\ref{eq:full conditionals}) will in general not be available in closed form. However, generalising slightly from Framework \ref{prop:framework}, we may set out to minimise the extended loss \vspace{-0.1cm}
\begin{equation}
\mathcal{L}^{\mathrm{ext}}_D(\phi,\theta) = D(q^{\mu,\phi}(\vy_{0:L}) || p^{\nu,\theta}(\vy_{0:L})),    
\end{equation}
where $D$ refers to a divergence on the `discrete path space' $\{\vy_{0:L}\}$. Clearly, $\mathcal{L}_D^\mathrm{ext}(\phi,\theta) = 0$ still implies (\ref{eq:coupling}), but is no longer equivalent. More specifically, in the case when $D = \KL$, the data processing inequality yields\vspace{-0.15cm}
 \begin{align}
 \label{eq:KL bound}
 \KL(q^{\mu,\phi}(\vy_{0:L}) || p^{\nu,\theta}(\vy_{0:L}))  \geq \KL \left( q^\phi (\vz|\vx)\mu(\vx)\big|\big| p^{\theta}(\vx | \vz)\nu(\vz) \right),
 \end{align}
so that $\mathcal{L}^{\mathrm{ext}}_{\KL}(\phi,\theta)$ provides an upper bound for $\mathcal{L}_{\KL}(\phi,\theta)$ as defined in (\ref{eq:abstract loss}). 
 	
\subsection{Diffusion models --  hierarchical VAEs in the infinite depth limit}
\label{sec:SDEs}
Here we take inspiration from Section \ref{sec:hierVAE} and \citet{tzen2019neural,li2020scalable,huang2021variational} to investigate the $L \rightarrow \infty$ limit, using stochastic differential equations (SDEs). To this end, we think of $l=0,\ldots,L$ as discrete instances in a fixed time interval $[0,T]$, equidistant with time step $\delta$, that is, we set $\delta = TL^{-1}$. The discrete paths $\vy_{0:L}$ give rise to continuous paths $(\fY_t)_{0 \le t \le T} \in C([0,T];\mathbb{R}^d)$ by setting $\fY_{\delta l} = \vy_l$ and linearly interpolating $\fY_{\delta l}$ and $\fY_{\delta(l+1)}$. To complete the set-up, we think of $a^\phi = (a^\phi_0,\ldots,a^\phi_{L-1})$ and $b^\theta = (b_1^\theta,\ldots, b^\theta_L)$ in (\ref{eq:interm trans}) as arising from time-dependent vector fields $a,b \in C^{\infty}([0,T]\times \mathbb{R}^d;\mathbb{R}^d)$ via $a^\phi_l(\vy_l) = a_{t\delta^{-1}}(\fY_{\delta l})$ and $b^\theta_l(\vy_l) = b_{t\delta^{-1}}(\fY_{\delta l})$. 

Taking the limit $\delta \rightarrow 0$, while keeping $T > 0$ fixed, transforms the Markov chains in (\ref{eq:Markov chains}) into continuous-time dynamics described by the SDEs \citep{tzen2019neural}
\vspace{-0.1cm}
\begin{subequations}
\label{eq:SDEs}
\begin{align}
\dd \fY_t &= a_t(\fY_t) \,\dd t + \sigma \ddf \fW_t, \quad \fY_0 \sim  \mu  \implies (\fY_t)_{0 \le t \le T} \sim \mathbb{Q}^{\mu,a} \equiv \fP^{\mu,a},
\label{eq:fwd_sde}
\end{align}\vspace{-0.5cm}
\begin{align}
\dd \bY_t &= b_t(\bY_t) \,\dd t + \sigma\ddb \bW_t, \quad \bY_T \sim  \nu \implies (\fY_t)_{0 \le t \le T} \sim \mathbb{P}^{\nu,b} \equiv \bP^{\nu,b}, 
\label{eq:back_sde}
\end{align}
\end{subequations}
where $\ddf$ and $\ddb$ denote forward and backward It{\^o} integration (see Appendix \ref{app:stochana} for more details and remarks on the notation), and $(\fW_t)_{0 \le t \le T}$ is a standard Brownian motion. In complete analogy with (\ref{eq:Markov chains}), the SDEs in (\ref{eq:SDEs}) induce the distributions $\mathbb{Q}^{\mu,a}$ and $\mathbb{P}^{\nu,b}$ on the path space $C([0,T];\mathbb{R}^d)$. Relating back to the discussion in the introduction, note that we maintain the relations $\fY_0 = \vx$ and $\fY_T = \vz$, and the transitions are parameterised by the vector fields $a,b$, in the sense that $p^\theta(\vx|\vz) = \P^{\nu,b^\theta}_0(\vx | \fY_T = \vz) = {\P}^{\delta_{\vz},b^\theta}_0(\vx)$ and $q^\phi(\vz|\vx) = \Q^{\mu,a^\phi}_T(\vz | \fY_0 = \vx) = {\Q}^{\delta_{\vx},a^\phi}_T(\vz)$.

The following well-known result \citep{anderson1982reverse,nelson1967dynamical} allows us to relate forward and backward path measures via a local (score-matching) condition for the reversal relation in (\ref{eq:reversal}). \footnote{The global condition $\ora{\P}^{\mu,a} \!= \! \ola{\P}^{\nu,b}$ is captured by the local condition (\ref{eq:score}) due to (\ref{eq:SDEs})'s Markovian nature.}

\begin{proposition}[Nelson's relation]
\label{prop:Nelson}
For $\mu$ and $a$ of sufficient regularity, denote the time-marginals of the corresponding path measure by $\ora{\P}^{\mu,a}_t =:\rho^{\mu,a}_t$. Then $\ora{\P}^{\mu,a} = \ola{\P}^{\nu,b}$ if and only if\vspace{-0.1cm}
\begin{align}
\label{eq:score}\nu = \ora\P^{\mu,a}_T \qquad \text{and}\qquad b_t = a_t - \sigma^2 \nabla \log \rho^{\mu,a}_t, \qquad \text{for all  }  t \in (0,T].  
\end{align}
\end{proposition}

\begin{remark}
A similarly clean characterisation of equality between forward and backward path measures is not available for the discrete-time setting as presented in (\ref{eq:Markov chains}). In particular, Gaussianity of the intermediate transitions is not preserved under time-reversal. 
\end{remark} \vspace{-0.15cm}
A recurring theme in this work and related literature is the interplay between the score-matching condition in (\ref{eq:score}) and the global condition $D(\ora{\P}^{\mu,a}| \ola{\P}^{\nu,b}) = 0$, invoking Framework \ref{prop:framework}. To enable calculations involving the latter, we will rely on the following result:
\begin{proposition}[forward-backward Radon-Nikodym derivatives]
\label{prop:RND}
Let $\ora{\P}^{\Gamma_0,\gamma^+} = \ola{\P}^{\Gamma_T,\gamma^-}$ be a reference path measure (that is, $\Gamma_0$, $\Gamma_T$ and $\gamma^{\pm}$  define diffusions as in (\ref{eq:SDEs}) and are related as in Proposition \ref{prop:Nelson}), absolutely continuous with respect to both $\ora{\P}^{\mu,a}$ and $\ola{\P}^{\nu,b}$. Then, $\ora{\P}^{\mu,a}$-almost surely, the corresponding Radon-Nikodym derivative (RND) can be expressed as follows, \vspace{-0.13cm}
\begin{subequations}
\label{eq:RND fb}
\begin{align}
\label{eq:boundary terms}
\log \left(\frac{\d\fP^{\mu,a}}{\d\bP^{\nu,b}}\right)(\fY)  & = \log\left( \frac{\d \mu}{\d \Gamma_0}\right)(\fY_0) - \log \left(\frac{\d \nu}{\d \Gamma_T}\right)(\fY_T)
\\
\label{eq:forward path integral}
& + \tfrac{1}{\sigma^2}\!\int_0^T \!\!\left(a_t -  \gamma^+_t\right)(\fY_t) \!\cdot \!\left( \ora{\d} \fY_t - \tfrac{1}{2}\left( a_t + \gamma^+_t \right)(\fY_t) \, \d t\right)
\\
\label{eq:backward path integral}
&
-  \tfrac{1}{\sigma^2}\!\int_0^T\!\! \left(b_t - \gamma^-_t\right)(\fY_t) \!\cdot \!\left( \ola{\d} \fY_t - \tfrac{1}{2}\left( b_t + \gamma^-_t \right)(\fY_t) \, \d t\right). 
\end{align}
\end{subequations}
\end{proposition}
\vspace{-0.45cm}
\begin{proof}
The proof relies on Girsanov's theorem \citep{ustunel2013transformation}, using the reference to relate the forward and backward processes.
For details, see Appendix \ref{app:proof}.    
\end{proof}
\vspace{-0.34cm}
\begin{remark}[Role of the reference process] According to Proposition \ref{prop:RND}, the Radon-Nikodym derivative between $\ora{\P}^{\mu,a}$ and $\ola{\P}^{\nu,b}$ can be decomposed into boundary terms (\ref{eq:boundary terms}), as well as forward and backward path integrals (\ref{eq:forward path integral}) and (\ref{eq:backward path integral}). Since the left-hand side of (\ref{eq:boundary terms}) does not depend on the reference $\Gamma_{0,T}$, $\gamma^{\pm}$, the expressions in (\ref{eq:RND fb}) are in principle equivalent for all choices of reference. The freedom in $\Gamma_{0,T}$ and $\gamma^{\pm}$ allows us to `reweight' between (\ref{eq:boundary terms}), (\ref{eq:forward path integral}) and (\ref{eq:backward path integral}), or even cancel terms.
A canonical choice is the Lebesgue measure for $\Gamma_0$ and $\Gamma_T$, and $\gamma^{\pm} = 0$, see Appendix \ref{app:discussion}. 
\end{remark}
\vspace{-0.2cm}
\begin{remark}[Discretisation and conversion formulae]
\label{rem:conversion} The distinction between forward and backward integration in (\ref{eq:RND fb}) is related to the time points at which the integrands $\left(a_t - \gamma^+_t\right)(\fY_t)$ and $\left(b_t - \gamma^-_t\right)(\fY_t)$ would be evaluated in discrete-time approximations, e.g., \vspace{-0.1cm}
\begin{align}
\nonumber
\int_0^T \!\!a_t(\fY_t) \!\cdot\! \ora{\d} \fY_t \approx \sum_i a_{t_i}(\fY_{t_i}) \!\cdot \! (\fY_{t_{i+1}} - \fY_{t_i}), \;\;\int_0^T\!\! a_t(\fY_t) \!\cdot \!\ola{\d} \fY_t \!\approx\! \sum_i a_{t_{i+1}}(\fY_{t_{i+1}})\!\cdot\!  (\fY_{t_{i+1}} - \fY_{t_i}).
\nonumber
\end{align}\vspace{-0.15cm}
Alternatively, forward and backward integrals can be transformed into each other using the conversion \vspace{-0.15cm}
\begin{equation}
\label{eq:conversion}
\int_0^T \!\!a_t(\fY_t) \cdot \ora{\d} \fY_t = \int_0^T\!\! a_t(\fY_t) \cdot \ola{\d} \fY_t - \sigma^2\!\! \int_0^T \!\!(\nabla \cdot a_t)(\fY_t) \, \mathrm{d}t. 
\end{equation}
We refer to \citet{kunita2019stochastic} and Appendix \ref{app:stochana} for further details. In passing, we note that (\ref{eq:conversion}) allows us to eliminate the Hutchinson estimator \citep{hutchinson1989stochastic}from a variety of common  score-matching objectives, potentially reducing the variance of gradient estimators, see Appendix \ref{app:discussion}. 
\end{remark}
Framework \ref{prop:framework} can be translated into the setting of (\ref{eq:SDEs}), noting that (\ref{eq:KL bound}) continues to hold with appropriate modifications:\vspace{-0.1cm}
\begin{frmbis}{prop:framework}\label{frm2}
For a divergence $D$ on path space, minimise $D(\ora{\P}^{\mu,a}|\ola{\P}^{\nu,b})$. If $D(\ora{\P}^{\mu,a}|\ola{\P}^{\nu,b})=0$, then (\ref{eq:fwd_sde}) transports $\mu$ to $\nu$, and (\ref{eq:back_sde}) transports $\nu$ to $\mu$. \footnote{Concurrently \cite{richter2023improved} propose an akin framework to ours.}
\end{frmbis} \vspace{-0.1cm}
At optimality, $D(\ora{\P}^{\mu,a}|\ola{\P}^{\nu,b})=0$, Proposition \ref{prop:Nelson} allows us to obtain the scores associated to the learned diffusion via $\sigma^2 \nabla \log \rho_t^{\mu,a} = a_t - b_t$. In this way, Framework \ref{frm2} is closely connected to (and in some ways extends) score-matching ideas \citep{song2019generative,song2020score}. 
Indeed, recent approaches towards generative modeling and sampling can be recovered from Framework \ref{frm2} by making specific choices for the divergence $D$, the  parameterisations for $a$ and $b$, as well as for the reference diffusion $\ora{\P}^{\Gamma_0,\gamma^+} = \ola{\P}^{\Gamma_T,\gamma^-}$ in Proposition \ref{prop:RND}:

\textbf{Score-based generative modeling:} Letting $\mu$ be the target 
 and fixing the forward drift $a_t$, and, motivated by Proposition \ref{prop:Nelson}, parameterising the backward drift as $b_t = a_t - s_t$, we recover the SGM objectives in \citet{hyvarinen2005estimation, song2019generative, song2020score} from $D = \KL$; when $\ora{\P}^{\mu,a} = \ola{\P}^{\nu,b}$, the variable drift component $s_t$ will represent the score $\sigma^2 \nabla \log \rho^{\mu,a}_t$. Modifications can be obtained from the conversion formula (\ref{eq:conversion}), see Appendix \ref{sec:SGM}.
 
\textbf{Score-based sampling -- ergodic drift:} In this setting, $\nu$ becomes the target and we fix $b_t$ to be the drift of an ergodic (backward) process. Then choosing $\Gamma_{0,T} = \mu$, $\gamma^{\pm} = b$ allows us to recover the approaches in \citet{vargas2023denoising,berner2022optimal}. Possible generalisations based on Framework \ref{frm2} include IWAE-type objectives, see Appendix \ref{sec:sampling}. 

\textbf{Score-based sampling -- F\"ollmer drift:} Finally choosing $b_t(x) = x/t$ we recover F\"ollmer sampling \citep[Appendix \ref{sec:Foellmer};][]{follmer1984entropy, vargas2021bayesian, zhang2021path,huang2021schrodinger}.




\section{Learning forward and backward transitions simultaneously}
\label{sec:applications}

Recall from the introduction that complete flexibility in $a$ and $b$ will render the minima of $D(\ora{\P}^{\mu,a}|\ola{\P}^{\nu,b})$ highly nonunique. Furthermore, the approaches surveyed at the end of the previous section circumvent this problem by fixing either $\ora{\P}^{\mu,a}$ or $\ola{\P}^{\nu,b}$. However, to leverage the full power of diffusion models, both $\ora{\P}^{\mu,a}$ or $\ola{\P}^{\nu,b}$ should be adapted to the problem at hand. In this section, we explore models of this kind, by imposing additional constraints on $a$ and $b$. We end this section by presenting our new CMCD sampler connecting it to prior methodology within VI \citep{doucet2022annealed,geffner2023langevin,papamakarios2017masked} and OT where we can view CMCD as an instance of entropy regularised OT in the infinite constraint limit \citep{bernton2019schr}.


\subsection{Connection to Entropic optimal transport}

\label{sec:schr}

One way of selecting a particular transition between $\mu$ and $\nu$ is by imposing an entropic penalty, encouraging the dynamics to stay close to a prescribed, oftentimes physically or biologically motivated, reference process. Using the notation employed in Framework \ref{prop:framework}, the \emph{static} Schr{\"o}dinger problem \citep{schrodinger1931uber,leonard2013survey} is given by \vspace{-0.15cm}
\begin{align}
\label{eq:static sch}
\pi^*(\vx,\vz) \in \argmin_{\pi(\vx,\vz)} \Big\{ \KL(\pi(\vx,\vz) || r(\vx,\vz)):  \pi_\vx(\vx) = \mu(\vx),  \pi_\vz(\vz) = \nu(\vz)  \Big\},  
\end{align}
where $r(\vx,\vz)$ is the Schr{\"o}dinger prior encoding additional domain-specific information. In an analogous way, we can introduce a regulariser to the path-space approach of Framework 1' to obtain the dynamic Schr\"odinger problem \vspace{-0.35cm}
\begin{equation}
\label{eq:dynamic sch}
\P^*\!\! \in\!\argmin_{\ora{\P}^{\mu,a}_{T} \;=\; \nu}\mathbb{E}_{\fY \sim \ora{\P}^{\mu,a}} \!\!\left[ \tfrac{1}{2 \sigma^2}\!\int_0^T \!\!\Vert a_t - f_t\Vert^2(\fY_t)  \,\d t \right]\!,     
\end{equation}
that is, the driving vector field $a_t$ determining $\P^*$ should be chosen in such a way that (i), the corresponding diffusion transitions from $\mu$ to $\nu$, and (ii), among such diffusions, the vector field $a_t$ remains close to the prescribed vector field $f_t$, in mean square sense. Under mild conditions, the solutions to (\ref{eq:static sch}) and (\ref{eq:dynamic sch}) exist and are unique. Further, the static and dynamic viewpoints are related through a mixture-of-bridges construction (assuming that the conditionals $r(\vz|\vx)$ correspond to the transitions induced by $f_t$), see \citep[Section 2]{leonard2013survey}.



\textbf{Iterative proportional fitting (IPF) and the EM algorithm.} 
\label{sec:IPF}
It is well known that approximate solutions for $\pi^*(\vx,\vz)$ and $\P^*$ can be obtained using alternating $\KL$-projections, keeping one of the marginals fixed in each iteration: 
Under mild conditions, the sequence defined by
\begin{subequations}
\label{eq:IPF}
\begin{align}
\label{eq:ipf1}
\pi^{2n+1}(\vx,\vz ) &=   \argmin_{\pi(\vx,\vz)} \left\{ \KL(\pi(\vx,\vz) || \pi^{2n}(\vx,\vz)): \,\, \pi_\vx (\vx)  = \mu(\vx) \right\},   
\\
\label{eq:ipf2}
\pi^{2n+2}(\vx,\vz)  & = \argmin_{\pi(\vx,\vz)} \left\{ \KL(\pi(\vx,\vz) || \pi^{2n+1}(\vx,\vz)): \,\, \pi_\vz(\vz)  = \nu(\vz) \right\}, \qquad n \ge 0, 
\end{align}
\end{subequations}
with initialisation $\pi^0(\vx,\vz) = r(\vx,\vz)$,  converges to $\pi^*(\vx,\vz)$ as $n \rightarrow \infty$ \citep{de2021diffusion}, and this procedure is commonly referred to as iterative proportional fitting (IPF) \citep{fortet1940resolution,kullback1968probability,ruschendorf1995convergence} or Sinkhorn updates \citep{cuturi2013sinkhorn}. IPF can straightforwardly be modified to the path space setting of (\ref{eq:dynamic sch}), and the resulting updates coincide with the F{\"o}llmer drift updates discussed in Section \ref{sec:Foellmer}, see \citep{vargas2021solving} and Appendix \ref{app:em_init}.


\textcolor{magenta}{To further demonstrate the coverage of our framework, we establish a connection between IPF and expectation-maximisation (EM) \citep{dempster1977maximum}, originally devised for finding maximum likelihood estimates in models with latent (or hidden) variables. According to \citet{neal1998view}, the EM-algorithm can be described in the setting from the introduction, and written in the form}
\begin{align}
\label{eq:EM} \theta_{n+1} = \argmin_\theta \mathcal{L}_{\KL}(\phi_{n},\theta), \qquad \phi_{n+1} = \argmin_\phi \mathcal{L}_{\KL} (\phi,\theta_{n+1}),  
\end{align}
with $\mathcal{L}_{\KL}$ defined as in (\ref{eq:abstract loss}). If the initialisations are matched appropriately, the following result establishes an exact correspondence between the IPF updates in (\ref{eq:IPF}) and the EM updates in (\ref{eq:EM}):
\begin{proposition}[EM $\iff$ IPF]
\label{prop:em ipf}
Assume that the transition densities $p^\theta(\vx|\vz)$ and  $q^\phi(\vz|\vx)$  are parameterised with perfect  flexibility,\footnote{In precise terms, we assume that for any transition densities $p(\vx|\vz)$ and $q(\vz|\vx)$, there exist $\theta_*$ and $\phi_*$ such that $p(\vx|\vz) = p^{\theta_*}(\vx|\vz)$ and $q(\vx|\vz) = q^{\phi_*}(\vx|\vz)$. 
} and furthermore that the EM-scheme (\ref{eq:EM}) is initialised at $\phi_0$ in such a way that $q^{\phi_0}(\vz|\vx) = r(\vz|\vx)$. Then the IPF iterations in (\ref{eq:IPF}) agree with the EM iterations in (\ref{eq:EM}) for all $n \ge 1$, in the sense that 
\begin{align}
\label{eq:em ipf}
\pi^n(\vx,\vz) = q^{\phi_{(n-1)/2}}(\vz|\vx)\mu(\vx), \quad \text{for} \,\,  n \, \text{odd}, \quad 
\pi^n(\vx,\vz) = p^{\theta_{n/2}}(\vx|\vz)\nu(\vz), \quad \text{for} \,\, n \, \text{even}.  
\end{align}
\end{proposition}
\textcolor{blue}{From the proof (Appenix \ref{app:proof}), it is clear that flexibility of parameterisations is crucial, and thus $\text{EM} \iff \text{IPF}$ fails for classical VAEs, but holds up to a negligle error for the SDE-parameterisations from Section \ref{sec:SDEs}, see also \citet{liu2022let}.}
Under this assumption, the key observation is that replacing forward-$\KL$ by reverse-$\KL$ in one or both of (\ref{eq:ipf1}) and (\ref{eq:ipf2}) does not -- in theory -- change the sequence of minimisers. 

\textcolor{magenta}{In practice favoring the EM objectives over IPF can offer an advantage as optimizing with respect to forward-$\KL$ and backward-$\KL$ encourages moment-matching and mode-seeking behavior, respectively, and so an alternating scheme as defined in (\ref{eq:EM}) might present a suitable compromise over optimizing a single direction of $\KL$'s, empirical exploration is left for future work.}

\textcolor{magenta}{Whilst EM and IPF might seem appealing for learning a sampler they both require sequentially solving a series of minimization problems, which we can only solve approximately; this is not only slow but also causes a sequential accumulation of errors arising from each iterate \citep{vargas2021solving,fernandes2021shooting}. In order to address both issues we will present a novel approach (CMCD) that similarly to IPF learns both the forward and backward processes whilst preserving the desired uniqueness property. However, in contrast to IPF it does so in an end-to-end fashion and performs updates simultaneously. As an alternative in Appendix \ref{app:HJB} we also discuss a regularised IPF objective and leave further empirical exploration for future work.
}

\subsection{Score-based annealing: the Controlled Monte Carlo Diffusion sampler}
\label{sec:annealing}

In this section, we fix a prescribed curve of distributions $(\pi_t)_{t \in [0,T]}$, whose scores $\nabla \log \pi_t$ (and unnormalised densities $\hat{\pi}_t$) are assumed to be available in tractable form; this is the scenario typically encountered in annealed importance sampling (IS) and related approaches towards computing posterior expectations \citep{neal2001annealed,reich2011dynamical,heng2015gibbs,heng2020controlled,arbel2021annealed,doucet2022score}. The \emph{Controlled Monte Carlo Diffusion sampler} (CMCD) learns the vector field $\nabla \phi_t$ in  \vspace{-0.2cm}
\begin{align}
\label{eq:controlled annealing} \d\fY_t  \!=\! \left(\sigma^2\nabla \log \pi_t(\fY_t)\! +\! \nabla \phi_t (\fY_t)\right) \d t + \sigma \!\sqrt{2 } \,\ora{\d} \fW_t, \qquad \fY_0 \!\sim \!\pi_0, 
\end{align}
so that (\ref{eq:controlled annealing}) produces the interpolation from the prior $\pi_0$ to the posterior $\pi_T$, i.e., $\ora{\P}^{\pi_0,\sigma^2\nabla \log \pi + \nabla \phi}_t = \pi_t$, for all $t \in [0,T]$. Note that if $\pi_t$ were constant in time ($\pi_t = \pi_0$), then $\phi = 0$ would reduce (\ref{eq:controlled annealing}) to equilibrium overdamped Langevin dynamics, preserving $\pi_0$. With $\pi_t$ varying in time, $\nabla \phi_t$ can be thought of as a control  enabling transitions between neighbouring densities $\pi_t$ and $\pi_{t + \delta t}$. 

To obtain $\nabla \phi_t$ we invoke Framework \ref{frm2}, but restrict $\ola{\P}^{\pi_T,b}$ to retain uniqueness. Proposition \ref{prop:Nelson} motivates the choice $b_t = (\sigma^2 \nabla \log \pi_t + \nabla \phi_t) - 2 \sigma^2 \nabla \log \pi_t = - \sigma^2 \nabla \log \pi_t + \nabla \phi_t$,\footnote{Note the additional factor of $2$ in Nelson's relation due to the noise scaling $\sigma\sqrt{2} \ora{\d}\fW_t$ in (\ref{eq:controlled annealing}).}
leading to\vspace{-0.1cm}
\begin{align}
\label{eq:annealed_div}
\!\!\!\!\mathcal{L}^{\mathrm{CMCD}}_D(\phi) := D \left( {\ora{\P}^{\pi_0,\sigma^2 \nabla \log \pi + \nabla \phi}}, {\ola{\P}^{\pi_T, - \sigma^2 \nabla \log \pi + \nabla \phi}} \right), 
\end{align}
which is valid for any choice of divergence $D$. The additional score constraint $b_t = a_t - 2\sigma^2 \nabla \log \pi_t$ restores uniqueness in Framework \ref{frm2} 
(see Appendix \ref{app:annealing} for a proof): 

\begin{algorithm}[t]  
\caption{ Controlled Monte Carlo Diffusions - Sampling and normalizing constant estimation} 
\label{alg:Fldvi2}
\begin{algorithmic}
\Require $\pi_0$, $\pi_T$, $\pi_t$, $\sigma$, $K$ step-sizes $\Delta t_{k}$,  network $f^{\phi}$ trained via minimising Eq \ref{eq:main objective}
\State $\fY_0 \sim \pi_0$
\State $ \ln \mW = -\log \pi_0(\fY_0)$
\For{$k=0$  to $K-1$}
\State $\fY_{t_{k+1}} \sim \gN\Big(\mY_{t_{k+1}} \big| \mY_{t_{k}} + (\sigma^2 \nabla \ln \pi_{t_k} +f^{\phi}_{t_k}) (\mY_{t_{k}})\Delta t_k, 2 \sigma^2 \Delta t_k \Big)$
\State $\ln \mW  = \ln \mW + \ln \frac{ \gN\Big(\mY_{t_{k}} \big| \mY_{t_{k+1}} + (\sigma^2\nabla  \ln \pi_{t_{k+1}} - f^{\phi}_{t_{k+1}})(\mY_{t_{k+1}})\Delta t_k, 2 \sigma^2 \Delta t_k \Big)}{\gN\Big(\mY_{t_{k+1}} \big| \mY_{t_{k}} + (\sigma^2 \nabla \ln \pi_{t_k} +f^{\phi}_{t_k}) (\mY_{t_{k}})\Delta t_k, 2 \sigma^2 \Delta t_k \Big)}$
\EndFor
\State $\ln \mW  = \ln \mW + \ln \pi_T(\fY_{T})$
\State \Return $( \text{Estimate of } \ln Z\approx \ln \mW ,  \text{Approximate sample } \fY_{T})$  
\end{algorithmic}
\end{algorithm}

\begin{proposition}[Existence and uniqueness]
\label{prop:annealing uniqueness}
Under mild conditions on $(\pi_t)_{t \in [0,T]}$, (\ref{eq:annealed_div}) admits a ($\pi_t$-a.e.) unique minimiser $\phi^*$, up to additive constants, satisfying $\mathcal{L}_{\mathrm{CMCD}}(\phi^*) = 0$. 
\end{proposition}

Given the optimal vector field $\nabla \phi_t^*$,  we can produce samples from $\pi_T$ by simulating (\ref{eq:controlled annealing}). Following \citet{zhang2021path,vargas2023denoising},we can estimate $Z$ in $\pi_T = \hat{\pi}_T/Z_T$ unbiasedly via  \vspace{-0.1cm}
\begin{equation}
\label{eq:Z estimate}
Z = \E \left[  \frac{\d \ola{\P}^{\hat{\pi}_T, - \sigma^2 \nabla \log \pi + \nabla \phi}}{\d \ora{\P}^{\pi_0,\sigma^2 \nabla \log \pi + \nabla \phi}} \right] = \frac{\d \ola{\P}^{\hat{\pi}_T, - \sigma^2 \nabla \log \pi + \nabla \phi^*}}{\d \ora{\P}^{\pi_0,\sigma^2 \nabla \log \pi + \nabla \phi^*}}(\fY),  \end{equation} 
where the expectation is taken with respect to (\ref{eq:controlled annealing}), and is valid for any (possibly suboptimal) $\nabla \phi_t$. The right-hand side, on the other hand, shows that optimality  of $\nabla \phi^*_t$ yields a zero-variance estimator of $Z$, as the statement holds almost surely in $\fY$, without taking the expectation. To give a broader perspective, we give the following slight generalisation of a  well-known result from statistical physics:  
\begin{proposition}[Controlled Crooks' fluctuation theorem and Jarzynki's equality]\label{prop:crooks}
Following \citet{jarzynski1997nonequilibrium,chen2019stochastic},
define \emph{work} and \emph{free energy} as $\gW_T(\fY) := -\int_0^T \sigma^2\partial_t \ln \hat{\pi}_t(\fY_t)\,\d t$, $\gF_t := -{\sigma^2} \ln Z_t:=\sigma^2 \log(\hat{\pi_t}/\pi_t)$.
Then, we have the \emph{controlled Crooks' identity}\vspace{-0.1cm},$$\!\!\left(\frac{\d\fP^{\pi_0,\sigma^2 \nabla \ln {{\pi}} + \nabla \phi}}{\d\bP^{\pi_T, -\sigma^2 \nabla \ln  {\pi}+ \nabla \phi}}\!\right)\!\!(\fY)\! =\exp \left(-\tfrac{1}{\sigma^2}(\gF_T \!-\!\gF_0)+\! \tfrac{1}{\sigma^2} \gW_T(\fY)\! + \mathcal{C}^\phi_T(\fY)\right)\!,$$
where $\mathcal{C}^\phi_T(\fY) := -  \!\!\int_0^T \nabla \phi_t(\fY_t) \!\cdot\! \nabla \log {\pi}_t(\fY_t) \, \d t - \int_0^T \Delta \phi_t(\fY_t) \d t$.
By taking expectations and $\phi = 0$, this implies \emph{Jarzynski's equality} $\E_{\fP^{\pi_0,\sigma^2 \nabla \ln {{\pi}}}}[\exp(-\tfrac{1}{\sigma^2} \gW_T)] = Z_T/Z_0$.
\end{proposition}

The proof uses Proposition \ref{prop:RND}  to compute the RND ${\d\!\fP^{\pi_0,\sigma^2 \nabla \ln {\pi}+\nabla \phi}}\!/{\d\!\bP^{\pi_T, -\sigma^2 \nabla \ln {\pi}+\nabla \phi}}$, followed by applying It{\^o}'s formula to $t \!\mapsto\!\! \log \!\hat{\pi}_t(\!\fY_t)$, see Appendix \ref{app:fluc}. For $\phi = 0$, we recover Crooks fluctuation theorem \citep{crooks1999entropy}, but the additional control allows CMCD to suppress said fluctuations by adjusting the interaction term $\mathcal{C}^\phi_T(\fY)$. Indeed, prior works \citep{neal2001annealed,chopin2002sequential,vaikuntanathan2008escorted,hartmann2019jarzynski,zhang2021some} have used the Jarzynski equality to estimate $Z$ via importance sampling, but this approach might suffer from high variance, see \citep{del2006sequential}, \citep[Section 4.1.4]{stoltz2010free}. In contrast, the CMCD estimator version of (\ref{eq:Z estimate}) achieves zero variance if trained to optimality (see Appendix \ref{app:dnf} for a convenient discretised version). Finally, we would like to highlight that \cite{zhong2023time} concurrently proposes this generalisation of Crook's identity using different techniques in their sketch.

Our next result connects CMCD to Section \ref{sec:schr}, showing that minimising (\ref{eq:annealed_div}) can be viewed as jointly solving an infinite number of Schr{\"o}dinger problems on infinitesimal time intervals:
\begin{proposition}
[infinitesimal Schr\"odinger problems]
\label{prop:sch inf}
The minimiser $\phi^*$ can be characterised as follows: For $N \in \mathbb{N}$, partition the interval $[0,T]$ into $N$ subintervals of length $T/N$, and on each subinterval $[(i-1)T/N,iT/N]$, solve the Schr{\"o}dinger problem (\ref{eq:dynamic sch}) with marginals $\mu = \pi_{(i-1)T/N}$, $\nu = \pi_{iT/N}$ and prior drift $f_t = \nabla \log \pi_t$. Concatenate the solutions to obtain the drift $\nabla \phi^{(N)}$, defined on $[0,T]$. Then, $\nabla \phi^{(N)} \!\!\rightarrow \!\nabla \phi$ as $N \!\rightarrow \!\infty$ in the sense of $L^2([0,T] \times \mathbb{R}^d,\pi)$ (proof in Appendix \ref{app:infi schr})
\end{proposition}
Note the similarity of this interpretation to the sequential Schr{\"o}dinger samplers of \citet{bernton2019schr}. Making specific choices for $D$ in \ref{eq:annealed_div}, we establish further connections to other methods:

1. For $D = \KL$, direct computation (see Appendix \ref{app:cmcd kl}) based on Proposition \ref{prop:RND} shows that\vspace{-0.15cm}
\begin{subequations}
\label{eq:main objective}
\begin{align}
\mathcal{L}^{\mathrm{CMCD}}_{\KL}\!(\phi)\! &=  \textcolor{magenta}{\E_{\fY \sim \ora{\P}^{\pi_0,\sigma^2 \nabla \log \pi + \nabla \phi}} \left[ \log \left( \frac{\d\ora{\P}^{\pi_0,\sigma^2 \nabla \log \pi + \nabla \phi}}{\d \ola{\P}^{\pi_T,-\sigma^2 \nabla \log \pi + \nabla \phi}} \right) (\fY) \right]} \nonumber\\ 
& = \!\E\! \left[ \sigma^2 \!\! \int_0^T \!\!\!\!|\nabla \log \pi_t(\fY_t)|^2  \d t \! + \! \tfrac{1}{\sigma \sqrt{2}} \! \int_0^T \!\!\!\! \left(\sigma^2 \nabla \log \pi_t \!- \!\nabla \phi_t \right)\!(\fY_t) \!\cdot \!\!\ola{\d} \!\!\fW_t - \log \hat{\pi}_T(\fY_T)\!\right] \! + \!\mathrm{const.}
\nonumber
\\
&\approx\! \!\E\!\!\left[\!\ln\! \frac{{\pi}_0(\mY_0)}{\hat{\pi}(\mY_T)} \!\!\prod_{k=0}^{K-1} \!\!\!\frac{\gN(\mY_{t_{k+1}} | \mY_{t_{k}} + (\sigma^2\nabla \ln \pi_{t_k} + \nabla \phi_{t_k}) (\mY_{t_{k}})\Delta t_k, 2 \sigma^2 \Delta t_k )}{\gN\!(\mY_{t_{k}} | \mY_{t_{k+1}}\!\! \!+ \!(\sigma^2 \nabla\! \ln \pi_{t_{k+1}}\! \!\!- \!\nabla\! \phi_{t_{k+1}})(\mY_{t_{k+1}})\Delta t_k, 2 \sigma^2 \Delta t_k )}\!\!\right]\!\!,
\tag{\ref*{eq:main objective}}
\end{align}
\end{subequations} the time-discretisation in the third line is derived in Appendix \ref{app:discr}.  \textcolor{magenta}{Our goal is then to numerically minimize $\mathcal{L}^{\mathrm{CMCD}}_{\KL}\!(\phi)\!$ wrt to $\phi$ (for a numerical minimisation scheme see Algorithm \ref{alg:Fldvi3})}. Note the first line in  (\ref{eq:main objective}) is akin to the optimal control type objectives of F{\"o}llmer and DDS samplers 
recalled at the end of Section \ref{sec:SDEs}, see also \citep{berner2022optimal}. Setting $\phi =0$ in the third line recovers Unadjusted Langevin Annealing (ULA), see, e.g., eq. (14) in \citep{thin2021monte} or eq. (21) in \citep{geffner2023langevin}; hence, we can view CMCD as a controlled version of ULA. Setting $\phi\!=\!\!0$ \emph{only in the numerator} is akin to Monte Carlo Diffusion (MCD), see Algorithm 1 and eq. (34) in \citep{doucet2022score}.   Finally, action matching \citep{neklyudov2022action} can be recovered from $D\!\!=\!\!\KL$ and Framework \ref{frm2} by choosing the reference $\ora{\P}^{\Gamma_0,\gamma^+}\!\!\!\! =\!\! \ola{\P}^{\Gamma_T,\gamma^-}$ in Proposition \ref{prop:RND} appropriately, see Appendix \ref{app:action matching}.


2. For the log-variance divergence $D_{\mathrm{Var}}(\Q,\P)\!=\!\!\Var\big(\! \ln \frac{\d \Q}{\d \P}\!\big)$, see Appendix \ref{app:vi}, we obtain \vspace{-0.15cm}
\begin{align}
\nonumber
&\mathcal{L}^{\mathrm{CMCD}}_{\mathrm{Var}}(\phi)\! = \Var \Bigg( \!\!\!\log \!\frac{\pi_T(\fY_T)} {\pi_0(\fY_0) }\!+ \!\int_0^T \!\!\!\!\!\Delta \phi_t(\fY_t)\, \d t \!\!-\!\! \sigma \sqrt{2 }\!\! \!\int_0^T\!\!\! \!\!\nabla \log \pi_t(\fY_t)\! \circ\! \d \fW_t\! - \sigma^2\! \!\int_0^T\!\!\! \!\!|\nabla \log \pi_t(\fY_t)|^2 \, \!\d t\!\!\Bigg) , 
\end{align}
see Appendix \ref{app:annealing}.
Here, $\circ \d \fW_t$ denotes Stratonovich integration, and the variance is taken with respect to samples from (\ref{eq:controlled annealing}). In the limit $\sigma \rightarrow 0$, log-Var CMCD enforces an integrated version of the instantaneous change of density formula $\partial_t \log \pi_t(\fY_t) = - \Delta \phi_t(\fY_t)$ for continuous-time normalising flows of the form $\dot{\fY_t} = \nabla \phi_t(\fY_t)$, \citep[Section 4]{papamakarios2021normalizing}.

\begin{figure}
    \centering
    \includegraphics[width=0.96\textwidth]{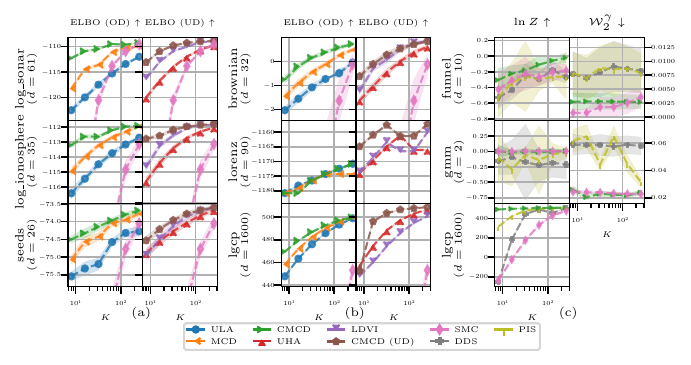}
    \vspace{-0.5cm}
    \caption{Figure panes a) and b) report ELBOs across methods and targets following the experimental setup in \cite{geffner2023langevin}, the (OD) and (UD) columns group over and under-damped methods seperately. Figure c) reports IS $\ln Z$ estimates and sample quality (where available) using eOT. Higher ELBO and $\ln Z$ denote better estimates, lower $\gW^\gamma_2$ signifies better sample quality.  }
    \label{fig:simple} \vspace{-0.5cm}
\end{figure}




\begin{remark}[Further related work]
The task of learning the vector field $\nabla \phi_t$ so that (\ref{eq:controlled annealing}) reproduces $(\pi_t)_{t \in [0,T]}$ has been approached from various directions. \citet{reich2011dynamical,heng2015gibbs,reich2022data,vaikuntanathan2008escorted} explore methodologies that exploit the characterisation of $\nabla \phi_t$ in terms of the elliptic PDE (\ref{eq:continuity}) in Appendix \ref{app:drift prop}.
\citet{arbel2021annealed} propose to leverage normalising flows sequentially to minimise KL divergences between implied neighboring densities. In an appropriate limiting regime, they recover the SDE (\ref{eq:controlled annealing}), see Remark \ref{rem:AFT}. These approaches approximate $\nabla \phi_t$ sequentially in time, whilst CMCD learns $(\nabla \phi_t)_{t \in [0,T]}$ `all-at-once'. 
\end{remark}



\section{Experiments}
\label{sec:experiments}

We now empirically demonstrate the performance of the proposed CMCD sampler (\ref{eq:main objective}) in both underdamped (detailed in Appendix \ref{app:annealing}) and overdamped (CMCD (OD)), Appendix \ref{sec:underdamped}) formulations on a series of sampling benchmarks. We first replicate the benchmarks from \citet{geffner2023langevin} on 6 standard target benchmark distributions.  Following the experimental methodology in \citet{geffner2023langevin}, we compare against two underdamped baselines, Unadjusted Langevin Annealing (ULA) \citep{wu2020stochastic, thin2021monte} and Monte Carlo Diffusion (MCD) \citep{doucet2022annealed}; and two overdamped baselines, Uncorrected Hamiltonian Annealing (UHA) \citep{geffner2021mcmc, ZhangAIS2021} and Langevin Diffusion Variational Inference (LDVI) \citep{geffner2023langevin}. Furthermore, we include comparisons of $\ln Z$ estimation on two datasets with known partition function, the \texttt{funnel} and \texttt{gmm}, and compare against baselines from \citet{vargas2023denoising}, PIS \citep{barr2020quantum, vargas2021bayesian, zhang2021path}, DDS \citep{vargas2023denoising}, and Sequential Monte Carlo Sampler (SMC) \citep{del2006sequential, zhou2016toward}.

We report the mean ELBO achieved by each method over 30 seeds of sampling, for Euler discretisation steps $K \in \{8, 16, 32, 64, 128, 256\}$, comparing the underdamped and overdamped baselines to their respective CMCD counterparts in Figure \ref{fig:simple}. We see that both overdamped and underdamped CMCD consistently outperform all baseline methods, especially at low $K$,  and in fact, across most targets overdamped CMCD outperforms the underdamped baselines. Figure \ref{fig:simple} also reports $\ln Z$ for two target distributions with known $Z$, comparing against PIS, DDS, and SMC. Again, CMCD recovers the log-partition more consistently, even at low $K$. Finally, as another measure of sample quality, we report the entropy-regularised OT distance ($\mathcal{W}_2^{\gamma}$) between obtained samples and samples from the target for \texttt{funnel} and \texttt{gmm}. Hyperparameter tuning and other experimental details can be found in Appendix \ref{app:cmcd_exps} and we provide a GitHub repository to reproduce our results \footnote{\url{ https://github.com/shreyaspadhy/CMCD}}.



\section{Discussion}

Overall we have successfully introduced a novel variational framework bridging VI and transport using modern advances in diffusion models and processes. In particular, we have shown that many existing diffusion-based methods for generative modelling and sampling can be viewed as special instances of our proposed framework. Building on this, we have developed novel objectives for dynamic entropy regularised transports (based on a relationship between the EM and IPF algorithms) and annealed flows (with connections to fluctuation theorems due to Crook and Jarzynski, rooted in statistical physics). Finally, we have explored the CMCD inference scheme obtaining state-of-the-art results across a suite of challenging inference benchmarks. We believe this experimental success is partly due to our approach striking a balance between parametrising a flexible family of distributions whilst being constrained enough such that learning the sampler is not overly expensive \citep{tzen2019theoretical,vargas2023expressiveness}. Future directions can explore optimal schemes for the annealed flow $\pi_t$ \citep{goshtasbpour2023adaptive} and alternate divergences \citep{nusken2021solving, richter2023improved, midgley2022flow}.


\section{Acknowledgements}

We would like to thank Jeremy Heng, Arnaud Doucet, and Valentin De Bortoli for their insightful discussions and feedback which led to the improvement of this manuscript. We would also like to thank Michael S. Albergo and Eric Vanden-Eijnden for highlighting a typo in Proposition \ref{prop:crooks}.

\bibliography{references}
\bibliographystyle{iclr2024_conference}

\appendix
 
\section{Stochastic analysis for backward processes}
\label{app:stochana}

In this appendix, we briefly discuss background in stochastic analysis relevant to the SDEs in (\ref{eq:SDEs}), here repeated for convenience:
\begin{subequations}
 \label{eq:SDEs_app}
 \begin{align}
\dd \fY_t &= a_t(\fY_t) \,\dd t + \sigma \ddf \fW_t, \qquad \fY_0 \sim  \mu,
\label{eq:fwd_sde_app}
\\
\dd \bY_t &= b_t(\bY_t) \,\dd t + \sigma\ddb \bW_t, \qquad \bY_T \sim  \nu. 
\label{eq:back_sde_app}
\end{align}
\end{subequations}

Recall that the forward It\^o differential $\ora{\d}$ in (\ref{eq:fwd_sde_app}) is far more commonly denoted simply\footnote{...but in this paper we stick to the notation $\ora{\d}$ to emphasise the symmetry of the setting.} by $\d$, and theory for the  forward SDE (\ref{eq:fwd_sde_app}) is widely known \citep{karatzas1991brownian,oksendal2003stochastic}. In contrast, reverse-time SDEs of the form (\ref{eq:back_sde_app}) are less common and there are fewer textbook accounts of their interactions with forward SDEs. We highlight  \citet{kunita2019stochastic} for an in-depth treatment, and alert the reader to the fact that `backward stochastic differential equations' as discussed in \citet{zhang2017backward,chen2021likelihood}, for instance, are largely unrelated. We therefore refer to (\ref{eq:back_sde_app}) as a `reverse-time' SDE.

\begin{remark}[Notation]
We deliberately depart from some of the notation employed in the recent literature (see, for instance, \citet{huang2021variational,liu2022let}) by using $\fY_t$ in both (\ref{eq:fwd_sde_app}) and (\ref{eq:back_sde_app}), and not introducing an auxiliary process capturing the reverse-time dynamics. From a technical perspective, this is justified since $(\fY_t)_{0 \le t \le T}$ merely represents a generic element in path space, and full information is encoded in the path measures $\Q^{\mu,a} \equiv \ora{\P}^{\mu,a}$ and $\P^{\nu,b} \equiv \ola{\P}^{\nu,b}$. Importantly, placing (\ref{eq:fwd_sde_app}) and (\ref{eq:back_sde_app}) on an equal footing seems essential for a convenient formulation of Proposition \ref{prop:RND}. Slightly departing from the VAE-inspired notation from Section \ref{sec:hierVAE}, we equivalently refer to these path measures by $\ora{\P}^{\mu,a}$ and $\ola{\P}^{\nu,b}$, highlighting the symmetry of the setting in (\ref{eq:SDEs_app}).
\end{remark}

Intuitively, (\ref{eq:SDEs_app}) can be viewed as continuous time limits of the Markov chains defined in (\ref{eq:Markov chains}), or, in other words, the Markov chains (\ref{eq:Markov chains}) are the Euler-Maruyama discretisations for (\ref{eq:SDEs_app}), see \citet[Section 9.1]{kloeden1992stochastic}. Throughout, we impose the following:

\begin{assumption}[Smoothness and linear growth of vector fields]
\label{ass:vector fields}
All (time-dependent) vector fields in this paper belong to the set
\begin{subequations}
\begin{align*}
\mathcal{U} := \Bigg\{ a \in C^{\infty}([0,T] \times& \mathbb{R}^d;\mathbb{R}^d): \quad \text{there exists a constant } \, C>0 \,\, 
\\
& \text{ such that } \Vert a_t(\vx) - a_t(\vy)\Vert \le C \Vert \vx - \vy \Vert, \, \text{for all } t \in [0,T], \,\, \vx,\vy \in \mathbb{R}^d \Bigg\}.
\end{align*}
\end{subequations}
\end{assumption}

The preceding assumption guarantees existence and uniqueness for (\ref{eq:fwd_sde_app}) and (\ref{eq:back_sde_app}), and it allows us to use Girsanov's theorem in the proof of Proposition \ref{prop:RND} (Novikov's condition can be shown to be satisfied, see \citet[Section 8.6]{oksendal2003stochastic}). Furthermore, Assumption \ref{ass:vector fields} is sufficient to conclude Nelson's relation (Proposition \ref{prop:Nelson}), see \citet{haussmann1985time,millet1989integration,follmer2006time} and the discussion in \citet{russo1996ito}. Having said all that, it is possible to substantially weaken Assumption \ref{ass:vector fields} with more technical effort. Moreover, we can replace the constant $\sigma>0$ by $\sigma:[0,T]\times \mathbb{R}^d \rightarrow \mathbb{R}^{d \times d}$ throughout, assuming sufficient regularity, growth and invertibility properties, and amending the formulas accordingly.

The precise meaning of (\ref{eq:SDEs_app}) is given by the  integrated formulations
\begin{subequations}
\begin{align}
\fY_t & = \fY_0 + \int_0^t a_s(\fY_s) \, \d s + \int_0^t \sigma \ora{\d} \fW_s, \quad && \fY_0 \sim \mu, 
\\
\fY_t & = \fY_T - \int_t^T b_s(\fY_s) \, \d s - \int_t^T \sigma \ola{\d} \fW_s, \quad && \fY_T \sim \nu,
\end{align}
\end{subequations}
where the forward and backward integrals need defining. Roughly speaking, we have 
\begin{subequations}
\label{eq:fb integrals def}
\begin{align}
\int_{t_0}^{t_1} \fX_s \cdot \ora{\d} \fZ_s & = \lim_{`\Delta t \rightarrow 0'} \sum_{i} \fX_{t_i} \cdot (\fZ_{t_{i+1}} - \fZ_{t_{i}}),
\label{eq:fwd}
\\
\quad \int_{t_0}^{t_1} \fX_s \cdot \ola{\d} \fZ_s & = \lim_{`\Delta t \rightarrow 0'} \sum_{i} \fX_{t_{i+1}} \cdot (\fZ_{t_{i+1}} - \fZ_{t_{i}}),
\label{eq:bwd}
\end{align}
\end{subequations}
see Remark \ref{rem:conversion},
for `appropriate' processes $(\fX_t)_{0 \le t \le T}$ and $(\fZ_t)_{0 \le t \le T}$, and where the limit $\Delta t \rightarrow 0$ of vanishing step sizes needs careful analysis (see Remark \ref{rem:limits} below). The most salient difference between (\ref{eq:fwd_sde_app}) and (\ref{eq:back_sde_app}) is the fact that $\fX_{t_i}$ is replaced by $\fX_{t_{i+1}}$ in (\ref{eq:bwd}). 
\begin{remark}[Convergence of the limits in (\ref{eq:fb integrals def})]
\label{rem:limits}
If we only assume that $\fX,\fZ \in C([t_0,t_1];\mathbb{R}^d)$, possibly pathwise, that is, deterministically, then the limits in (\ref{eq:fb integrals def}) might not exist, or when they do, their values might depends on the specific sequence of mesh refinements. The  following approaches are available to make the definitions (\ref{eq:fb integrals def}) rigorous:
\begin{enumerate}
\item  It\^o calculus (see, for example, \citet[Chapter 9]{revuz2013continuous}) uses adaptedness and semimartingale properties for the forward integral in (\ref{eq:fwd}), but note that the definition is not pathwise (that is, the limit (\ref{eq:fwd}) is defined up to a set of measure zero). For the backward integrals in (\ref{eq:bwd}) and, importantly for us, in (\ref{eq:RN backward}), it can then be shown that the relevant processes are (continuous) reverse-time martingales (see \citet{kunita2019stochastic} for a discussion of the corresponding filtrations). The latter property is guaranteed under Assumption \ref{ass:vector fields}, see the discussion around Theorem 2.3 in \citet{russo1996ito}.
\item 
F{\"o}llmer's `It\^o calculus without probabilities' \citep{follmer2006calcul} is convenient, since it allows to us to perform calculations using (\ref{eq:SDEs_app}) and Proposition \ref{prop:RND} without introducing filtrations and related stochastic machinery. The caveat is that the results may in principle depend on the sequence of mesh refinements, but under Assumption \ref{ass:vector fields}, those differences only appear on a set of measure zero, see \citet{russo1995generalized,follmer2000ito}.
\item Similarly, the integrals in (\ref{eq:fb integrals def}) can be defined in a pathwise fashion using rough path techniques, see \citet[Section 5.4]{friz2020course}.
\end{enumerate}
\end{remark}

For the current paper, the following conversion formulas are crucial,
\begin{subequations}
\label{eq:conversion app}
\begin{align}
\label{eq:quad var}
\int_{0}^{t} \fX_s \cdot 
 \ola{\d} \fZ_s & - \int_0^t \fX_s \cdot \ora{\d} \fZ_s =  \langle \fX,\fZ\rangle_t,
\\
\label{eq:Strato}
\int_{0}^{t} \fX_s \cdot 
 \ola{\d} \fZ_s  & + \int_0^t \fX_s \cdot \ora{\d} \fZ_s =  2 \int_0^t \fX_s \circ \fZ_s,
 \end{align}
\end{subequations}
where $\langle \fX,\fZ\rangle$ is the quadratic variation process (if defined, see \citet{russo1995generalized}; see in particular equations (3) and (4) therein), and $\circ$ denotes Stratonovich integration. For solutions to (\ref{eq:SDEs_app}), we obtain (\ref{eq:conversion}) from (\ref{eq:quad var}). In particular, we can often trade backward integrals for divergence terms (see Appendix \ref{app:discussion}), using the (backward) martingale properties
\begin{subequations}
\begin{align}
\label{eq:martingale}
\mathbb{E} \left[ \int_0^t f_t(\fY_t) \cdot \ora{\d} \fW_t \right] = 0, \quad \text{if } (\fY_t)_{0 \le t \le T} \,\, \text{solves (}\ref{eq:fwd_sde_app}\text{)},
\\
\mathbb{E} \left[ \int_t^T f_t(\fY_t) \cdot \ola{\d} \fW_t \right] = 0, \quad \text{if } (\fY_t)_{0 \le t \le T} \,\, \text{solves (}\ref{eq:back_sde_app}\text{)}.
\end{align}
\end{subequations}




\section{Variational inference and divergences}
\label{app:vi}

Various concepts well-known in the variational inference community have direct counterparts in the diffusion setting. In this appendix we review a few that are directly relevant to this paper.

\noindent \textbf{Maximum likelihood.} Framework \ref{prop:framework} with $D=\KL$ leads via direct calculations to
\begin{align}
\mathcal{L}_{\KL}(\phi,\theta) = - &\mathbb{E}_{\vx \sim \mu(\vx) }  \overbrace{\left[\int \ln \frac{p^{\theta}(\vx | \vz)\nu(\vz)}{q^\phi (\vz|\vx) }q^\phi(\mathrm{d}\vz|\vx)\right]}^{=\mathrm{ELBO}_x(\phi,\theta)} + \int \log \mu(\vx) \mu(\mathrm{d}\vx),
 \label{eq:ELBO}
 \end{align}
so that maximising $\mathbb{E}_{\vx \sim \mu(\vx) }[\mathrm{ELBO}_x(\phi,\theta)]$ is equivalent to minimising $\mathcal{L}_{\KL}(\phi,\theta)$.

However, the traditional  approach \citep{blei2017variational,kingma2019introduction} towards the \emph{evidence lower bound} (ELBO) in (\ref{eq:ELBO}) is via maximum likelihood in latent variable models. Using the notation and set-up from the introduction, one can show using Jensen's inequality (or dual representations of the KL divergence), that 
\begin{equation}
\label{eq:likelihood bound}
\log \left( \int p_\theta(\vx,\vz) \, \d \vz  \right) = \log p_{\theta}(\vx) \ge \mathrm{ELBO}_{\vx}(\phi,\theta),
\end{equation}
with equality if and only if $q_\phi(\vz|\vx) = p_\theta(\vz|\vx)$. The bound in (\ref{eq:likelihood bound}) motivates maximising the (tractable) right-hand side, performing model selection (according to Bayesian evidence) and posterior approximation (in terms of the variational family $q_\phi(\vz|\vx)$) at the same time. The calculation in (\ref{eq:ELBO}) shows that this objective can equivalently be derived from Framework \ref{prop:framework} and connected to the KL divergence between the joint distributions $q_\phi(\vx,\vz)$ and $p_\theta(\vx,\vz)$. 

\noindent \textbf{Reparameterisation trick \citep{kingma2013auto,rezende2014stochastic}.} For optimising $\mathrm{ELBO}_{\vx}(\phi,\theta)$, it is crucial to select efficient low-variance gradient estimators. In this context, it has been observed that reparameterising $\vz \sim q_\phi(\vz|\vx)$ in the form $\vz = g(\epsilon,\phi,\vx)$, see \citet[Section 2.4.1]{kingma2019introduction}, substantially stabilises the training procedure. Here, $\epsilon$ is an auxiliary random variable with tractable `base distribution' that is independent of $\phi$ and $\vx$, and $g$ is a deterministic function (transforming $\epsilon$ into $\vz$), parameterised by $\phi$ and $\vx$. We would like to point out that many (although not all, see below) objectives in diffusion modelling are already reparameterised, since the SDEs (\ref{eq:SDEs}) transform the `auxiliary' variables $(\fW_t)_{0 \le t \le T}$ into $(\fY_t)_{0 \le t \le T}$. With this viewpoint, the vector field $a_t$ corresponds to the parameter $\phi$, $(\fW_t)_{0 \le t \le T}$ corresponds to $\epsilon$, and $g$ corresponds to the solution map associated to the SDE (\ref{eq:fwd_sde}), sometimes referred to as the It\^o  map. In this sense, the objectives (\ref{eq:sch objective}), (\ref{eq:Foellmer}) and (\ref{eq:DDS}) are reparameterised, but $\mathcal{L}^{\mathrm{CMCD}}_{\mathrm{Var}}$ from Section \ref{sec:annealing} is not if the gradients are detached as in \citep{nusken2021solving,richter2020vargrad,richter2023improved}. We mention in passing that \emph{sticking the landing} \citep{roeder2017sticking} offers a further variance reduction close to optimality, and that the same method can be employed for diffusion objectives, see \citet{vargas2021bayesian,xu2021infinitely}. 

\noindent \textbf{Reinforce gradient estimators.} As an alternative to the KL-divergence, \citet{nusken2021solving} investigated the family of \emph{log-variance divergences} 
\begin{equation}
\label{eq:vargrad}
D^u_{\mathrm{Var}}(q||p) = \Var_{\vx \sim u}\left( \log \frac{\d q}{\d p} (\vx)\right),
\end{equation}
parameterised by an auxiliary distribution $u$, in order to connect variational inference to backward stochastic differential equations \citep{zhang2017backward}. The fact that gradients of (\ref{eq:vargrad}) do not have to be taken with respect to $\vx$ (see Remark \citep{nusken2021solving,richter2020vargrad}) reduces the computational cost and provides additional flexibility in the choice of $u$, but the gradient estimates potentially suffer from higher variance since the reparameterisation trick is not available. The latter drawback is alleviated somewhat by the fact that particular choices of $u$ can be linked to control variate enhanced reinforce gradient estimators \citep{richter2020vargrad} that are particularly useful when reparameterisation is not available (such as in discrete models). We note that the same divergence has also been used as a variational inference objective in \citet{el2012bayesian}.

\noindent \textbf{Importance weighted autoencoders (IWAE).} \citet{burda2015importance} have developed a multi-sample version of $\mathrm{ELBO}_{\vx}(\phi,\theta)$ that achieves a tighter lower bound on the marginal log-likelihood in (\ref{eq:likelihood bound}). To develop similar objectives in a diffusion setting, we observe that for each $K \ge 1$,
\begin{equation}
\label{eq:IWAE KL}
\KL^{(K)} (q||p) = \mathbb{E}_{x_1,\ldots, x_K \overset{iid}{\sim} q} \left[ \log \left( \frac{1}{K} \sum_{i=1}^K \frac{\d q}{\d p}(x_i)\right) \right] 
\end{equation}
defines a generalised KL divergence\footnote{Indeed, by Jensen's inequality, we have that $\KL^{(K+1)}(q||p) \ge \KL^{(K)}(q||p)$, so that in particular $q \neq p$ implies $\KL^{(K)}(q||p) > 0$.} that reproduces the IWAE lower bound as per Framework \ref{prop:framework}, in the sense of equation (\ref{eq:ELBO}). To the best of our knowledge, the precise formulation in (\ref{eq:IWAE KL}) is new, but similar to the previous works \citet{hernandez2016black,li2016renyi,daudel2022alpha}. We exhibit an example of (\ref{eq:IWAE KL}) applied in a diffusion context in Section \ref{sec:sampling}, see Remark \ref{rem:IWAE sampling}.

\section{Connections to previous work}
\label{app:calcs}

\subsection{Discussion of equivalent expressions for $\KL(\ora{\P}^{\mu,a}||\ola{\P}^{\nu,b})$}
\label{app:discussion}
Notice that we can realise samples from  $\bP^{\nu, b}$ both via the reverse-time SDE in (\ref{eq:back_sde}) or via its time reversal given by the following forward SDE \citep{nelson1967dynamical,anderson1982reverse,haussmann1985time}:
\begin{align}
\dd \widehat{\fY}_t &= \left(b_{T-t}(\widehat{\fY}_t) -\sigma^2 \nabla \ln \ola{\rho}^{\nu,b}_{T-t}(\widehat{\fY}_t) \right)\dd t + \sigma \,\ora{\d} \fW_t, \;\;\ \quad \widehat{\fY}_0 \sim \nu,
\end{align}
using $\widehat{\fY}_t := \widehat{\fY}_{T-t}$.
This allows us to obtain an expression for $\KL (\fP^{\mu, a} | \bP^{\nu, b})$ via Girsanov's theorem:
\begin{align}
\KL (\fP^{\mu, a} || \bP^{\nu, b}) =  \KL(\ola{\rho}_0^{\nu,b}|| \nu) + \E\left[\tfrac{1}{2\sigma^2}\int_0^T \Big|\Big| a_t(\fY_t)  - \left(b_{t}(\fY_t) -\sigma^2 \nabla \ln \ola{\rho}^{\nu,b}_{t}(\fY_t) \right) \Big|\Big|^2 \dd t\right].
 \end{align}
However there are several terms here that we cannot estimate or realise in a tractable manner, one being the score $\nabla \ln \ola{\rho}^{\nu,b}_{t}$ and the other being sampling from the distribution $\ola{\rho}_0^{\nu,b}$. \footnote{When $\widehat{\fY}_t$ is an OU process and $\mu$ is Gaussian we are in the traditional DDPM setting \citep{song2020score} and these two quantities admit the classical tractable score matching approximations}
 	
In order to circumvent the score term, the authors \citet{vargasshro2021,chen2021likelihood} use the Fokker-Plank (FPK) equation and integration by parts, respectively, trading of the score with a divergence term, whilst  \citet{huang2021variational} use a variant of the Feynman Kac formula to arrive at an equivalent solution. From Proposition \ref{prop:RND}, we can avoid the divergence entirely and replace it by a backwards integral (making use of the conversion formula (\ref{eq:conversion}) and the fact that the ensuing forward integral is zero in expectation). As hinted at in Remark \ref{rem:conversion}, this replacement might have favourable variance-reducing properties, but numerical evidence would be necessary.

\subsection{Score-based generative modeling}
\label{sec:SGM}

Generative modeling is concerned with the scenario where $\mu(\vx)$ can be sampled from (but its density is unknown), and the goal is to learn a backward diffusion as in (\ref{eq:back_sde}) that allows us to generate further samples from $\mu(\vx)$, see \citet{song2020score}. We may fix a reference forward drift $a_t$, and, motivated by Proposition \ref{prop:Nelson}, parameterise the backward drift as $b_t = a_t - s_t$, so that in the case when $\ora{\P}^{\mu,a} = \ola{\P}^{\nu,b}$, the variable drift component $s_t$ will represent the score $\sigma^2 \nabla \log \rho^{\mu,a}_t$. When the diffusion associated to $a_t$ is ergodic and $T$ is large, $\ora{\P}^{\mu,a} = \ola{\P}^{\nu,b}$ requires that $\nu(\vz)$ is close to the corresponding invariant measure. Choosing $\gamma_t^- = a_t$, and, for simplicity $\sigma=1$, direct calculations using Proposition \ref{prop:RND} show that
\begin{align}
\label{eq:ism}
\mathcal{L}_{\mathrm{ISM}} (s) := \KL(\ora{\P}^{\mu,a}|| \ola{\P}^{\nu,a - s}) 
= \mathbb{E}_{\fY \sim \ora{\P}^{\mu,a}} \left[ \tfrac{1}{2} \int_0^T s_t^2(\fY_t) \, \mathrm{d}t + \int_0^T (\nabla \cdot s_t)(\fY_t) \, \d t \right]
+ \mathrm{const.} 
\end{align} 
recovers the implicit score matching objective \citep{hyvarinen2005estimation}, up to a constant that does not depend on $s_t$.
\begin{proof}
We start by noticing that the contributions in (\ref{eq:boundary terms}) and (\ref{eq:forward path integral}) do not depend on $s_t$, and can therefore be absorbed in the constant in (\ref{eq:ism}) Notice that the precise forms of $\Gamma_0$, $\Gamma_T$ and $\gamma^+$ are left unspecified or unknown, but this does not affect the argument. We find 
\begin{align*}
\KL(\ora{\P}^{\mu,a}|| \ola{\P}^{\nu,a - s}) &=  \mathbb{E}_{\fY \sim \ora{\P}^{\mu,a}} \left[  \int_0^T s_t (\fY_t) \cdot \left( \ola{\d} \fY_t - \tfrac{1}{2}\left( 2a_t - s_t \right)(\fY_t) \, \d t\right) \right] + \mathrm{const.}
\\
&= \mathbb{E} \left[\int_0^T s_t (\fY_t) \cdot \left( \sigma\ola{\d} \fW_t + \tfrac{1}{2}s_t(\fY_t) \, \d t\right) \right] + \mathrm{const.} \\
&= \mathbb{E} \left[ \tfrac{1}{2} \int_0^T s_t^2(\fY_t) \, \mathrm{d}t + \int_0^T (\nabla \cdot s_t)(\fY_t) \, \d t \right] + \mathrm{const.}, 
\end{align*}
where in the first line we use Proposition \ref{prop:RND} together with $b_t = a_t - s_t$ and $\gamma_t^- = a_t$, and to proceed to the second line we substitute $\ola{\d}\fY_t$ using the SDE in (\ref{eq:fwd_sde}). The last equality follows from the conversion formula between forward and backward It{\^o} integrals, see (\ref{eq:conversion}), and the fact that forward integrals with respect to Brownian motion have zero (forward) expectation, see (\ref{eq:martingale}).    
\end{proof}

Notice that the nonuniqueness in Framework \ref{frm2} has been circumvented by fixing the forward drift $a_t$; indeed $\mathcal{L}_{\mathrm{ISM}}$ is convex in $s$, confirming Note that using integration by parts, $\mathcal{L}_{\mathrm{ISM}}$ is equivalent to denoising score matching \citep{song2020sliced, song2020score}: 
\begin{align}
   \KL(\ora{\P}^{\mu,a}|| \ola{\P}^{\nu,a - s}) =  \mathbb{E}_{\fY \sim \ora{\P}^{\mu,a}} \left[ \tfrac{1}{2 \sigma^2} \int_0^T \left\Vert s_t(\fY_t) - \nabla \ln \rho^{\mu, a}_{t|0}(\fY_t| \fY_0) \right\Vert^2 \,  \d t \right] + \mathrm{const.}. 
\end{align}
Framework \ref{frm2} accommodates  modifications of (\ref{eq:ism}); in particular the divergence term in (\ref{eq:ism}) can be replaced by a  backward integral, see Appendix \ref{app:discussion} and Remark \ref{rem:conversion}.  Note that the settings discussed in this section are also akin to the formulations in \citet{kingma2021variational,huang2021variational}.

Finally, it is worth highlighting that this setting is not limited to ergodic models and can in fact accommodate finite time models in the exact same fashion as the F\"ollmer drift is used for sampling (Section \ref{sec:Foellmer}) by using a Doob's transform \citep{rogers2000diffusions} based SDE for $\ora{\P}^{\mu,a}$ as opposed to the classical VP-SDE see Example 2.4 in \citet{ye2022first}.

\subsection{Score-based sampling}
\label{sec:sampling}

Consider the setting when $\nu(\vz)$ is a target distribution that can be evaluated pointwise up to a normalisation constant. In order to construct a diffusion process that transports an appropriate auxiliary distribution $\mu(\vx)$ to $\nu(\vz)$, one approach is to fix a  drift $b_t$ in the backward diffusion (\ref{eq:back_sde}), and then learn the corresponding forward diffusion (\ref{eq:fwd_sde}) by minimising $a \mapsto D(\ora{\P}^{\mu,a}|\ola{\P}^{\nu,b})$. Tractability of this objective requires that $\mu := \ola{\P}^{\nu,b}_0$ be known explicitly, at least approximately. In the following we review possible choices.

\paragraph{The F{\"ollmer drift}.}  \label{sec:Foellmer}
Choosing $b_t(x) = x/t $, one can show using Doob's transform \citep[Theorem 40.3(iii)]{rogers2000diffusions}, that $\ola{\P}^{\nu,b}_0(\vx) = \delta(\vx)$, for any terminal distribution $\nu(\vz)$. Hence, minimising $a \mapsto \KL(\ora{\P}^{\delta_0,a}|\ola{\P}^{\nu,b})$ leads to a tractable objective. In particular consider the choice $\Gamma_0 = \delta_0$, $\gamma^{+} = 0$, corresponding to a standard Brownian motion, then it follows that $\gamma^{-} = \frac{x}{t}$, $\Gamma_T = \gN(0, T \sigma^2)$ and thus via Proposition \ref{prop:RND}:
\begin{align}
\label{eq:Foellmer}
    \KL(\ora{\P}^{\delta_0,a}|\ola{\P}^{\nu,b}) = \mathbb{E} _{\fY \sim \ora{\P}^{\mu,a}}\!\left[\frac{1}{\sigma^2}\!\!\int_0^T \!\!a^2(\fY_t)\, \d t +\! \log\! \left(\frac{\d \gN(0, T \sigma^2)}{\d \nu}\right)\!\!(\fY_T) \right] + \mathrm{const.}, 
\end{align}
in accordance with \citep{dai1991stochastic,vargas2021bayesian,zhang2021path}. For further details, see \citet{follmer1984entropy, vargas2021bayesian, zhang2021path,huang2021schrodinger}. As hinted at in Appendix \ref{app:vi}, replacing $\KL$ in (\ref{eq:Foellmer}) by the log-variance divergence (\ref{eq:vargrad}) leads to an objective that directly links to BSDEs, see \citep[Section 3.2]{nusken2021solving}.


\paragraph{Ergodic diffusions.}

\citet{vargas2023denoising,berner2022optimal} fix a backward drift $b_t$ that induces an ergodic backward diffusion, so that for large $T$, the marginal at initial time $\ola{\P}^{\nu,b}_{t = 0}$ is close to the corresponding invariant distribution, and in particular (almost) independent of $\nu(\vz)$.\footnote{\citet{vargas2023denoising} chose a  (time-inhomoegenous) backward Ornstein-Uhlenbeck process, so that $\ola{\P}^{\nu,b}_{t = 0}$ is close to a Gaussian, but generalisations are straightforward.} Defining $\mu := \ola{\P}^{\nu,b}_{t = 0}$, \citet{vargas2023denoising,berner2022optimal} set out to minimise the denoising diffusion sampler loss  $\mathcal{L}_{\mathrm{DDS}}(f) := \KL(\ora{\P}^{\mu,b+\sigma^2 f} | \ola{\P}^{\nu,b})$. Choosing the reference process to be $\Gamma_{0,T} = \mu$, $\gamma^{\pm} = b$ (that is, the reference process is at stationarity, with invariant measure $\mu(\vz))$, direct calculation based on (\ref{eq:RND fb}) shows that
\begin{align}
\label{eq:DDS}
\mathcal{L}_{\mathrm{DDS}}(f) =
\mathbb{E} _{\fY \sim \ora{\P}^{\mu,b+ \sigma^2 f}}\!\left[ \sigma^2 \!\!\int_0^T \!\!f^2(\fY_t)\, \d t +\! \log\! \left(\frac{\d \Gamma_T}{\d \nu}\right)\!\!(\fY_T) \right],  
\end{align}

\begin{remark}[IWAE-objective]
\label{rem:IWAE sampling}
In line with (\ref{eq:IWAE KL}), we may also consider the multi-sample objective \begin{align*}
\mathcal{L}^{(K)}_{\mathrm{DDS}}(f) & := \KL^{(K)}(\ora{\P}^{\mu,b+\sigma^2 f} | \ola{\P}^{\nu,b})
\\
& = \mathbb{E} _{\fY^1,\ldots,\fY^K \overset{iid}{\sim} \ora{\P}^{\mu,b+ \sigma^2 f}}\!\left[ \log \left(\tfrac{1}{K}\sum_{i=1}^K\exp\left(\sigma^2 \!\!\int_0^T \!\!f^2(\fY^i_t)\, \d t +\! \log\! \left(\frac{\d \Gamma_T}{\d \nu}\right)\!\!(\fY^i_T) \right)\right)\right]
\end{align*}
\end{remark}

\begin{proof}
We start by noticing that the choice $\gamma^-_t = b_t$ cancels the terms in (\ref{eq:backward path integral}), and the choice $\Gamma_0 = \mu$ cancels the first term in (\ref{eq:boundary terms}). Using $a_t = b_t +  \sigma^2 f_t$, we therefore obtain
\begin{subequations}
\begin{align}
\mathcal{L}_{\mathrm{DDS}}(f) &= \KL(\ora{\P}^{\mu,b+ \sigma^2 f} | \ola{\P}^{\nu,b}) \\
& = \mathbb{E} \left[\sigma^2 \int_0^T f_t(\fY_t) \cdot \left( (b_t + f_t)(\fY_t) \, \d t - \tfrac{1}{2}(2b_t + f_t)(\fY_t) \, \d t\right) + \log \left( \frac{\d \Gamma_T}{\d \nu}\right)(\fY_T)  \right]
\nonumber
\\
& = \mathbb{E} \left[ \sigma^2 \int_0^T f^2_t(\fY_t)\, \d t + \log \left( \frac{\d \Gamma_T}{\d \nu}\right)(\fY_T)\right].
\end{align}
\end{subequations}
As is implicit in \citet{berner2022optimal}, it is also possible to choose $\gamma^{\pm} = 0$ for the reference process, with $\Gamma_0 = \Gamma_T = \mathrm{Leb}$, the Lebesgue measure on $\mathbb{R}^d$. We notice in passing that although the Lebesgue measure is not normalisable, it is invariant under Brownian motion (the forward and backward drifts are both zero), and the arguments can be made rigorous by a limiting argument (take Gaussians with diverging variances), or by using the techniques in \citet[Appendix A.1]{leonard2013survey}.
By similar calculations as above, we obtain 
\begin{subequations}
\begin{align}
\mathcal{L}_{\mathrm{DDS}}(f) & = \mathbb{E} \left[  \sigma^2 \int_0^T f^2_t(\fY_t)\, \d t - \tfrac{1}{\sigma} \int_0^T b_t(\fY_t) \cdot \ola{\d} \fW_t + \log \mu(\fY_0) - \log \nu(\fY_T)\right]
\\
\label{eq:DDS2}
& = \mathbb{E} \left[  \sigma^2 \int_0^T f^2_t(\fY_t)\, \d t - \int_0^T (\nabla \cdot b_t)(\fY_t) \, \d t - \log \nu(\fY_T)\right] + \mathrm{const.},
\end{align}
\end{subequations}
where we overload notation and denote the Lebesgue densities of $\mu$ and $\nu$ with the same letters. 
In the second line we have used the conversion  formula in (\ref{eq:conversion}), together with the fact that the forward It{\^o} integrals are forward martingales \citep{kunita2019stochastic}, and therefore have zero expectation. Comparing (\ref{eq:DDS}) and (\ref{eq:DDS2}), we notice the additional divergence term, due to the fact that the choice $\gamma^- = 0$ does not cancel the terms in (\ref{eq:RN backward}). See also the discussion in Appendix \ref{app:discussion}.     
\end{proof}

Finally we note that whilst the work in \citet{berner2022optimal} focuses on exploring a VP-SDE-based approach which is ergodic, their overarching framework generalises beyond ergodic settings, notice this objective is akin to the KL expressions in \citet[Proposition 1]{vargas2021machine} and \citet[Proposition 9]{liu2022deep}.



\subsection{Action matching \citep{neklyudov2022action}}
\label{app:action matching}
Similar to our approach in Section \ref{sec:annealing}, \citet{neklyudov2022action} fix a curve of distributions $(\pi_t)_{t \in [0,T]}$. In contrast to us, they assume that samples from $\pi_t$ are available, for each $t \in [0,T]$ (but scores and unnormalised densities are not). Still, we can use Framework \ref{frm2} to rederive their objective:  

Akin to the proof of Proposition \ref{prop:annealing uniqueness}, under mild conditions on $(\pi_t)_{t \in [0,T]}$, there exists a unique vector field $\nabla\phi^*_t$ that satisfies the Fokker-Planck equation 
\begin{equation}
\partial_t \pi_t + \nabla \cdot (\pi_t \nabla \phi^*_t) = \tfrac{\sigma^2}{2}\Delta \pi_t.     
\end{equation}
We can now use the reference process $\ora{\P}^{\pi_0,\nabla \phi^*}$(that is, $\Gamma_0 = \pi_0$, $\gamma_t^+ = \nabla \phi_t^*$, $\Gamma_T = \pi_T$, $\gamma_t^- = \nabla \phi_t^* - \sigma^2 \nabla \log \pi_t$) to compute the objective
$$\psi \mapsto \KL(\ora{\P}^{\pi_0,\nabla \psi}|| \ola{\P}^{\pi_T, \nabla \psi - \sigma^2 \nabla \log \pi}),$$
relying on the same calculational techniques as in Sections \ref{sec:SGM} and \ref{sec:sampling} (the particular choice of reference process cancels the terms in (\ref{eq:boundary terms})). Notice that the parameterisation in this objective constrains the target diffusion to have time-marginals $\pi_t$, just as in Section \ref{sec:annealing}. By direct calculation, we obtain (up to a factor of $2/\sigma^2$) the action-gap in equation (5) in \citet{neklyudov2022action}. Indeed, we see that
\begin{subequations}
\begin{align*}
 \KL(\ora{\P}^{\pi_0,\nabla \psi}|| &\ola{\P}^{\pi_T, \nabla \psi - \sigma^2 \nabla \log \pi})   = \mathbb{E}_{\ora{\P}^{\pi_0,\nabla \psi}} \left[ \log \left( \frac{\d \ora{\P}^{\pi_0,\nabla \psi}}{\d \ola{\P}^{\nabla \psi - \sigma^2 \nabla \log \pi}}\right) \right]
\\
& = \mathbb{E} \Bigg[ \tfrac{1}{\sigma^2} \int_0^T \left( \nabla \psi_t - \nabla \phi^*_t \right)^2 (\fY_t) \, \d t - \tfrac{1}{\sigma} \int_0^T (\nabla \psi_t - \nabla \phi_t^*)(\fY_t) \cdot \ola{\d} \fW_t \\
&\quad\quad \quad - \int_0^T \nabla \log \pi_t(\fY_t) \cdot (\nabla \psi_t - \nabla \phi^*_t)(\fY_t)\, \d t\Bigg]
\\
& = \mathbb{E} \left[ \tfrac{1}{\sigma^2} \int_0^T \left( \nabla \psi_t - \nabla \phi^*_t \right)^2 (\fY_t) \, \d t \right], 
\end{align*}
\end{subequations}
where in the last line we have used the conversion formula (\ref{eq:conversion}) together with  (\ref{eq:martingale}) to compute
\begin{subequations}
\begin{align*}
\mathbb{E} \left[\tfrac{1}{\sigma} \int_0^T (\nabla \psi_t - \nabla \phi_t^*)(\fY_t) \cdot \ola{\d} \fW_t \right]
& = \mathbb{E} \left[ \int_0^T (\nabla \cdot (\nabla \psi_t - \nabla \phi_t^*))(\fY_t) \, \d t \right] 
\\
=  \int_0^T  \int_{\mathbb{R}^d}(\nabla \cdot (\nabla \psi_t - \nabla \phi_t^*))(\vx) \pi_t(\d \vx)\, \d t &  = - \int_0^T  \int_{\mathbb{R}^d}(\nabla \psi_t - \nabla \phi_t^*)(\vx)\cdot \nabla \log \pi_t (\vx) \pi_t(\d \vx)\, \d t  
\\
& = \mathbb{E} \left[\int_0^T \nabla \log \pi_t(\fY_t) \cdot (\nabla \psi_t - \nabla \phi^*_t)(\fY_t)\, \d t\right]
\end{align*}
\end{subequations}
and cancel the two last terms in the penultimate line.
\section{Controlled Monte Carlo Diffusions (Section \ref{sec:annealing})}
\label{app:annealing}

\subsection{Derivation of $\mathcal{L}^{\mathrm{CMCD}}_{\KL}$}
\label{app:cmcd kl}
The proof uses Proposition \ref{prop:RND}, choosing $\Gamma_0 = \Gamma_T$ to be the Lebesgue measure, with $\gamma^+ = \gamma^- = 0$ (but notice that $\sigma$ in (\ref{eq:RND fb}) needs to be replaced by $\sigma \sqrt{2}$ due to the scaling in (\ref{eq:controlled annealing})). We compute
\begin{align}
\nonumber
& \mathcal{L}^{\mathrm{CMCD}}_{\KL}\!(\phi) = \E_{\fY \sim \ora{\P}^{\pi_0,\sigma^2 \nabla \log \pi + \nabla \phi}} \left[ \log \left( \frac{\d\ora{\P}^{\pi_0,\sigma^2 \nabla \log \pi + \nabla \phi}}{\d \ola{\P}^{\pi_T,-\sigma^2 \nabla \log \pi + \nabla \phi}} \right) (\fY) \right]  \\
\nonumber
 = & \E \left[ \log \pi_0(\fY_0) -\log \pi_T(\fY_T) \right] \\
\nonumber
& + \E \left[ \tfrac{1}{2 \sigma^2} \int_0^T (\sigma^2 \nabla \log \pi_t + \nabla \phi_t)(\fY_t) \cdot \left( \ora{\d} \fY_t - \tfrac{1}{2}(\sigma^2 \nabla \log \pi_t + \nabla \phi_t)(\fY_t) \, \d t  \right) \right]
\\
\nonumber
& - \E \left[ \tfrac{1}{2 \sigma^2} \int_0^T (-\sigma^2 \nabla \log \pi_t + \nabla \phi_t)(\fY_t) \cdot \left( \ola{\d} \fY_t - \tfrac{1}{2}(-\sigma^2 \nabla \log \pi_t + \nabla \phi_t)(\fY_t) \, \d t  \right) \right]
\\
\nonumber
 = & \E \left[ \log \pi_0(\fY_0) -\log \pi_T(\fY_T) \right]
\\
\nonumber
& + \E \left[ \tfrac{1}{2 \sigma^2} \int_0^T (\sigma^2 \nabla \log \pi_t + \nabla \phi_t)(\fY_t) \cdot \ora{\d} \fY_t \right] - \E \left[ \tfrac{1}{2 \sigma^2} \int_0^T (-\sigma^2 \nabla \log \pi_t + \nabla \phi_t)(\fY_t) \cdot \ola{\d} \fY_t \right]
\\
\nonumber
& - \tfrac{1}{\sigma^2} \E \left[ \int_0^T (\sigma^2 \nabla \log \pi_t \cdot \nabla \phi_t) (\fY_t) \, \d t \right]
\\
\nonumber
= & \E \left[ \log \pi_0(\fY_0) -\log \pi_T(\fY_T) \right] 
\\
\nonumber
& +
\E \left[ \sigma^2  \int_0^T |\nabla \log \pi_t(\fY_t)|^2  \d t  + \tfrac{1}{\sigma \sqrt{2}}  \int_0^T  \left(\sigma^2 \nabla \log \pi_t - \nabla \phi_t \right)(\fY_t) \cdot \ola{\d} \fW_t \right], 
\end{align}
where in the last equality we have inserted the dynamics (\ref{eq:controlled annealing}) and used the martingale property (\ref{eq:martingale}). Notice that the expectation of the backward integral is not zero, see Appendix \ref{app:stochana}.
\subsection{Derivation of $\mathcal{L}^{\mathrm{CMCD}}_{\mathrm{Var}}$}

In this section, we first verify the expression for $\mathcal{L}^{\mathrm{CMCD}}_{\mathrm{Var}}$ in Section \ref{sec:annealing}, using Proposition \ref{prop:RND}, and choosing $\Gamma_0 = \Gamma_T$ to be the Lebesgue measure, $\gamma^+ = \gamma^- = 0$. We recall that although the Lebesgue measure in not normalisable, the arguments can be made rigorous using the techniques in \citet[Appendix A]{leonard2013survey}.

The Radon-Nikodym derivative (RND) along (\ref{eq:controlled annealing}) reads
\begin{subequations}
\begin{align}
&\left(\log \frac{\d \ora{\P}^{\pi_0,\sigma^2 \nabla \log \pi + \nabla \phi}}{\d \ola{\P}^{\pi_T, - \sigma^2 \nabla \log \pi + \nabla \phi}} \right)(\fY)  = (\log \pi_0) (\fY_0) - (\log \pi_T) (\fY_T) 
\nonumber
\\
& +\!\tfrac{1}{2 \sigma^2} \!\!\int_0^T \!\!\!\!\!(\sigma^2 \nabla \!\log \pi_t \!+\! \nabla \phi_t)(\fY_t) \!\! \left(\!(\!\sigma^2 \nabla \!\log \pi_t \!+\! \nabla \phi_t)(\fY_t) \, \d t \!+\! \sqrt{2 }\sigma  \ora{d} \fW_t\! -\! \tfrac{1}{2} (\sigma^2 \nabla\! \log \pi_t \! + \!\nabla \phi_t)(\fY_t) \, \d t \!\right)
\nonumber
\\
&-\!\tfrac{1}{2 \sigma^2} \!\!\!\int_0^T \!\!\!\!\!\!(\!-\sigma^2 \nabla\! \log \pi_t\! +\! \nabla \phi_t)(\fY_t) \!\!\left(\!\! (\!\sigma^2 \nabla \!\log \pi_t\! + \!\nabla \phi_t)(\fY_t) \, \d t \!+\! \sqrt{2 } \sigma \ola{d} \fW_t \!-\! \tfrac{1}{2} \!(\!\nabla \phi_t \!-\!\sigma^2 \nabla \log \pi_t  )(\fY_t) \, \d t\! \right)
\nonumber
\\
& = (\log \pi_0) (\fY_0) - (\log \pi_T) (\fY_T) \nonumber
+ \sigma^2 \int_0^T | \nabla \log \pi_t(\fY_t)|^2 \, \d t
\\
& \quad\quad\quad+ {\tfrac{\sigma}{\sqrt{2}}} \left(\int_0^T \nabla \log \pi_t(\fY_t) \cdot \ora{\d} \fW_t + \int_0^T \nabla \log \pi_t(\fY_t) \cdot \ola{\d} \fW_t \right)\nonumber\\
&\quad\quad\quad+ \tfrac{1}{\sigma \sqrt{2 }} \left(\int_0^T \nabla \phi(\fY_t) \cdot \ora{\d} \fW_t - \int_0^T \nabla \phi(\fY_t) \cdot \ola{\d} \fW_t \right).\nonumber
\end{align}
\end{subequations}
Using (\ref{eq:Strato}), we obtain
\begin{equation*}
{\tfrac{\sigma}{\sqrt{2} }}\left(\int_0^T \nabla \log \pi_t(\fY_t) \cdot \ora{\d} \fW_t + \int_0^T \nabla \log \pi_t(\fY_t) \cdot \ola{\d} \fW_t \right) = \sqrt{2}\sigma \int_0^T \nabla \log \pi_t(\fY_t) \circ \d \fW_t. 
\end{equation*}
Furthermore, from (\ref{eq:conversion}) we see that 
\begin{equation}
\label{eq:Laplace approx}
\tfrac{1}{\sigma \sqrt{2 }} \left(\int_0^T \nabla \phi(\fY_t) \cdot \ora{\d} \fW_t - \int_0^T \nabla \phi(\fY_t) \cdot \ola{\d} \fW_t \right) = -\int_0^T \Delta \phi_t(\fY_t) \, \d t,
\end{equation}
from which the claim follows.

\begin{remark}[Estimating $\mathcal{L}^{\mathrm{CMCD}}_{\mathrm{Var}}$ without second derivatives]
\label{rem:kick out Laplace}
Using (\ref{eq:Laplace approx}), we can equivalently write the RND as 
\begin{subequations}
\begin{align}
\nonumber
& \left(\log \frac{\d \ora{\P}^{\pi_0,\sigma^2 \nabla \log \pi + \nabla \phi}}{\d \ola{\P}^{\pi_T, - \sigma^2 \nabla \log \pi + \nabla \phi}} \right)(\fY)  =
\log \pi_T(\fY_T) - \log \pi_0(\fY_0) 
\\
& - \tfrac{1}{\sigma \sqrt{2 }} \left(\int_0^T \nabla \phi(\fY_t) \cdot \ora{\d} \fW_t - \int_0^T \nabla \phi(\fY_t) \cdot \ola{\d} \fW_t \right) \nonumber
\\
&-\!\! \sigma \sqrt{2 } \!\int_0^T\!\!\! \!\!\nabla \log \pi_t(\fY_t)\! \circ\! \d \fW_t\! - \sigma^2\! \!\int_0^T\!\!\! |\nabla \log \pi_t(\fY_t)|^2 \, \d t,\!\! \nonumber
\end{align}
\end{subequations}
so that $\mathcal{L}^{\mathrm{CMCD}}_{\mathrm{Var}}$ can be estimated without the need to evaluate $\Delta \phi$. Note that the identity (\ref{eq:Laplace approx}) is similar to a finite difference approximation of $\Delta \phi$  along the process $\fY_t$. 
\end{remark}

\subsection{Existence and uniqueness of the drift}
\label{app:drift prop}

Before proceeding to the proof of Propostion \ref{prop:annealing uniqueness}, we state the following assumption on the curve of distributions $(\pi_t)_{t \in [0,T]}$:

\begin{assumption}
\label{ass:annealing}
Assume that $\pi \in C^{\infty}([0,T]\times \mathbb{R}^d; \mathbb{R})$, and that for all $t \in [0,T]$ 
\begin{enumerate}
\item the time derivative $\partial_t \pi_t$ is square-integrable, that is, $\partial_t \pi_t(t,\cdot) \in L^2(\mathbb{R}^d)$,
\item $\pi_t$ satisfies a Poincar{\'e} inequality, that is, there exists a constant $C_t>0$ such that
\begin{equation}
\label{eq:Poincare}
\Var_{\pi_t}(f) \le C_t \int_{\mathbb{R}^d} |\nabla f|^2 \d \pi_t,    
\end{equation}
for all $f \in C_b^1(\mathbb{R}^d)$. 
\end{enumerate}
\end{assumption}
Note that at the boundary $\partial [0,T] = \{0,T\}$, we agree to denote by $\partial_t \pi_t$ the `inward-pointing derivative' and interpret $C^{\infty}([0,T]\times \mathbb{R}^d; \mathbb{R})$ in that way.
We remark that the Poincar\'e inequality (\ref{eq:Poincare}) is satisfied under relatively mild conditions on the tails of $\pi_t$  (for instance, Gaussian tails) and control of its derivatives, see, e.g., \citet[Chapter 4]{bakry2014analysis}. Under Assumption \ref{ass:annealing}, we can prove Proposition \ref{prop:annealing uniqueness} as follows:

\begin{proof}[Proof of Proposition \ref{prop:annealing uniqueness}] The Fokker-Planck equation associated to (\ref{eq:controlled annealing}) is given by 
\begin{equation}
\label{eq:continuity}
\partial_t \pi_t + \nabla \cdot (\pi_t \nabla \phi_t) = 0.
\end{equation}
The operator $\phi \mapsto - \nabla \cdot (\pi_t \nabla \phi)$ is essentially self-adjoint in $L^2(\mathbb{R}^d)$, and, by (\ref{eq:Poincare}) coercive on $L_0^2(\mathbb{R}^d) := \{f \in L^2(\mathbb{R}^d) : \,\, \int f \d x = 0\}$. Therefore, there exists a unique solution $\phi_t^* \in L^2_0(\mathbb{R}^d)$ to (\ref{eq:continuity}), for any $t \in [0,T]$. This solution is smooth by elliptic regularity. By Proposition \ref{prop:Nelson} and our general framework, any minimiser $\widetilde{\phi}$ of  (\ref{eq:annealed_div}) necessarily satisfies (\ref{eq:continuity}) as well. We then obtain
\begin{equation}
\nonumber
\nabla \cdot (\pi_t \nabla (\phi_t - \widetilde{\phi_t})) = 0.
\end{equation}
Multiplying this equation by $\phi_t - \widetilde{\phi}_t$, integrating, and integrating by parts shows that $\int \Vert \nabla (\phi - \widetilde{\phi})\Vert^2 \, \d \pi_t = 0$, proving the claim.
\end{proof}

\begin{remark}[Relation to previous work]
\label{rem:AFT}
Note we can carry out a change of variables to equation (\ref{eq:continuity}), 
\begin{align}
\nonumber
    \partial_t \ln  \pi_t  = -\pi_t^{-1}( \nabla \pi_t  \cdot \nabla \phi_t + \pi_t \Delta \phi ) = -\nabla \ln \pi_t \cdot  \nabla \phi  - \Delta \phi,
\end{align}
yielding the PDE
\begin{align}
\nonumber
    \partial_t \ln  \pi_t   +\nabla \ln \pi_t \cdot  \nabla \phi  + \Delta \phi = 0,
\end{align}
which when considered in terms of the unnormalised flow $\hat{\pi}_t = Z_t \pi_t$ coincides with PDE in \citet{vaikuntanathan2008escorted,arbel2021annealed}:
\begin{align}
\nonumber
    \partial_t \ln  \hat{\pi}_t   +\nabla \ln \hat{\pi}_t \cdot  \nabla \phi  + \Delta \phi -  \E_{\pi_t}[\partial_t \ln \hat{\pi}_t] = 0.
\end{align}
In particular, we note that the Markov chain proposed in \citet{arbel2021annealed} converges to our proposed parameterisation in equation (\ref{eq:controlled annealing}) (see equation (12) in \cite{arbel2021annealed}).
\end{remark}

\subsection{Infinitesimal Schr\"odinger bridges (proof of Proposition \ref{prop:sch inf})}
\label{app:infi schr}

Throughout this proof, we assume that the Schr\"odinger problems on the intervals $[iT/N,(i+1)T/N]$, $i=0,\ldots,N-1$ admit unique solutions, with drifts of regularity specified in Assumption \ref{ass:vector fields}, see \cite[Proposition 2.5]{leonard2013survey} for sufficient conditions. We also work under Assumption \ref{ass:annealing}, so that the drift $\nabla \phi^*$ exists and is unique by Proposition \ref{prop:annealing uniqueness}.

Given the interpolation $(\pi_t)_{t \in [0,T]}$, we define the constraint sets

\begin{subequations}
\label{eq:UN}
\begin{align}
\mathcal{M}^N(\pi) := \Bigg\{ & a \in \mathcal{U}^N: \quad \ora{\P}^{\pi_0, \nabla \log \pi + a}_{t_i} = \pi_{t_i} \,\,
& \text{at times } \quad t_i = \tfrac{iT}{N}, \quad i = 0,\ldots,N  \Bigg\},
\nonumber
\tag{\ref*{eq:UN}}
\end{align}
\end{subequations}
as well as
\begin{subequations}
\label{eq:Cinfty}
\begin{align}
\mathcal{M}^\infty(\pi) := \Bigg\{ & a \in \mathcal{U}: \,\, \ora{\P}^{\pi_0, \nabla \log \pi + a}_{t} = \pi_{t} \quad \text{for all } \quad t \in [0,T]  \Bigg\}.
\end{align}
\end{subequations}
In (\ref{eq:UN}), the set $\mathcal{U}^N$ is given by
\begin{subequations}
\begin{align*}
\mathcal{U}^N := \Bigg\{a \in C([0,T] \times&\mathbb{R}^d;\mathbb{R}^d): \quad  a \in C^{\infty}([\tfrac{iT}{N},\tfrac{(i+1)T}{N}] \times \mathbb{R}^d;\mathbb{R}^d), \quad \text{for all } i=0,\ldots,N-1,
\\
\quad   \exists L >0 
& \text{ such that } \Vert a_t(\vx) - a_t(\vy)\Vert \le L \Vert \vx - \vy \Vert, \, \text{for all } t \in [0,T], \,\, \vx,\vy \in \mathbb{R}^d \Bigg\},
\end{align*}
\end{subequations}
and we recall that $\mathcal{U}$ has been defined in Assumption \ref{ass:vector fields}. 

By the construction in Proposition \ref{prop:sch inf}, the drift $\nabla 
\phi^{(N)}$ can be characterised by
\begin{subequations}
\label{eq:phiN}
\begin{align}
\nabla \phi^{(N)} & \in \argmin_{a \in \mathcal{M}^N(\pi)} \E_{\fY \sim \ora{\P}^{\pi_0,\nabla \log \pi + a}} \left[ \tfrac{1}{2 \sigma^2} \int_0^T \Vert a_t (\fY)\Vert^2 \, \d t\right]
\\
&  = \argmin_{a \in  \mathcal{M}^N(\pi)} \KL(\ora{\P}^{\pi_0,\nabla \log \pi + a }| \ora{\P}^{\pi_0,\nabla \log \pi}),
\end{align}
\end{subequations}
where the second line follows from Girsanov's theorem, see the proof of Proposition \ref{prop:RND}.

We now claim that the CMCD drift $\nabla \phi^*$, by definition the minimiser in (\ref{eq:annealed_div}), can be characterised in a similar way by
\begin{subequations}
\label{eq:phiinf}
\begin{align}
\label{eq:phistar}
\nabla \phi^* & \in \argmin_{a \in \mathcal{M}^\infty(\pi)} \E_{\fY \sim \ora{\P}^{\pi_0,\nabla \log \pi + a}} \left[ \tfrac{1}{2 \sigma^2} \int_0^T \Vert a_t (\fY)\Vert^2 \, \d t\right]
\\
&  = \argmin_{a \in  \mathcal{M}^\infty(\pi)} \KL(\ora{\P}^{\pi_0,\nabla \log \pi + a }| \ora{\P}^{\pi_0,\nabla \log \pi}).
\end{align}
\end{subequations}
Indeed, the constraint $\ora{\P}^{\pi_0, \nabla \log \pi + a}_{t} = \pi_{t}$ for all $t \in [0,T]$ implies that $a$ satisfies the Fokker-Planck equation $\partial_t \pi_t + \nabla \cdot (\pi_t a_t) = 0$. By the Helmholtz decomposition \cite[Section 2.5.4]{figalli2021invitation}, minimisers of  $a_t \mapsto \int a_t^2 \d \pi_t$ are of gradient form, thus (\ref{eq:phistar}) holds.



Comparing (\ref{eq:phiN}) and (\ref{eq:phistar}), it is plausible to infer the convergence $\nabla \phi^{(N)} \rightarrow \nabla \phi^*$, as the marginal constraints at the discrete time points $0,1/T, 2/T,...,T$ become dense and approach the continuous-time constraint in (\ref{eq:Cinfty}). 

To make this more precise, we note that since $\mathcal{M}^\infty(\pi) \subset \mathcal{M}^N(\pi)$ for all $N \in \mathbb{N}$, we have that 
\begin{equation}
\KL(\ora{\P}^{\pi_0,\nabla \log \pi + \nabla \phi^{(N)} }| \ora{\P}^{\pi_0,\nabla \log \pi}) \le \KL(\ora{\P}^{\pi_0,\nabla \log \pi + \nabla \phi^* }| \ora{\P}^{\pi_0,\nabla \log \pi}),   
\end{equation}
for all $N \in \mathbb{N}$. Since $\KL(\cdot| \ora{\P}^{\pi_0,\nabla \log \pi})$ has weakly compact sublevel sets \cite[Lemma 1.4.3]{dupuis2011weak}, we can extract a subsequence $\ora{\P}^{\pi_0,\nabla \log \pi + \nabla \phi^{(N_k)}}$ that converges weakly towards a path measure  $\widetilde{\P} \in \mathcal{P}(C([0,T];\mathbb{R}^d))$. To show that indeed  $\widetilde{\P} = \ora{\P}^{\pi_0,\nabla \log \pi + \nabla \phi^{*}}$, it is sufficient to note that by the constraints in (\ref{eq:UN}) those measures necessarily have the same finite-dimensional marginals, and to combine this observation with the continuity statement of Theorem 2.7.3 in \cite{billingsley2013convergence}, as well as the uniqueness from Proposition \ref{prop:annealing uniqueness}. The convergence of the drifts in the sense of $L^2([0,T]\times \mathbb{R}^d;\mathbb{R}^d)$ now follows from the lower semicontinuity of $\KL$ in combination with Girsanov's theorem.




\begin{algorithm}[t]
\begin{algocolor}
\caption{ Controlled Monte Carlo Diffusions - Training}
\label{alg:Fldvi3}
\begin{algorithmic}
\Require $\pi_0$, $\pi_T$, $\pi_t$, $\sigma$, K step-sizes $\Delta t_{k}$,  network $f^{\phi}$
\For{$i$  in $\mathrm{epochs}$}
\State $ \ln \mW_T, \mY_T \sim$ \textbf{Algorithm \ref{alg:Fldvi2}}($\pi_0$, $\pi_T$, $\pi_t$, $\sigma$, $\{\Delta t_{k}\}_k$, $f^{\phi}$)
\State Gradient descent step $\nabla_\phi -\ln \mW_T$
\EndFor
\State \Return $f^\phi$
\end{algorithmic}
\end{algocolor}
\end{algorithm}

\subsection{Discretisation and Objective}
\label{app:discr}

In the setting of CMCD with KL divergence we can use the  EM approximations to the RND presented in Proposition \ref{prop:euler} to express the objective as:
\begin{align}
\mathcal{L}^{\mathrm{CMCD}}_{\KL}(\phi)  \! \approx\! \E\!\left[\!\ln\! \frac{{\pi}_0(\mY_0)}{\hat{\pi}(\mY_T)} \!\!\prod_{k=0}^{K-1} \!\!\frac{\gN(\mY_{t_{k+1}} | \mY_{t_{k}} + (\sigma^2\nabla \ln \pi_{t_k} + \nabla \ln \phi_{t_k}) (\mY_{t_{k}})\Delta t_k, 2 \sigma^2 \Delta t_k )}{\gN\!(\mY_{t_{k}} | \mY_{t_{k+1}}\!\! \!+ \!(\sigma^2 \nabla\! \ln \pi_{t_{k+1}}\! \!\!- \!\nabla \ln \phi_{t_{k+1}})(\mY_{t_{k+1}})\Delta t_k, 2 \sigma^2 \Delta t_k )}\!\!\right],
\end{align}
where the expectation is taken wrt to the EM approximation of the SDE in (\ref{eq:controlled annealing}), that is:
\begin{align}
  \mY_{t_{k+1}}  \sim  \gN( \mY_{t_{k}} + (\sigma^2 \nabla \ln \pi_{t_k} + \nabla \ln \phi_{t_k}) (\mY_{t_{k}})\Delta t_k,\; 2 \sigma^2 \Delta t_k ).
\end{align}

\subsection{Underdamped Langevin Dynamics}
\label{sec:underdamped}
In this section, we motivate the underdamped generalisation of CMCD which is used across our experiments. This parameterisation is inspired by the underlying theory for the overdamped approach, and we leave a rigorous extension of those foundations for future work. However, we have found this heuristic parameterisation to perform very well empirically.

Following \cite{geffner2023langevin} we parametrise as:
\begin{align}
\label{eq:controlled annealing2}
\fY_0, \fZ_0 &\sim   \gN(0,I) \otimes\!\pi_0, \nonumber \\
\d \fZ_t  &= \fY_t \d t , \nonumber \\
\d \fY_t  &\!=\! \left( \sigma^2\nabla \log \pi_t(\fZ_t)\! - \sigma^2 \fY_t +\! \nabla \phi_t (\fY_t, \fZ_t)\right) \d t + \sigma \!\sqrt{2 } \,\ora{\d} \fW_t. 
\end{align}
and it's time reversal as: 
\begin{align}
\label{eq:controlled annealing2rev}
\fY_T, \fZ_T &\sim  \gN(0,I) \otimes \!\pi_T , \nonumber \\
\d \fZ_t  &= \fY_t \d t , \nonumber \\
\d \fY_t  &\!=\! \left( -\sigma^2\nabla \log \pi_t(\fZ_t)\! + \sigma^2 \fY_t +\! \nabla \phi_t (\fY_t, \fZ_t)\right) \d t + \sigma \!\sqrt{2 } \,\ola{\d} \fW_t. 
\end{align}

\subsubsection{Time Discrteisation and Objective}
To discretise the above processes we follow the exact same discretisation scheme carried out in  \cite{geffner2023langevin}, however in this case we have to adapt the forward discretisation scheme to include the non-linear drift when carrying out the momentum re-sample step, specific details for this scheme can be found in Algorithms \ref{alg:Fldvi} and \ref{alg:Bldvi}. This discretisation in turn allows us to compute the discrete RND between these two processes which we require for Framework \ref{frm2}.




\begin{algorithm}[t]
\caption{Forward transition $F_{t_k}(\fZ_{t_{k+1}}, \fY_{t_{k+1}} \vert \fZ_{t_{k}}, \fY_{t_{k}})$}
\label{alg:Fldvi}
\begin{algorithmic}
\Require $\fZ_{t_{k}}$, $\fY_{t_{k}}$, step-size $\Delta t_{k}$
\State Re-sample momentum $\fY_{t_{k}}'  \sim \mathcal{N}\left(  \fY_{t_{k}}(1-\sigma \Delta {t_k}) + \nabla \phi_{t_k}( \fY_{t_{k}} ,\fZ_{t_{k}}) \Delta t_k , \;2 \sigma \Delta t_{k} I\right)$\\
\hspace{-0.21cm}$\left.\begin{array}{l} \vspace{0.05cm}
\mbox{Update} \, \fY_{t_{k}}'' = \fY_{t_{k}}' + \frac{\Delta t_{k}}{2} \nabla \log \pi_{t_k}(\fZ_{t_{k}})\\\vspace{0.05cm}
\mbox{Update} \, \fZ_{t_{k+1}} = \fZ_{t_{k}} + \Delta t_{k} \fY_{t_{k}}''\\
\mbox{Update} \, \fY_{t_{k+1}} = \fY_{t_{k}}'' + \frac{\Delta t_{k}}{2} \nabla \log \pi_{t_k}(\fZ_{t_{k+1}})
\end{array}\right\}
\left . \begin{array}{l}
\mbox{Leapfrog step}\\
\Phi(\fZ_{t_{k}}, \fY_{t_{k}}')
\end{array}\right .$
\State \Return $(\fZ_{t_{k+1}}, \fY_{t_{k+1}})$
\end{algorithmic}
\end{algorithm}


\begin{algorithm}[t]
\caption{Backward transition $B_{t_k}(\fZ_{t_{k}}, \fY_{t_{k}} \vert \fZ_{t_{k+1}}, \fY_{t_{k+1}})$}
\label{alg:Bldvi}
\begin{algorithmic}
\Require $\fZ_{t_{k+1}}$, $\fY_{t_{k+1}}$, step-size $\Delta t_{k}$
\State
\hspace{-0.21cm}$\left.\begin{array}{l} \vspace{0.05cm}
\mbox{Update} \, \fY_{t_{k}}'' = \fY_{t_{k+1}} - \frac{\Delta t_{k}}{2} \nabla \log \pi_{t_k}(\fZ_{t_{k}})\\\vspace{0.05cm}
\mbox{Update} \, \fZ_{t_{k}} = \fZ_{t_{k+1}} - \Delta t_{k} \fY_{t_{k}}''\\
\mbox{Update} \, \fY_{t_{k}}' = \fY_{t_{k}}'' - \frac{\Delta t_{k}}{2} \nabla \log \pi_{t_k}(\fZ_{t_{k+1}})
\end{array}\right\}
\left . \begin{array}{l}
\mbox{Inverse leapfrog}\\
\Phi^{-1}(\fZ_{t_{k+1}}, \fY_{t_{k+1}})
\end{array}\right .$
\State Re-sample momentum $\fY_{t_{k}}  \sim \mathcal{N}\left(\fY_{t_{k}}'(1-\sigma \Delta {t_k}) - \nabla \phi_{t_k}( \fY_{t_{k}}' ,\fZ_{t_{k}}) \Delta t_k , \;2 \sigma \Delta t_{k} I\right)$
\State \Return $(\fZ_{t_{k}}, \fY_{t_{k}})$
\end{algorithmic}
\end{algorithm}

Now via Propostion 1 in \citep{geffner2023langevin} it follows that
\begin{align}
&\frac{\pi_0(\fY_0, \fZ_0)}{\pi_T(\fY_T, \fZ_T)}    \prod_{k=0}^{K-1} \frac{F_{t_k}(\fZ_{t_{k+1}}, \fY_{t_{k+1}} \vert \fZ_{t_{k}}, \fY_{t_{k}})}{B_{t_k}(\fZ_{t_{k}}, \fY_{t_{k}} \vert \fZ_{t_{k+1}}, \fY_{t_{k+1}})}  \nonumber \\
&\;\;=\frac{\pi_0(\fY_0, \fZ_0)}{\pi_T(\fY_T, \fZ_T)} \prod_{k=0}^{K-1} \frac{\mathcal{N}\left(\fY_{t_{k}}' \mid \fY_{t_{k}}(1-\sigma \Delta {t_k}) + \nabla \phi_{t_k}( \fY_{t_{k}} ,\fZ_{t_{k}}) \Delta t_k , 2 \sigma \Delta t_{k} I\right)}{\mathcal{N}\left(\fY_{t_{k}} \mid \fY_{t_{k}}'(1-\sigma \Delta {t_k}) - \nabla \phi_{t_k}( \fY_{t_{k}}' ,\fZ_{t_{k}}) \Delta t_k , 2 \sigma \Delta t_{k} I\right)}
\end{align}
then we can use the above discrete time RND to approximate the KL divergence between SDEs (\ref{eq:controlled annealing2}) and  (\ref{eq:controlled annealing2rev}) yielding our objective for the under dampened setting:

\begin{align}
   \mathcal{L}^{\mathrm{CMCD-{UD}}}_{\KL}(\!\phi\!) \!\approx\! \E\left[\!\ln\!\frac{\pi_0(\fY_0, \fZ_0)}{\pi_T(\fY_T, \fZ_T)} \!\!\prod_{k=0}^{K-1} \!\frac{\mathcal{N}\left(\fY_{t_{k}}'\! \!\mid \!\!\fY_{t_{k}}(1-\sigma \Delta {t_k}) + \nabla \!\phi_{t_k}\!( \fY_{t_{k}} ,\fZ_{t_{k}}) \Delta t_k , 2 \sigma \Delta t_{k} I\right)}{\mathcal{N}\left(\fY_{t_{k}} \!\!\mid \!\!\fY_{t_{k}}'(1-\sigma \Delta {t_k}) - \nabla \!\phi_{t_k}\!( \fY_{t_{k}}' ,\fZ_{t_{k}}) \Delta t_k , 2 \sigma \Delta t_{k} I\right)}\!\right]
\end{align}

Where the expectation is taken with respect to the discrete-time process in Algorithm \ref{alg:Fldvi}.

\section{Proofs}
\label{app:proof}
\subsection{Proof of Proposition \ref{prop:RND} (forward-backward Radon-Nikodym derivatives)}
 \begin{proof}
We begin with the forward Radon-Nikodym  derivative 
\begin{equation}
\label{eq:RN forward}
\log\left(\frac{\mathrm{d}\overrightarrow{\mathbb{P}}^{\mu,a}}{\mathrm{d}\overrightarrow{\mathbb{P}}^{\nu,b}}\right)(\fY) = \log\left(\frac{\mathrm{d}\mu}{\mathrm{d}\nu}\right)(\fY_0) + \tfrac{1}{\sigma^2}  \int_0^T (a_t-b_t)(\fY_t) \cdot \ora{\mathrm{d}}\fY_t + \tfrac{1}{2 \sigma^2} \int_0^T \left( b_t^2 -a_t^2\right)(\fY_t) \, \mathrm{d}t,
\end{equation}
following from Girsanov's theorem (see, for instance, \citet[Lemma A.1]{nusken2021solving} and substitute $\sigma u = a - b$). To compute the backward Radon-Nikodym derivative, we temporarily introduce the time-reversal operator $\mathcal{R}$, acting as $(\mathcal{R}\fY)_t := \fY_{T-t}$ on paths\footnote{Although pathwise definitions should be treated with care (because It\^o integrals are defined only up to a set of measure zero), the arguments can be made rigorous using the machinery referred to in Appendix \ref{app:stochana}.}, and as $(\mathcal{R}a)_t(\vy) := a_{T-t}(\vy)$ on vector fields. We then observe that 
\begin{equation} 
\nonumber
\log\left(\frac{\mathrm{d}\overleftarrow{\mathbb{P}}^{\mu,\mathcal{R}a}}{\mathrm{d}\overleftarrow{\mathbb{P}}^{\nu,\mathcal{R}b}}\right)(\mathcal{R}\fY) = \log\left(\frac{\mathrm{d}\overrightarrow{\mathbb{P}}^{\mu,a}}{\mathrm{d}\overrightarrow{\mathbb{P}}^{\nu,b}}\right)(\fY), 
\end{equation}
for instance by comparing the discrete-time processes in (\ref{eq:forward chain}) and (\ref{eq:backward chain}). Equivalently,
\begin{equation} 
\nonumber
\log\left(\frac{\mathrm{d}\overleftarrow{\mathbb{P}}^{\mu,a}}{\mathrm{d}\overleftarrow{\mathbb{P}}^{\nu,b}}\right)(\fY) = \log\left(\frac{\mathrm{d}\overrightarrow{\mathbb{P}}^{\mu,\mathcal{R}a}}{\mathrm{d}\overrightarrow{\mathbb{P}}^{\nu,\mathcal{R}b}}\right)(\mathcal{R}\fY), \end{equation}
since $\mathcal{R}^2$ is the identity.
Building on (\ref{eq:RN forward}), the backward Radon-Nikodym derivative therefore reads
\begin{subequations}
\label{eq:RN backward}
\begin{align}
\nonumber
\log\left(\frac{\mathrm{d}\overleftarrow{\mathbb{P}}^{\mu,a}}{\mathrm{d}\overleftarrow{\mathbb{P}}^{\nu,b}}\right)(\fY) &  = \log\left(\frac{\mathrm{d}\mu}{\mathrm{d}\nu}\right)((\mathcal{R}\fY)_0) + \tfrac{1}{\sigma^2}  \int_0^T ((\Rc a)_t- (\Rc b)_t)(\Rc \fY_t) \cdot \ora{\mathrm{d}} (\Rc \fY)_t 
\\
\nonumber
& + \tfrac{1}{2 \sigma^2} \int_0^T \left( (\Rc b)_t^2 -(\Rc a)_t^2\right)((\Rc \fY)_t) \, \mathrm{d}t,   \\
\nonumber
& = \log\left(\frac{\mathrm{d}\mu}{\mathrm{d}\nu}\right)(\fY_T) + \tfrac{1}{\sigma^2}  \int_0^T (a_t-b_t)(\fY_t) \cdot \ola{\mathrm{d}}\fY_t + \tfrac{1}{2 \sigma^2} \int_0^T \left( b_t^2 -a_t^2\right)(\fY_t) \, \mathrm{d}t,
\tag{\ref{eq:RN backward}}
\end{align}    
\end{subequations}
where the integrals have been transformed using the substitution $t \mapsto T - t$. The result in (\ref{eq:RND fb}) now follows by writing
\begin{equation}
\nonumber
\log\left(\frac{\mathrm{d}\overrightarrow{\mathbb{P}}^{\mu,a}}{\mathrm{d}\overleftarrow{\mathbb{P}}^{\nu,b}}\right)(\fY)  = \log\left(\frac{\mathrm{d}\overrightarrow{\mathbb{P}}^{\mu,a}}{\mathrm{d}\overrightarrow{\mathbb{P}}^{\Gamma_0,\gamma^+}}\right)(\fY)  + \log\left(\frac{\mathrm{d}\overleftarrow{\mathbb{P}}^{\Gamma_T,\gamma^-}}{\mathrm{d}\overleftarrow{\mathbb{P}}^{\nu,b}}\right)(\fY), 
\end{equation}
using the assumption $\ora{\P}^{\Gamma_0,\gamma^+} = \ola{\P}^{\Gamma_T,\gamma^-}$, and inserting (\ref{eq:RN forward}) as well as (\ref{eq:RN backward}).
\end{proof}

\subsubsection{Discretisation and connection to DNFs (diffusion normalising flows)} \label{app:dnf}

In this section we derive the main discretisation formula used in our implementations for the forward-backwards Radon-Nikodym derivative (RND).
\begin{proposition}\label{prop:euler}
Letting $\Gamma_0=\Gamma_T=\mathrm{Leb}$ and $\gamma^{\pm} = 0$, we have that the RND in (\ref{eq:RND fb}) is given by 
\begin{subequations}
\begin{align*}
\log \left(\frac{\d\fP^{\mu,a}}{\d\bP^{\nu,b}}\right)(\fY) & = \log\mu(\fY_0) - \log \nu(\fY_T) + \tfrac{1}{\sigma^2}\int_0^T a_t(\fY_t) \cdot  \ora{\d} \fY_t - \tfrac{1}{2\sigma^2} \int_0^T|| a_t(\fY_t)||^2 \, \d t
\\
&
-  \tfrac{1}{\sigma^2}\int_0^T b_t (\fY_t) \cdot \ola{\d} \fY_t  +\tfrac{1}{2\sigma^2}\int_0^T ||b_t (\fY_t)|||^2 \, \d t, \qquad \ora{\P}^{\mu,a}\text{-almost surely,}
\end{align*}
\end{subequations}
and admits the following discrete-time approximation up to constant terms in $a_t$ and $b_t$ (following Remark \ref{rem:conversion}),
\begin{align*}
 \log\! \left( \!  \widehat{\frac{\d\fP^{\mu,a}}{\d\bP^{\nu,b}}}\!\right)\!\!(\fY) \! = \!\!- \!\log \nu(\fY_T) \!+\!\!\!\sum_{i=0}^{K-1}\!\!\tfrac{1}{2 \sigma^2{(t_{i+1} - t_{i}})}  ||  \mY_{t_{i}} \! \! - \!\mY_{t_{i+1}} \!\!+b_{t_{i+1}}\!(\mY_{t_{i+1}}) (t_{i+1} \!- t_{i})||^2 \!\!+ \!\mathrm{const},
\end{align*}
when using the Euler-Maruyama discretisation: 
\begin{align*}
    \mY_{t_{i+1}} = \mY_{t_{i}} + a_{t_{i}}(\mY_{t_{i}}) (t_{i+1} - t_{i}) + \sqrt{(t_{i+1} - t_{i})}\sigma \xi, \;\; \qquad \xi \sim \gN(0,I).
\end{align*}

\end{proposition}
\begin{proof}
The first part follows by direct computation.

From here on, we will use the notation $f_{t_i} = f_{t_i}(\fY_{t_i})$ for brevity. Following Remark \ref{rem:conversion} we have that
\begin{subequations}
\begin{align*}
\log \left(\frac{\d\fP^{\mu,a}}{\d\bP^{\nu,b}}\right)(\fY)  & \approx\log\mu(\fY_0) - \log \nu(\fY_T) \\
&+ \tfrac{1}{\sigma^2}\sum_{i=0}^{K-1} a_{t_i} \cdot  (\mY_{t_{i+1}} - \mY_{t_{i}}) - \tfrac{1}{2\sigma^2}\sum_{i=0}^{K-1} || a_{t_i}||^2 \, (t_{i+1} - t_{i})
\\
&
-  \tfrac{1}{\sigma^2}\sum_{i=0}^{K-1} b_{t_{i+1}}  \cdot  (\mY_{t_{i+1}} - \mY_{t_{i}})  +\tfrac{1}{2\sigma^2}\sum_{i=0}^{K-1} ||b_{t_{i+1}}|||^2 \, (t_{i+1} - t_{i}). 
\end{align*}
\end{subequations}
Adding and subtracting $||\mY_{t_{i+1}} - \mY_{t_{i}}||^2 / (\sigma^2(t_{i+1} - t_{i}))$ allows us to complete the square in each sum, resulting in: 
\begin{align} \label{eq:drnd}
\log \left(\frac{\d\fP^{\mu,a}}{\d\bP^{\nu,b}}\right)(\fY)  & \approx\log\mu(\fY_0) - \log \nu(\fY_T) - \sum_{i=0}^{K-1} \tfrac{1}{2 \sigma^2{(t_{i+1} - t_{i})}} || \mY_{t_{i+1}} - \mY_{t_{i}} - a_{t_i} (t_{i+1} - t_{i})||^2 \nonumber
\\
&
+ \sum_{i=0}^{K-1}\tfrac{1}{2 \sigma^2{(t_{i+1} - t_{i})}}  ||  \mY_{t_{i}}   - \mY_{t_{i+1}} +b_{t_{i+1}} (t_{i+1} - t_{i})||^2.
\end{align}
Now notice that under the Euler-Maruyama discretisation $|| \mY_{t_{i+1}} - \mY_{t_{i}} - a_{t_i} (t_{i+1} - t_{i})||^2 = (t_{i+1} - t_{i})\sigma^2 ||\xi||^2$ where $\xi \sim \gN(0,I)$ does not depend on $a_t$ or $b_t$; in particular when using $\KL$ for the divergence we have that $\E_{\fP^{\mu,a}_{\mathrm{EM}}}|| \mY_{t_{i+1}} - \mY_{t_{i}} - a_{t_i} (t_{i+1} - t_{i})||^2 = \sigma^2$ and thus: 
\begin{align}
     \log \left(   \widehat{\frac{\d\fP^{\mu,a}}{\d\bP^{\nu,b}}}\right)(\fY) \propto  \log\mu(\fY_0) - \log \nu(\fY_T) +  \sum_{i=0}^{K-1}\tfrac{1}{2 \sigma^2{(t_{i+1} - t_{i})}}  ||  \mY_{t_{i}}   - \mY_{t_{i+1}} +b_{t_{i+1}} (t_{i+1} - t_{i})||^2. \label{eq:dnf}
\end{align}
\end{proof}
Notice that in expectation (for computing $\KL$), equation (\ref{eq:dnf}) matches equation (15) in \citet{zhang2021diffusion} and thus provides a theoretical backing to the objective used in \citet{zhang2021diffusion}. Resolving the term  $\E_{\fP^{\mu,a}_{\mathrm{EM}}}|| \mY_{t_{i+1}} - \mY_{t_{i}} - a_{t_i} (t_{i+1} - t_{i})||^2$ analytically may offer a variance reduction similar to the analytic calculations in \citet[Equation 14]{sohl2015deep} and the Rao-Blackwelizations of $\KL$ in \citet{ho2020denoising}.


\begin{remark} \label{rem:discrete_elbo}
    The time discretised RND in equation (\ref{eq:drnd}) can be expressed as the ratio of the transition densities corresponding to two discrete-time Markov chains $\mu(\vy_0)q^a(\vy_{1:K}|\vy_0) / p^b(\vy_{0:K-1}|\vy_{K})\nu(\vy_{K})$ with $\vy_{0:K}\sim q^a(\vy_{1:K}|\vy_0)\mu(\vy_0)$. As a result considering $\nu(x) =\hat{\nu}(\vx)/Z$ and the IS estimator $\hat{Z} = p^b(\vy_{0:K-1}|\vy_{K})\hat{\nu}(\vy_{K}) /\mu(\vy_0)q^a(\vy_{1:K}|\vy_0) $ it follows that  $\mathbb{E}_{ q^a(\vy_{1:K}|\vy_0)\mu(\vy_0)} [\ln \hat{Z}]$ is an ELBO of $\hat{Z}$ (e.g. $\mathbb{E}_{ q^a(\vy_{1:K}|\vy_0)\mu(\vy_0)} [\ln \hat{Z}] \leq  \ln {Z}$).
\end{remark}

Whilst superficially simple, Remark \ref{rem:discrete_elbo} guarantees that normalizing constant estimators arising from our discretisation do not overestimate the true normalizing constant. This result is beneficial in practice as it allows us to compare estimators possessing this property by selecting the one with the largest value. As highlighted in \cite{vargas2023denoising} many SDE discretisations can result in estimators that do not yield an ELBO: for example, the estimators used in \cite{berner2022optimal} can result in overestimating the normalising constant. Note similar remarks have been established in the context of free energy computation and the Jarzynski equality see \citet[Remark 4.5]{stoltz2010free}.

\subsection{Proof of Proposition \ref{prop:crooks} (Controlled Crooks' fluctuation theorem and the Jarzinky equality)}
\label{app:fluc}
\begin{proof}

Following the computations in Appendix \ref{app:cmcd kl} and using the formulae (\ref{eq:conversion app}), we compute 
\begin{subequations}
\begin{align*}
\log \left(\frac{\d\fP^{\mu,\sigma^2 \nabla \ln {{\pi}}+\nabla \phi}}{\d\bP^{\nu,-\sigma^2 \nabla \ln {{\pi}}+\nabla \phi}}\right)(\fY) = \log \mu(\fY_0) - \log \nu(\fY_T)
\\
+ \int_0^T \nabla \log {\pi}_t(\fY_t) \circ \d \fY_t - \tfrac{1}{\sigma^2} \int_0^T \nabla \phi_t(\fY_t) \cdot \sigma^2 \nabla \log {\pi}_t(\fY_t) \, \d t - \int_0^T \Delta \phi(\fY_t) \, \d t.
\end{align*}
\end{subequations}

Then via It\^o's lemma applied to the unnormalised annealed log target $\ln \hat{\pi}_t = \ln {\pi}_t - \ln \gZ_t$ we have
\begin{align*}
\ln \hat{\pi}_T(\fY_T) &- \ln \hat{\pi}_0(\fY_0) - \int_0^T \partial_t \ln \hat{\pi}_t(\fY_t) \,\d t =\int_0^T \nabla \log \hat{\pi}_t(\fY_t) \circ \d \fY_t = \int_0^T \nabla \log \pi_t(\fY_t) \circ \d \fY_t,
\end{align*}
thus we arrive at
\begin{align} \label{eq:crooks_gen}
\nonumber
& \log \left(\frac{\d\fP^{\mu,\sigma^2 \nabla \ln {{\pi}}+\nabla \phi}}{\d\bP^{\nu,-\sigma^2 \nabla \ln {{\pi}}+\nabla \phi}}\right)(\fY) = \log\mu(\fY_0) - \log \nu(\fY_T) + \ln \hat{\pi}_T(\fY_T) - \ln \hat{\pi}_0(\fY_0)
\\
\nonumber
& - \int_0^T \partial_t \ln \hat{\pi}_t(\fY_t)\,\d t - \tfrac{1}{\sigma^2} \int_0^T \nabla \phi_t(\fY_t) \cdot \sigma^2 \nabla \log \pi_t(\fY_t) \, \d t - \int_0^T \Delta \phi_t(\fY_t) \, \d t.,
\end{align}
for arbitrary  initial and final densities $\mu$ and $\nu$. Crooks' generalised fluctuation theorem \citep{crooks1999entropy} now follows from taking $\phi = 0$, and the controlled version in Proposition \ref{prop:crooks} follows from $\mu = \pi_0$ and $\nu = \pi_T$.
Finally notice that:
\begin{align*}
1 & =   \E_{\fP^{\mu,\sigma^2 \nabla \ln {{\pi}}}} \left[\left(\frac{\d\fP^{\mu,\sigma^2 \nabla \ln {{\pi}}}}{\d\bP^{\nu,-\sigma^2 \nabla \ln {{\pi}}}}\right)^{-1} \right]
\\
& = \E_{\fP^{\mu,\sigma^2 \nabla \ln {{\pi}}}} \left[\exp\left(-\log\mu(\fY_0) + \log \nu(\fY_T) -  \ln \hat{\pi}_T(\fY_T) + \ln \hat{\pi}_0(\fY_0) + \int_0^T \partial_t \ln \hat{\pi}_t(\fY_t)\,\d t\right)\right],
\end{align*}
which implies the Jarzynski equality when considering the boundaries $\mu = \pi_0$ and $\nu=\pi_T$, resulting in: 
\begin{align}
\nonumber
\E_{\fP^{{\pi}_0,\sigma^2 \nabla \ln {{\pi}}}} \left[\exp\left( \int_0^T \partial_t \ln \hat{\pi}_t(\fY_t)\, \d t\right)\right]=e^{-(\ln \gZ_0 - \ln \gZ_T)}.
\end{align}
\end{proof}

We want to highlight that in \cite[Equations 10-14]{vaikuntanathan2008escorted} ; we can see a similar formulation to our proposed generalised Crooks' fluctuation theorem, that said \cite{vaikuntanathan2008escorted} seems to pose this as a conjecture providing no rigorous proof. Furthermore, unlike our work, they do not formulate this result through SDEs, which we believe we are the first to do. In short, our work and concurrently \cite{zhong2023time} are the first to provide a rigorous treatment in establishing the escorted version of Crooks' fluctuation theorem.

Finally, we also note that \citep{yang2020unified} derive an akin fluctuation theorem for the uncontrolled (not escorted) setting (i.e. $\phi_t =0$) using the classic Girsanov Theorem for forward time SDEs. Whilst superficially it bares some similarity to our sketch  their content and overall result is quite different. In  \citep[Section A]{yang2020unified} rather than estimating $\frac{\d\fP^{\pi_0,\sigma^2 \nabla \ln {{\pi}}}}{\d\bP^{\pi_T,-\sigma^2 \nabla \ln {{\pi}}}}$ as we do they instead compute the RND between $\fP^{\pi_0,\sigma^2 \nabla \ln {{\pi}}}$ and $\fP^{\pi_T,-\sigma^2 \nabla \ln {{\pi}_{T-t}} + \sigma^2 \nabla \ln {{\rho}_{T-t}} }$, in simpler terms they compare the original forward time process with its reversal but simulated also going forward in time, effectively measuring how out of equilibrium the process is, whilst this has a nice physical interpretation the result is significantly less general than ours which allows to compare the RND between processes in different time directions as well as seamlessly derive a controlled variant of the fluctuation theorem. Furthermore, \cite{yang2020unified} uses symmetry arguments to obtain the RND between forward and reverse Brownian motions (see paragraph leading to Equation 78); which can be made more formal via the disintegration theorem and the use of Brownian bridges.

 \subsection{Proof of Proposition \ref{prop:em ipf}: EM $\iff$ IPF} 

In applications, IPF is faced with the following challenges:
\begin{enumerate}
\item 
The sequential nature of IPF, coupled with the need for each iteration to undergo comprehensive training as outlined in Section \ref{sec:Foellmer}, results in significant computational demands.
\item The reference distribution $r(\vx,\vz)$ (or the reference vector field $f_t$) enters the iterations in (\ref{eq:IPF}) only through the initialisation. As a consequence, numerical errors accumulate, and it is often observed that the Schr\"odinger prior is `forgotten' as IPF proceeds \citep{vargas2021solving,fernandes2021shooting,shi2023diffusion}. 
\end{enumerate}

Thus to address these challenges this section will focus on establishing the connection between EM and IPF which in turn will provide us with a family of algorithms that circumvent the sequential nature of IPF, further bridging variational inference and entropic optimal transport.

\begin{proof}
The proof proceeds by induction. 

To begin with, the update formula in (\ref{eq:ipf1}) implies that
\begin{equation}
\nonumber
\pi^{1}(\vx,\vz)  = \argmin_{\pi(\vx,\vz)} \left\{ \KL(\pi(\vx,\vz) || r(\vx,\vz)): \,\, \pi_\vx (\vx)  = \mu(\vx) \right\},
\end{equation}
recalling the initialisation $\pi(\vx,\vz) = r(\vx,\vz)$. To take account of the marginal constraint, we may write $\pi(\vx,\vz) = \mu(\vx) \pi(\vz|\vx)$ and vary over the conditionals $\pi(\vz|\vx)$. By the chain rule for $\KL$, we see that 
\begin{equation}
\label{eq:KL ipf chain}
\KL(\mu(\vx) \pi(\vz|\vx) || r(\vx,\vz)) = \KL(\mu(\vx)||r(\vx)) + \mathbb{E}_{\vx \sim \mu(\vx)} [\KL(\pi(\vz|\vx)|| r(\vz|\vx))],    
\end{equation}
which is minimised at $\pi(\vz|\vx) = r(\vz|\vx)$. From this, it follows that $\pi^1(\vx,\vz) = \mu(\vx)r(\vz|\vx)$ for the first IPF iterate. By assumption, the EM iteration is initialised in such a way that $q^{\phi_0}(\vz|\vx) = r(\vz|\vx)$, so that indeed $\pi^1(\vx,\vz) = q^{\phi_0}(\vz|\vx)\mu(\vx)$. 

The induction step is split (depending on whether $n$ is odd or even):

1.) First assume that the first line of (\ref{eq:em ipf}) holds for a fixed odd $ n \ge 1$. Our aim is to show that this implies that
\begin{equation}
\label{eq:em ipf proof1}
\pi^{n+1}(\vx,\vz) = p^{\theta_{(n+1)/2}}(\vx|\vz)\nu(\vz),    
\end{equation}
that is, the second line of (\ref{eq:em ipf}) with $n$ replaced by $n+1$. From (\ref{eq:ipf2}), we see that 
\begin{equation}
\nonumber
\pi^{n+1}(\vx,\vz) = \argmin_{\pi(\vx,\vz)} \left\{ \KL(\pi(\vx,\vz) || \pi^{n}(\vx,\vz)): \,\, \pi_\vz(\vz)  = \nu(\vz) \right\}.
\end{equation}
Again, we enforce the marginal constraint by setting $\pi(\vx,\vz) = \pi(\vz|\vx) \nu(\vz)$ and proceed as in (\ref{eq:KL ipf chain}) to obtain $\pi^{n+1}(\vx,\vz) = \pi^n(\vx|\vz)\nu(\vz)$. The statement in (\ref{eq:em ipf proof1}) is therefore equivalent to $\pi^n(\vx|\vz) = p^{\theta_{(n+1)/2}}(\vx|\vz)$. To show this, we observe from the EM-scheme in (\ref{eq:EM}) that
\begin{equation}
\nonumber
  \theta_{(n+1)/2} = \argmin_\theta \mathcal{L}_{\KL}(\phi_{(n-1)/2},\theta).  
\end{equation}
In combination with the second line of (\ref{eq:em ipf}) and the definition of $\mathcal{L}_{D}(\phi,\theta)$ in (\ref{eq:abstract loss}), we obtain
\begin{equation}
\nonumber
\theta_{(n+1)/2} = \argmin_{\theta} \KL (\pi^n(\vx,\vz)|| p^\theta(\vx|\vz)\nu(\vz)) =  \argmin_{\theta} \mathbb{E}_{\vz \sim \pi^n_\vz(\vz)} \left[\KL (\pi^n(\vx|\vz)|| p^\theta(\vx|\vz)) \right],  
\end{equation}
where the second equality follows from the chain rule for $\KL$ as in (\ref{eq:KL ipf chain}). Since by assumption the parameterisation of $p^\theta(\vx|\vz)$ is flexible, we indeed conclude that $\pi^n(\vx|\vz) = p^{\theta_{(n+1)/2}}(\vx|\vz)$.

2.) Assume now that the second line of (\ref{eq:em ipf}) holds for a fixed even $n \ge 2$. We need to show that the first line holds with $n$ replaced by $n+1$, that is, 
\begin{equation}
\nonumber
\pi^{n+1}(\vx,\vz) = q^{\phi_{n/2}}(\vz|\vx)\mu(\vx).    
\end{equation}
Using similar arguments as before, we see that $\pi^{n+1}(\vx,\vz) = \pi^n(\vx|\vz)\mu(\vx)$, so that it is left to show that $\pi^n(\vx|\vz) = q^{\phi_{n/2}}(\vz|\vx)$. Along the same lines as in 1.), we obtain
\begin{subequations}
\begin{align*}
\phi_{n/2} = \argmin_\phi \mathcal{L}_{\KL}(\phi,\theta_{n/2}) & = \argmin_\phi  \KL (q^\phi(\vz|\vx) \mu(\vx) || \pi^n(\vx,\vz))
\\
& =  \argmin_\phi \mathbb{E}_{\vx \sim \mu(\vx)} \left[ q^\phi(\vz|\vx) || \pi^n(\vz|\vx)\right].
\end{align*}
\end{subequations}
Again, this allows us to conclude, since the parameterisation in $q^\phi(\vz|\vx)$ is assumed to be flexible enough to allow for $q^{\phi_{n/2}}(\vz|\vx) = \pi^n(\vx|\vz)$. 

The proof for the path space IPF scheme is verbatim the same after adjusting the notation. For completeness, we consider a drift-wise version below.
\end{proof}

\textcolor{blue}{
\begin{remark}
[Extension to $f$-divergences]
\label{rem:fdiv}
The proof does not make use of specific properties of $D_{\mathrm{KL}}$, other than that it satisfies the chain rule. As a consequence, the statement of Proposition \ref{prop:em ipf} straightforwardly extends to other divergences with this property, in particular to $f$-divergences, see Proposition 6 in \citep{baudoin2002conditioned}.
\end{remark}}

\subsection{Drift based EM} \label{app:em_init}

As remarked in the previous subsection, the proof of the equivalence between IPF and EM in path space follows the exact same lines, replacing the chain rule of $\KL$ with the (slightly more general) disintegration theorem \citep{leonard2014some}. In this section, we provide a direct extension to the control setting, yielding yet another IPF-type algorithm and motivating certain design choices for the family of methods we study. 

\begin{corollary}[Path space EM]
\label{cor:path space EM}
For the initialisation $\phi_0 = 0$ the alternating scheme
\begin{align}
\nonumber
    \theta_{n+1} &= \argmin_{\theta} \KL(\ora{\P}^{\mu,f+ \sigma^2\nabla \phi_n}, \ola{\P}^{\nu,f + \sigma^2 \nabla \theta}),  \\
    \phi_{n+1} &=  \argmin_{\phi} \KL(\ora{\P}^{\mu,f + \sigma^2\nabla \phi}, \ola{\P}^{\nu,f + \sigma^2 \nabla \theta_{n+1}})
\end{align}
agrees with the path space IPF iterations in \citep{bernton2019schr,vargas2021solving,de2021diffusion}.
\end{corollary}
\begin{proof}
For brevity let $\mathcal{L}_{\mathrm{FB}}(\phi,\theta) := \KL(\ora{\P}^{\mu,f + \sigma^2\nabla \phi}, \ola{\P}^{\nu,f + \sigma^2 \nabla \theta})$. Additionally, we parameterise the forwards and backwards SDEs with respective path distributions $\ora{\P}^{\mu,f + \sigma^2 \nabla \phi}$, $\ola{\P}^{\nu,f + \sigma^2 \nabla \theta}$ as:
\begin{align}
\nonumber
\mathrm{d}\fY_t\!&= f_t(\fY_t) \, \mathrm{d}t + \sigma^2 \nabla \phi_t(\fY_t)\, \mathrm{d}t + \sigma \, \ora{\mathrm{d}}\fW_t, \quad \bY_0 \!\sim\! \mu,  \\
\nonumber
\mathrm{d}\fY_t &= f_t(\fY_t) \, \mathrm{d}t + \sigma^2\nabla \theta_t(\fY_t)\, \mathrm{d}t + \sigma \, \ola{\mathrm{d}}\fW_t, \quad \fY_T \sim \nu.    
\end{align}

The proof will proceed quite similarly, so instead we will consider just the inductive step for the odd half bridge:
\begin{align}
\nonumber
     \theta_{n} = \argmin_{\theta} \mathcal{L}_{\mathrm{FB}}(\phi_{n-1},\theta).
\end{align}
We can show via the $\KL$ chain rule and the disintegration theorem  \citep{leonard2014some} that the above is minimised when $\theta$ satisfies $\ola{\P}^{\nu,f + \sigma^2\nabla \theta} = \ora{\P}^{\mu,f + \sigma^2 \nabla \phi_{n-1} } \frac{\dd\nu}{\dd \rho_T^{\mu, f + \sigma^2 \nabla \phi_{n-1}}}$ which corresponds to $\nabla \theta_{n} = \sigma^2 \nabla \phi_{n-1} - \sigma^2 \nabla \ln \rho_t^{\mu, f + \sigma^2 \nabla \phi_{n-1}}$ following Observation 1 in \citet{vargas2021solving}. Similarly as per Proposition \ref{prop:em ipf} the results will follow for the even half bridges.

\end{proof}

\paragraph{EM initialisation:}{
The above corollary provides us with convergence guarantees when performing coordinate descent on $\KL(\ora{\P}^{\mu,f + \sigma^2 \nabla \phi}, \ola{\P}^{\nu,f + \sigma^2\nabla \theta})$ subject to initialising $\phi_0=0$. n practice, this indicates that the way of initialising $\phi$ has a major impact on which bridge we converge to.

Thus as a rule of thumb we propose initialising $\phi_0=0$ such that we initialise at the Schr{\"o}dinger prior: then one may carry out joint updates as an alternate heuristic, we call this approach DNF (EM Init), as it is effectively a clever initialisation of DNF inspired by the relationship between IPF and EM.
}


\subsection{HJB-Regularizers }

\label{app:HJB}

As per Section \ref{sec:IPF}, IPF resolves the nonquniqueness in minimising $\mathcal{L}_D(\phi,\theta)$ by performing the coordinate-wise updates (\ref{eq:EM}) starting from an initialisation informed by the Schr{\"o}dinger prior. On the basis of this observation, the joint updates $(\phi_{n+1},\theta_{n+1}) \leftarrow (\phi_n,\theta_n) - h \nabla_{\phi,\theta} \mathcal{L}_D(\phi,\theta)$ suggest themselves, in the spirit of VAEs \citep{kingma2019introduction} and as already proposed in this setting by \citet{neal1998view}. However, as is clear from the introduction, the limit $\lim_{n \rightarrow \infty}(\phi_n,\theta_n)$, can merely be expected to respect the marginals in (\ref{eq:coupling}), and no optimality in the sense of (\ref{eq:dynamic sch}) is expected. As a remedy, we present the following result:
\begin{proposition}
\label{prop:sch objective}
For $\lambda > 0$, a divergence $D$ on path space, and $\phi,\psi \in  C^{1,2}([0,T] \times \mathbb{R}^d;\mathbb{R})$,  let  
\begin{align}
\label{eq:sch objective} 
\mathcal{L}&^{\mathrm{Schr}}(\phi,\theta) := D(\ora{\P}^{\mu,f + \sigma^2\nabla \!\phi},\! \!\ola{\P}^{\nu,f + \sigma^2\nabla \theta}) + \lambda \mathrm{Reg}(\phi),
\end{align}
where $\mathrm{Reg}(\phi) = 0$ if and only if the HJB-equation $\partial_t \phi + f \cdot \nabla \phi + \frac{\sigma^2}{2} \Delta \phi + \tfrac{\sigma^2}{2} |\nabla \phi|^2 = 0$ holds. Then $\mathcal{L}_{\mathrm{Schr}}(\phi,\theta) = 0$ implies that the drift $a_t := \sigma^2 \nabla \phi_t$ solves (\ref{eq:dynamic sch}).
\end{proposition}
The proof rests on an optimal control  reformulation of the Schr{\"o}dinger problem (see  Appendix \ref{app:proof}), identifying the HJB-equation as the missing link that renders joint minimisation of (\ref{eq:sch objective}) theoretically sound for solving (\ref{eq:dynamic sch}). 
The loss in (\ref{eq:sch objective}) has two important benefits compared to standard IPF. First, it circumvents the need for the sequential updates used in IPF, thereby simplifying and speeding up the optimisation procedure. Second, it enforces the Schr{\"o}dinger prior drift $f$ directly, rather than recursively via eq.  (\ref{eq:ipf1}), (\ref{eq:ipf2}). This prevents the prior from being forgotten, as is usually the case in regular IPF.
In Appendix \ref{app:HJB}, we detail possible constructions of $\mathrm{Reg}(\phi)$, discuss relationships to previous work, and evaluate the performance of the suggested approach in numerical experiments.

This result can be found in \citet[Proposition 5.1]{chen2021stochastic}, for instance, but since it is relevant to the connections pointed out in Remark \ref{rem:con} below, we present an independent proof:

\begin{proposition}[Mean-field game formulation]
\label{prop:HJB} 
    Assume that $\phi \!\in \!C^{1,2}([0,T] \!\!\times\!\! \mathbb{R}^d;\!\mathbb{R})$ satisfies the conditions:
\begin{enumerate}
\item
The forward SDE\begin{equation}
\label{eq:Schr SDE} \mathrm{d}\fY_t\!= f_t(\fY_t) \, \mathrm{d}t + \sigma^2 \nabla \phi_t(\fY_t)\, \mathrm{d}t + \sigma \ora{\mathrm{d}}\!\fW_t, \; \bY_0 \!\sim\! \mu    
\end{equation}admits a unique strong solution on $[0,T]$, satisfying moreover the terminal constraint $\fY_T \sim \nu$.
\item
The Hamilton-Jacobi-Bellmann (HJB) equation\begin{equation}\label{eq:HJB}
\partial_t \phi + f \cdot \nabla \phi + \frac{\sigma^2}{2} \Delta \phi + \tfrac{\sigma^2}{2} |\nabla \phi|^2 = 0 \end{equation}holds for all $(t,x) \in [0,T] \times \mathbb{R}^d$. 
\end{enumerate}
Then $a = \sigma^2\nabla \phi$ provides the unique solution to the dynamical Schr{\"o}dinger problem as posed in (\ref{eq:dynamic sch}). 
\end{proposition}
\begin{proof}
We denote the path measures associated to the SDE
\begin{equation}
\label{eq:Schr prior}
\mathrm{d}\fY_t = f_t(\fY_t) \, \mathrm{d}t  + \sigma \, \mathrm{d}\fW_t
\end{equation}
by $\mathbb{P}$ and the SDE (\ref{eq:Schr SDE}) by $\mathbb{P}^{\phi}$, respectively. According to Girsanov's theorem, the Radon-Nikodym derivative satisfies   \begin{equation}
\label{eq:Girsanov HJB}
\frac{\d \P^{\phi}} {\d \P} = \exp \left( \sigma \int_0^T  \nabla \phi_t (\fY_t) \cdot \d \fW_t - \tfrac{\sigma^2}{2} \int_0^T |\nabla \phi_t|^2(\fY_t) \, \mathrm{d}t  \right),  
\end{equation} 
provided that the marginals agree at initial time, $\P_0 = \P^\phi_0$.
Along solutions of (\ref{eq:Schr prior}), we have by It{\^o}'s formula
\begin{subequations}
\label{eq:proof HJB}
\begin{align}
\phi_T(\fY_T)\! - \!\phi_0(\fY_0) & =\! 
\int_0^T \!\!\!\!\partial_t \phi_t (\fY_t) \, \d t +\! \int_0^T \!\!\!\!(f_t \cdot \nabla \phi_t) (\fY_t) \, \d t + \tfrac{\sigma^2}{2}\!\!\int_0^T\! \!\!\!\!\!\Delta \phi_t (\fY_t) \, \d t +
\sigma \!\!\int_0^T\!\!\! \!\nabla \phi_t(\fY_t) \!\cdot\! \d \fW_t 
\nonumber
\\
& = - \tfrac{\sigma^2}{2} \int_0^T |\nabla \phi_t|^2(\fY_t) \, \d t  +
\sigma \int_0^T \nabla \phi_t(\fY_t) \cdot \d \fW_t =\log \left(\frac{\d \P^{\phi}} {\d \P} \right),
\nonumber \tag{\ref{eq:proof HJB}}
\end{align}
\end{subequations}
where we have used the HJB-equation (\ref{eq:HJB})  in the second line. Combining this with (\ref{eq:Girsanov HJB}), we see that 
\begin{equation}
\label{eq:product form}
\frac{\d \P^{\phi}} {\d \P}(\fY) = \exp\left( - \phi_0(\fY_0)\right)     \exp\left(\phi_T(\fY_T)\right).
\end{equation}
The claim now follows, since the unique solution to the Schr{\"o}dinger problem is characterised  by the product-form expression in  (\ref{eq:product form}, see \citet[Section 2]{leonard2013survey}, together with the marginal constraints $\P^\phi_0 = \mu$ and $\P^\phi_T = \nu$, which are satisfied by assumption. 
\end{proof}

\begin{remark}[Summarised relationship to previous work]
\label{rem:previous sch}
For $\lambda = 0$, coordinate-wise updates of $\mathcal{L}_{\mathrm{Schr}}(\phi,\theta)$ recover the IPF updates from \citet{de2021diffusion,vargas2021solving} according to Corollary \ref{prop:em ipf}. Note that $\mathcal{L}_{\mathrm{Schr}}$ is an unconstrained objective, in contrast to (\ref{eq:dynamic sch}); previous works \citep{koshizuka2022neural,zhang2023mean} have suggested incorporating the marginal constraints softly by adding penalising terms to the running cost in (\ref{eq:dynamic sch}). Those approaches require a limiting argument (from an algorithmic standpoint, adaptive tuning of a weight parameter) to recover the solution to (\ref{eq:dynamic sch}). In contrast, the conclusion of Proposition \ref{prop:sch objective} holds for arbitrary $\lambda > 0$. \citet{shi2023diffusion,peluchetti2023diffusion} suggest an algorithm involving reciprocal projections onto the reciprocal class associated to $f_t$. From \cite{clark1991local,thieullen2002reciprocal,roelly2013reciprocal}, the HJB-equation (\ref{eq:HJB}) is a local characteristic ($\mathrm{Reg}(\phi) = 0$ forces (\ref{eq:Schr SDE}) to be in the reciprocal class); hence $\mathrm{Reg}(\phi)$ in (\ref{eq:sch objective}) plays a similar role as the reciprocal projection \citep[Definition 3]{shi2023diffusion}, see Remark \ref{rem:con}. \citet{liu2022deep} suggest an iterative IPF-like scheme involving a temporal difference term \citep[Chapter 6]{sutton2018reinforcement}. As in \citet{nusken2021interpolating}, this is a an HJB-regulariser in the sense of Proposition \ref{prop:sch objective}, see Remark \ref{rem:con}. Finally, \citet[Theorem 5.3]{albergo2023stochastic} and \citet{gushchin2022entropic} develop saddle-point objectives for (\ref{eq:dynamic sch}).
\end{remark}

\begin{remark}[Connection to \emph{reciprocal classes} \citep{shi2023diffusion,peluchetti2023diffusion} and  \emph{TD learning} \citep{liu2022deep}]
\label{rem:con}
The calculation in equation (\ref{eq:proof HJB}) makes the relationship between the HJB equation (\ref{eq:HJB}) and reciprocal classes manifest (since reciprocal classes can essentially be  defined through the relationship (\ref{eq:product form}), see \citet{leonard2014reciprocal,roelly2013reciprocal}). Moreover, equation (\ref{eq:proof HJB}) showcases the relationship between TD learning \citep[Chapter 6]{sutton2018reinforcement} as suggested in  \citet{liu2022deep} and HJB regularisation. Indeed, 
\begin{equation}
\mathrm{Reg}_{\mathrm{BSDE}}(\phi) := \Var \left( \phi_T(\fY_T) - \phi_0(\fY_0) + \tfrac{\sigma^2}{2} \int_0^T |\nabla \phi_t|^2(\fY_t) \, \d t  -
\sigma \int_0^T \nabla \phi_t(\fY_t) \cdot \d \fW_t  \right),
\end{equation}
where the variance is taken with respect to the path measure induced by (\ref{eq:Schr prior}),
is a valid HJB-regulariser in the sense of Proposition \ref{prop:sch objective}. The equivalence between $\mathrm{Reg}_{\mathrm{BSDE}}(\phi) = 0$ and the HJB equation (\ref{eq:HJB}) follows from the theory of backward stochastic differential equations (BSDEs)\footnote{... not to be confused with reverse-time SDEs as in (\ref{eq:SDEs}).}, see, for example, the proof of Proposition 3.4 in \citet{nusken2021interpolating} and the discussion in \citet[Section 3.2]{nusken2021solving}.
\end{remark}

In the following, we present an analogue of Proposition \ref{prop:HJB} involving the backward drift \citep{chen2019stochastic}:
\begin{proposition}
\label{cor:backward HJB}
Assume that $\theta \in C^{1,2}([0,T] \times \mathbb{R}^d;\mathbb{R})$ satisfies the following two conditions:
\begin{enumerate}
\item
The backward SDE
\begin{equation}
\label{eq:Schr SDE back} \mathrm{d}\fY_t = f_t(\fY_t) \, \mathrm{d}t + \sigma^2\nabla \theta_t(\fY_t)\, \mathrm{d}t + \sigma \, \ola{\mathrm{d}}\fW_t, \qquad \fY_T \sim \nu    
\end{equation}
admits a unique strong solution on $[0,T]$, satisfying moreover the initial constraint $\fY_0 \sim \mu$.
\item
The Hamilton-Jacobi-Bellmann (HJB) equation
\begin{equation}
\label{eq:HJB back}
\partial_t \theta + f \cdot \nabla \theta - \frac{\sigma^2}{2} \Delta \theta + \tfrac{\sigma^2}{2} |\nabla \theta|^2 - \nabla \cdot f = 0 \end{equation}
holds for all $(t,x) \in [0,T] \times \mathbb{R}^d$. 
\end{enumerate}
Assuming furthermore that the solution to (\ref{eq:Schr SDE back}) admits a smooth positive density $\rho$,  
we have that $a_t = \nabla \theta_t + \sigma^2 \nabla \log \rho_t$ provides the unique solution to the Schr{\"o}dinger problem as posed in (\ref{eq:dynamic sch}).
\end{proposition}

\begin{remark}
\label{rem:backward HJB}
 As opposed to \citet[equation (41)]{chen2016relation}, the HJB-equation  (\ref{eq:HJB back}) does not involve the time reversal of the Schr\"odinger prior; the form of the HJB equations is not uniquely determined. On the other hand, (\ref{eq:HJB back}) contains the divergence term $\nabla \cdot f$, which discourages us from enforcing this constraint in the same way as (\ref{eq:HJB}). An akin result can be found in \cite{liu2022deep} stated in terms of BSDEs.
\end{remark}

\begin{proof}[Proof of Corollary \ref{cor:backward HJB}]
Using the forward-backward Radon-Nikodym derivative in (\ref{eq:RND fb}), we compute 
\begin{align}
&\log \left( \frac{\d \ora{\P}^{\mu,f}}{\d \ola{\P}^{\nu,f + \sigma^2\nabla \psi}}\right) \!(\fY) \!=\! \log \left( \frac{\d \mu}{\d \mathrm{Leb}} \right) \!-\! \log \left( \frac{\d \nu}{\d \mathrm{Leb}}\right)  + \sigma\int_0^T  \! \! \!f_t(\fY_t) \cdot \d \fW_t - \sigma \!\!\int_0^T \! \! \! f_t(\fY_t)\! \cdot \!\ola{\d}\fW_t  \nonumber
\\
 &- \sigma \int_0^T \nabla \theta_t(\fY_t) \cdot \ola{\d}\fW_t + \tfrac{\sigma^2}{2} \int_0^T |\nabla \theta_t|^2 (\fY_t) \, \d t
 \nonumber
 \\
 &=  \log \left( \frac{\d \mu}{\d \mathrm{Leb}} \right) \! -  \!\log \left( \frac{\d \nu}{\d \mathrm{Leb}}\right) - \sigma \!\!\int_0^T   \!\!\!\! \!(\nabla \cdot f_t)(\fY_t) \, \d t - \sigma \int_0^T  \! \! \! \! \! \!\nabla\theta_t(\fY_t) \cdot \ola{\d}\fW_t + \tfrac{\sigma^2}{2} \!\!\int_0^T  \!\! \!|\nabla \theta_t|^2 (\fY_t) \, \d t.
 \nonumber
\end{align}
Here we have chosen $\ora{\gamma} = \ola{\gamma} = 0$, and $\Gamma_0 = \Gamma_T = \mathrm{Leb}$. The initial measure for the Schr{\"o}dinger prior is $\mu$, but the argument is unaffected by this choice (as the solution is independent of this). We now use the (backward) It{\^o} formula along the Schr{\"o}dinger prior,
\begin{equation}
\theta_t(\fY_T) - \theta_0(\fY_0)\! = \!\int_0^T \!\!\!\partial_t \theta_t (\fY_t) \, \d t + \int_0^T \!\!\!\nabla \theta_t(\fW_t) \cdot \ola{\d} \fW_t + \int_0^T \!\!\!\nabla \theta_t(\fY_t) \cdot f_t(\fY_t) \, \d t - \tfrac{1}{2} \!\int_0^T \!\!\!\Delta \theta_t(\fY_t) \, \d t.
\nonumber
\end{equation}
Using the HJB-equation (\ref{eq:HJB back}), we see that 
\begin{equation}
\log \left( \frac{\d \ora{\P}^{\mu,f}}{\d \ola{\P}^{\nu,f + \sigma^2\nabla \theta}}\right) (\fY) = \log \left( \frac{\d \mu}{\d \mathrm{Leb}} \right) - \log \left( \frac{\d \nu}{\d \mathrm{Leb}}\right) - \theta_t(\fY_T) + \theta_0(\fY_0), 
\end{equation}
and we can conclude as in the proof of Proposition \ref{prop:HJB}.
\end{proof}


\section{CMCD Experiments} \label{app:cmcd_exps}

In this section, we will cover further details pertaining to our experimental setup.

\subsection{ELBO experiments and Comparison to  \cite{geffner2023langevin} }

We compare our underdamped and overdamped CMCD variants against 5 datasets from \citet{geffner2023langevin}, which we describe in further detail below.
\begin{itemize}
\item \texttt{log\_sonar} ($d=61$) and \texttt{log\_ionosphere} ($d=35$) are Bayesian logistic regression models: $ x \sim \mathcal{N}(0, \sigma_w^2 I),y_i \sim \mathrm{Bernoulli}(\mathrm{sigmoid}( x^\top u_i))$ with posteriors conditioned on the 
\textit{sonar} and \textit{ionosphere} datasets respectively.

\item \texttt{brownian} ($d=32$) corresponds to the time discretisation of a Brownian motion:
\begin{align*}
    \alpha_{\mathrm{inn}} &\sim \mathrm{LogNormal}(0,2),\\
    \alpha_{\mathrm{obs}} &\sim \mathrm{LogNormal}(0,2) ,\\
    x_1 &\sim \gN(0, \alpha_{\mathrm{inn}}),\\
    x_i &\sim \gN(x_{i-1}, \alpha_{\mathrm{inn}}), \quad i=2,\hdots 20,\\
    y_i &\sim \gN(x_{i}, \alpha_{\mathrm{obs}}), \quad i=1,\hdots 30.
\end{align*}
inference is performed over the variables $\alpha_{\mathrm{inn}}, \alpha_{\mathrm{obs}}$ and $\{x_i\}_{i=1}^{30}$ given the observations $\{y_i\}_{i=1}^{10} \cup \{y_i\}_{i=20}^{30}$.

\item  \texttt{lorenz} ($d=90$)  is the discretisation of a highly stiff 3-dimensional SDE that models atmospheric convection:
$$
\begin{array}{rlrl}
x_1 & \sim \mathcal{N}(\text { loc }=0, \text { scale }=1) & \\
y_1 & \sim \mathcal{N}(\text { loc }=0, \text { scale }=1) & \\
z_1 & \sim \mathcal{N}(\text { loc }=0, \text { scale }=1) & & \\
x_i & \sim \mathcal{N}\left(\text { loc }=10\left(y_{i-1}-x_{i-1}\right), \text { scale }=\alpha_{\text {inn }}\right) & i=2, \ldots, 30 \\
y_i & \left.\sim \mathcal{N}\left(\text { loc }=x_{i-1}\left(28-z_{i-1}\right)-y_{i-1}\right), \text { scale }=\alpha_{\text {inn }}\right) & i=2, \ldots, 30 \\
z_i & \sim \mathcal{N}\left(\text { loc }=x_{i-1} y_{i-1}-\frac{8}{3} z_{i-1}, \text { scale }=\alpha_{\text {inn }}\right) & i=2, \ldots, 30, \\
o_i & \sim \mathcal{N}\left(\text { loc }=x_i, \text { scale }=1\right) & i=2, \ldots, 30
\end{array}
$$
where $\alpha_{\text {inn }}=0.1$ (determined by the discretization step-size used for the original SDE). The goal is to do inference over $x_i, y_i, z_i$ for $i=1, \ldots, 30$, given observed values $o_i$ for $i \in\{1, \ldots, 10\} \cup\{20, \ldots, 30\}$.

\item \texttt{seeds} ($d=26$) is a random effect regression model trained on the \textit{seeds} dataset:
$$
\begin{aligned}
& \tau \sim \operatorname{Gamma}(0.01,0.01) \\
& a_0 \sim \mathcal{N}(0,10) \\
& a_1 \sim \mathcal{N}(0,10) \\
& a_2 \sim \mathcal{N}(0,10) \\
& a_{12} \sim \mathcal{N}(0,10) \\
& b_i \sim \mathcal{N}\left(0, \frac{1}{\sqrt{\tau}}\right) \\
& i=1, \ldots, 21 \\
& \operatorname{logits}_i=a_0+a_1 x_i+a_2 y_i+a_{12} x_i y_i+b_1 \\
& i=1, \ldots, 21 \\
& r_i \sim \operatorname{Binomial}\left(\operatorname{logits}_i, N_i\right) \\
& i=1, \ldots, 21 \text {. } \\
&
\end{aligned}
$$

The goal is to do inference over the variables $\tau, a_0, a_1, a_2, a_{12}$ and $b_i$ for $i=1, \ldots, 21$, given observed values for $x_i, y_i$ and $N_i$.
\end{itemize}
For all target distributions, we follow the hyperparameter setup from \citet{geffner2023langevin} from their code repository\footnote{\url{https://github.com/tomsons22/LDVI}} for all baseline methods (ULA, MCD, UHA, and LDVI) as well as our overdamped and underdamped variants. We first pretrain the source distribution to a mean-field Gaussian distribution trained for $150,000$ steps with ADAM and a learning rate of $10^{-2}$. Following \citep{zhang2021differentiable, geffner2023langevin} we parametrise our curve of bridging densities $\pi_t = \pi_0^{1 - \beta_t} \pi_T^{\beta_t}$, where $\beta_t$ are a parametrised grid between 0 and 1 that are also trained jointly with other parameters. We then train for $150000$ iterations with a batch size of $5$, tuning learning rate between $[10^{-5}, 10^{-4}, 10^{-3}]$ picking the best one based on mean ELBO after training. For all methods, during training the mean-field source distribution is continued to be trained, as well as the discretisation step size and $\epsilon = \delta t\sigma$. For the underdamped methods we also train the damping coefficient $\gamma$, and for methods involving a score network, i.e. MCD, LDVI, CMCD and CMCD (UD), we train the networks which are chosen to be fully-connected residual networks with layer sizes of $[20, 20]$. In order to report the mean ELBO after training, we obtain 500 samples with 30 seeds and report an averaged value over them.

\textcolor{teal}{\subsection{$\ln Z$, sample quality experiments and comparison to  \citep{zhang2021path,vargas2023denoising}}}

\textcolor{teal}{
 Furthermore, we also include comparisons to a large-dimensional target distribution and two standard distributions with known $\ln Z$ replicated from \citet{vargas2023denoising}, which we summarise below.}
\textcolor{teal}{
\begin{itemize}
    \item \texttt{lgcp} ($d=1600$) is a high-dimensional Log Gaussian Cox process popular in spatial statistics \citep{moller1998log}. Using a $d = M \times M = 1600$ grid, we obtain the unnormalised target density $\mathcal{N}(x ; \mu, K) \prod_{i \in[1: M]^2} \exp \left(x_i y_i-a \exp \left(x_i\right)\right)$.
    \item \texttt{funnel} ($d=10$) is a challenging distribution given by \ \ $\pi_T(x_{1:10} = \mathcal{N}(x_1 ; 0, \sigma_f^2) \mathcal{N}(x_{2: 10} ; 0, \exp 
 (x_1) I)$, with $\sigma_f^2 = 9$
 \citep{neal2003slice}.
 \item \texttt{gmm} ($d=2$) is a two-dimensional Gaussian mixture model with three modes, given by the following target distribution 
 \begin{align*}
     \pi_T(x) &= \frac{1}{3} \mathcal{N}\left(x; \begin{bmatrix}3 \\ 0\end{bmatrix}, \begin{bmatrix}
         0.7 & 0 \\ 0 & 0.05
     \end{bmatrix} \right) + \frac{1}{3}\mathcal{N}\left(x; \begin{bmatrix}-2.5 \\ 0\end{bmatrix}, \begin{bmatrix}
         0.7 & 0 \\ 0 & 0.05
     \end{bmatrix} \right) \\ & \ \ + \mathcal{N}\left(x; \begin{bmatrix}2 \\ 3\end{bmatrix}, \begin{bmatrix}
         1 & 0.95 \\ 0.95 & 1
     \end{bmatrix} \right)
 \end{align*}
\end{itemize}
}

\textcolor{teal}{
For these target distributions, we follow the hyperparameter setup from \citet{vargas2023denoising} from their code repository\footnote{\url{https://github.com/franciscovargas/denoising_diffusion_samplers}}for the baseline methods of DDS and PIS, and replicate them as closely as possible for CMCD. Unlike the previous, we don't pretrain the mean-field Gaussian source distribution $\mathcal{N}(0, \sigma_{\text{init}}^2 I)$. We select the optimal learning rate in $[10^{-3}, 10^{-4}, 10^{-5}]$, the optimal standard deviation of the source distribution $\sigma_{\text{init}}$ in $[1, 2, 3, 4, 5]$ and the optimal $\alpha$ in $[0.1, 0.5, 1, 1.5, 2]$. Instead of training $\epsilon = \delta t \sigma$, we sweep over an optimal value in $[10^{-2}, 10^{-1}, 1]$. The models are trained with a batch size of $300$ for $11000$ steps, where we keep the source distribution parameters fixed, as well as $\epsilon$. For evaluation, we use 30 seeds with a batch size of 2000, and report average performance over the seeds. DDS and PIS use a 128-dimensional positional embedding, along with an additional network for the time parameters, however MCMD uses a regular score network. In order to make exact comparisons, we select differing network architecture sizes that result in an equivalent number of parameters for \texttt{funnel} and \texttt{gmm}. For \texttt{lgcp}, due to the high dimensionality of the dataset, we choose a small network for CMCD. We summarise these below. For \texttt{gmm} and \texttt{funnel}, it is possible to sample from the target distribution, and we report an OT-regularised distance ($\mathcal{W}_2^{\gamma}$) with a regularisation $\gamma = 10^{-2}$. Similar to the mean ELBO, we draw 2000 samples from the models and the targets, and average $\mathcal{W}_2^{\gamma}$ over 30 seeds. We use the Python Optimal Transport\footnote{\url{https://pythonot.github.io/}} library's default implementation of entropy-regularised distance. Results for comparisons to DNF can be found in Table \ref{tab:dnf}.
}

\begin{figure}
    \centering
    \includegraphics{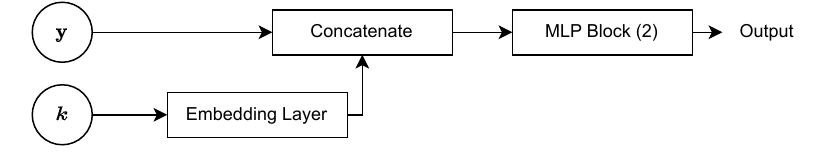}
    \caption{ \textcolor{orange}{Architecture from \citep{geffner2023langevin} used across experiments for our CMCD drift network. Softplus activations are used.}}
    \label{fig:cmcd_net}
\end{figure}

\begin{table}[htb!]
\caption{\textcolor{orange}{\textbf{Network Sizes for comparison.} Note that CMCD has less parameters for the despite the Funnel target despite the larger drift due to the PIS and DDS networks having an additional grad network.}}
\label{table:dds_hp}
\begin{center}
\begin{small}
\begin{sc}
\resizebox{0.4 \textwidth}{!}{%
\begin{tabular}{lccc}
\toprule
    & GMM & LGCP & Funnel \\ 
    \midrule
    DDS
    & $[10, 10]$
    & $[64, 64]$
    & $[64, 64]$
    \\
    PIS
    & $[10, 10]$
    & $[64, 64]$
    & $[64, 64]$
    \\
    CMCD
    & $[38, 38]$
    & $[64, 64]$
    & $[110, 110]$
    \\
    \bottomrule
  \end{tabular}
 }
\end{sc}
\end{small}
\end{center}
\vskip -0.1in
\end{table}

\textcolor{olive}{
\subsection{Comparisons with the log-variance loss - Mode collapse failure mode}}

\textcolor{olive}{
Here, we report performance using the log-variance divergence-based loss  \citep{nusken2021solving} introduced at the end of Section \ref{sec:annealing}, 
\begin{align}
  \mathcal{L}^{\mathrm{CMCD}}_{\mathrm{Var}}(\phi)\!  \approx\! \mathrm{Var}\!\left[\!\ln\! \frac{{\pi}_0(\mY_0)}{\hat{\pi}(\mY_T)} \!\!\prod_{k=0}^{K-1} \!\!\frac{\gN(\mY_{t_{k+1}} | \mY_{t_{k}} + (\nabla \ln \pi_{t_k} + \nabla \phi_{t_k}) (\mY_{t_{k}})\Delta t_k, 2 \sigma^2 \Delta t_k )}{\gN\!(\mY_{t_{k}} | \mY_{t_{k+1}}\!\! \!+ \!(\nabla\! \ln \pi_{t_{k+1}}\! \!\!- \!\nabla\! \phi_{t_{k+1}})(\mY_{t_{k+1}})\Delta t_k, 2 \sigma^2 \Delta t_k )}\!\right]\!,
\tag{\ref*{eq:main objective}}
\end{align}}

\textcolor{olive}{
A careful reader will note this loss simply consists of replacing the expectation in the KL loss with a variance. A major computational advantage of this loss is that the measure that the expectations are taken with respect to can be any measure and is not restricted to the forward or backward SDEs like in KL \citep{richter2023improved,richter2020vargrad,nusken2021solving}, this allows us to detach the samples and thus accommodating for a much more computational objective.}

\textcolor{olive}{
which we find performs quite well compared to our default loss function, especially for multimodal target distributions. We consider the very multi-modal mixture of Gaussian target distribution from \citet{midgley2022flow}, and report the ELBO and $\ln Z$ numbers in the table below. For this experiment, we use a batch size of 2000 and train neural networks with a size $[130, 130]$ for $150k$ iterations.}
\begin{table}[htb!]
\caption{\textcolor{olive}{\textbf{ELBO and $\ln Z$ on 40-GMM}}}
\label{table:logvar}
\begin{center}
\begin{sc}
\resizebox{0.7 \textwidth}{!}{%
\begin{tabular}{lccc}
\toprule
    & ELBO & $\ln Z$ & $\mathcal{W}_2$ \\ 
    \midrule
    \textit{log-variance} loss
    & -1.279 $\pm$ 0.096
    & -0.065 $\pm$ $0.101$
    & 0.0143 $\pm$ 0.001
    \\
    KL loss
    & -2.286 $\pm$ 0.1109
    & -0.244 $\pm$ 0.3309
    & 0.0441 $\pm$ 0.012
    \\
    \bottomrule
  \end{tabular}
 }
\end{sc}
\end{center}
\vskip -0.1in
\end{table}

\begin{figure}[H]
    \centering
    \includegraphics[width=0.4\textwidth]{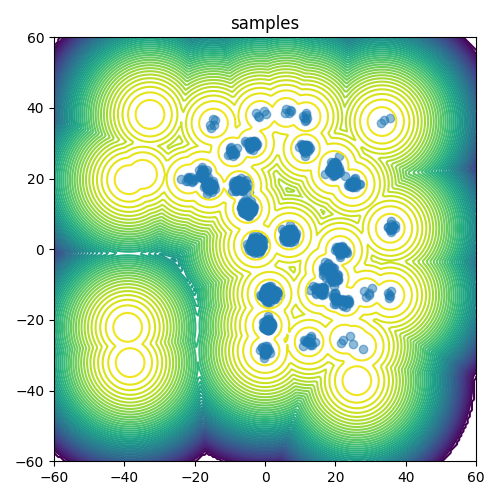}
    \includegraphics[width=0.4\textwidth]{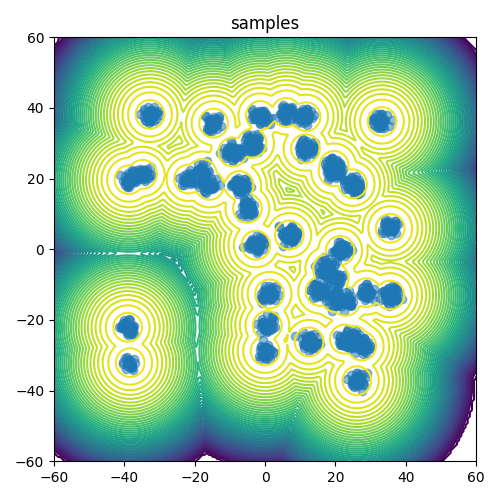}
    \caption{\textcolor{olive}{(left) 2000 samples drawn from the CMCD algorithm trained with the default loss function, and (right) 2000 samples drawn from the algorithm trained with the log-variance divergence-based loss. We can see that the default loss function misses many modes in the target distribution, whereas the log-variance loss has not missed any modes. We report final results after sweeping over $\Delta_{t_k}$ and learning rates for both methods, picking the one with the lowest training loss. We highlight that concurrent work by \cite{richter2023improved} explores the log variance divergence in more detail and proposes an akin general framework for diffusion-based sampling.}}
    \label{fig:vargrad_plot}
\end{figure}

\begin{figure}[H]
    \centering
    \includegraphics[width=0.6\textwidth]{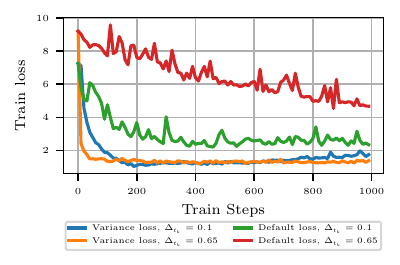}
    \caption{\textcolor{olive}{Plots showing training loss curves for the log-variance loss and the default loss for different values of $\Delta_{t_k}$. We find that a low value of $\Delta_{t_k}=0.1$ is needed in order to obtain a low training loss for the default loss, whereas the log-variance loss is much more robust to different values of $\Delta_{t_k}$. The x-axis reports an evaluation every 150 steps of training}}
    \label{fig:vargrad_training}
\end{figure}

\textcolor{teal}{\subsection{Specification and Tuning: SMC, AFT, and NF-VI}}
\textcolor{teal}{
We adopted the implementations\footnote{\url{https://github.com/google-deepmind/annealed_flow_transport}} provided by the studies in \cite{arbel2021annealed,matthews2022continual} and initialized them with default hyperparameters before fine-tuning.}

\textcolor{teal}{
\textbf{Sequential Monte Carlo (SMC).}
For SMC, we utilized 2000 particles sampled from a zero-mean unit Gaussian distribution, implementing re-sampling if the effective sample size (ESS) fell below 0.3. We employed Hamiltonian Monte-Carlo (HMC) for particle mutation, executing one Markov Chain Monte Carlo (MCMC) step after each annealing step. The number of leapfrog steps was fixed at 10, and an extensive grid search over different step sizes was conducted, consistent with \cite{arbel2021annealed}. This search spanned four different step sizes, contingent on the temperature, resulting in a grid search over 256 parameters. The finalized values are presented in Table \ref{table:mdn_hp}. For SMC, $K$ is defined by the number of temperatures. }

\textcolor{teal}{Furthermore, we report the results for SMC in Table \ref{table:mdn_hp} for the tuned hyperparameters used for each target. Please note that we were able to obtain similar $\ln Z$ values as in \cite{vargas2023denoising} suggesting SMC was well-tuned. Finally, for each result, we report the mean and standard deviations across 30 different seeds, results can be seen in Table \ref{table:mdn_hp}.}

\textcolor{teal}{\textbf{Annealed Flow Transport Monte Carlo (AFT).}
We maintained a similar setup to SMC, with a few adjustments: using 500 particles for training and 2000 for evaluation to accommodate the added complexity from the normalizing flows. We also decreased the number of temperatures and increased the number of MCMC steps to mitigate memory requirements from the flows. $K$ is defined as the number of temperatures $\times$ MCMC steps, with the latter fixed at 4, resulting in a maximum of 64 flows trained simultaneously. Inverse autoregressive flows (IAFs) were employed in all experiments except for \textit{lgcp}, using a neural network with one hidden layer whose dimension matches the problem's dimensionality. For \textit{lgcp}, a diagonal affine flow was used due to memory constraints arising from the high dimensionality. AFT flows were trained for 300 iterations until convergence. }


\textcolor{teal}{\textbf{Variational Inference with Normalizing Flows (VI-NF).}
We utilized the same flows as for AFT. In this case, $K$ denotes the number of flows to stack. The flows were trained over a total of 2000 iterations with a batch size of 500. For some targets}

\begin{table}[htb!]
\caption{\textcolor{teal}{\textbf{Tuned MCMC Step Sizes.}}}
\label{table:mdn_hp}
\begin{center}
\begin{small}
\begin{sc}
\resizebox{\textwidth}{!}{%
\begin{tabular}{lcccccccccc}
\toprule
    & gmm & lgcp & lorenz & brownian  & log\_sonar & log\_ionosphere & seeds & funnel \\ 
    \midrule
    $\Delta t$
    & $[0.5, 0.5, 0.5, 0.3]$
    & $[0.3, 0.3, 0.2, 0.2]$
    & $[0.01, 0.01, 0.008, 0.01]$
    & $[0.2,0.2,0.05,0.05]$
    & $[0.2,0.05,0.2,0.2]$
    & $[0.1,0.2,0.2,0.2]$
    & $[0.2,0.1,0.05,0.01]$
    & $[0.05, 0.2, 0.2, 0.05]$
    \\
    \bottomrule
  \end{tabular}
 }
\end{sc}
\end{small}
\end{center}
\vskip -0.1in
\end{table}

\begin{table}[htb!]
\caption{\textcolor{teal}{\textbf{SMC Results.} $ \text{ELBO}$ and $\ln Z$ values for a different number of steps $K$ and experiments.}}
\label{table:mdn_hp_2}
\begin{center}
\begin{small}
\begin{sc}
\resizebox{\textwidth}{!}{%
\begin{tabular}{lcccccccccc}
\toprule
    $\ln Z$ & gmm & lgcp & lorenz & brownian  & log\_sonar & log\_ionosphere & seeds & funnel \\ 
    \midrule
    $K=8$
& $-0.536 \pm 0.042$ 
& $-364.074 \pm 7.797$ 
& $-87502.352 \pm 4004.495$ 
& $-63.32 \pm 8.016$ 
& $-178.589 \pm 2.784$ 
& $-204.594 \pm 3.049$ 
& $-108.676 \pm 1.221$ 
& $-1.013 \pm 0.116$ 
    \\
    $K=16$
& $-0.255 \pm 0.034$ 
& $-135.207 \pm 4.665$ 
& $-42148.287 \pm 1047.478$ 
& $-28.714 \pm 3.71$ 
& $-137.691 \pm 1.656$ 
& $-149.107 \pm 1.088$ 
& $-88.068 \pm 0.467$ 
& $-0.65 \pm 0.1$ 
    \\
    $K=32$
& $-0.119 \pm 0.017$ 
& $86.106 \pm 5.989$ 
& $-19288.267 \pm 834.52$ 
& $-12.23 \pm 2.212$ 
& $-120.557 \pm 0.613$ 
& $-127.964 \pm 0.394$ 
& $-79.89 \pm 0.273$ 
& $-0.408 \pm 0.17$ 
    \\
    $K=64$
& $-0.059 \pm 0.015$ 
& $269.566 \pm 7.832$ 
& $-8894.525 \pm 119.723$ 
& $-4.76 \pm 1.042$ 
& $-113.835 \pm 0.167$ 
& $-118.812 \pm 0.192$ 
& $-76.275 \pm 0.189$ 
& $-0.359 \pm 0.087$ 
    \\
    $K=128$
& $-0.029 \pm 0.009$ 
& $390.33 \pm 5.427$ 
& $-5419.678 \pm 90.362$ 
& $-1.675 \pm 0.442$ 
& $-110.901 \pm 0.094$ 
& $-114.827 \pm 0.307$ 
& $-74.774 \pm 0.097$ 
& $-0.255 \pm 0.108$ 
    \\
    $K=256$
& $-0.013 \pm 0.006$ 
&$477.162 \pm 4.998$ 
& $-3745.218 \pm 68.342$ 
& $-0.131 \pm 0.22$ 
& $-109.562 \pm 0.072$ 
& $-113.123 \pm 0.172$ 
& $-74.049 \pm 0.088$ 
& $-0.211 \pm 0.074$ 
    \\
    \midrule
     ELBO &  &  &  &   & &  &  &  \\
     \midrule
    $K=8$
& $0.002 \pm 0.066$ 
& $-236.087 \pm 9.623$ 
& $-56122.917 \pm 5402.094$ 
& $-10.147 \pm 3.427$ 
& $-117.499 \pm 4.049$ 
& $-123.772 \pm 3.689$ 
& $-75.183 \pm 1.447$ 
& $-0.417 \pm 0.236$ 
    \\
    $K=16$
& $-0.003 \pm 0.037$ 
& $-23.219 \pm 7.756$ 
& $-27397.2 \pm 1987.523$ 
& $-3.924 \pm 2.114$ 
& $-110.707 \pm 1.823$ 
& $-113.476 \pm 1.361$ 
& $-73.524 \pm 0.543$ 
& $-0.322 \pm 0.184$ 
    \\
    $K=32$
& $0.003 \pm 0.018$ 
& $174.797 \pm 7.241$ 
& $-12110.983 \pm 1204.2$ 
& $-0.426 \pm 1.415$ 
& $-108.574 \pm 0.547$ 
& $-112.048 \pm 0.519$ 
& $-73.459 \pm 0.29$ 
& $-0.215 \pm 0.222$ 
    \\
    $K=64$
& $0.001 \pm 0.015$ 
& $332.187 \pm 9.025$ 
& $-5360.819 \pm 306.407$ 
& $0.884 \pm 0.778$ 
& $-108.424 \pm 0.154$ 
& $-111.715 \pm 0.184$ 
& $-73.375 \pm 0.214$ 
& $-0.267 \pm 0.101$ 
    \\
    $K=128$
& $0.001 \pm 0.009$ 
& $430.838 \pm 6.441$ 
& $-3624.167 \pm 168.119$ 
& $1.008 \pm 0.27$ 
& $-108.395 \pm 0.087$ 
& $-111.603 \pm 0.298$ 
& $-73.436 \pm 0.095$ 
& $-0.2 \pm 0.124$ 
    \\
    $K=256$
& $0.002 \pm 0.006$ 
& $453.395 \pm 4.43$
& $-2811.161 \pm 106.68$ 
& $1.142 \pm 0.125$ 
& $-108.368 \pm 0.071$ 
& $-111.611 \pm 0.171$ 
& $-73.413 \pm 0.087$ 
& $-0.181 \pm 0.081$ 
    \\
    \bottomrule
  \end{tabular}
 }
\end{sc}
\end{small}
\end{center}
\vskip -0.1in
\end{table}

\textcolor{teal}{
\subsection{Further Ablation with NF-style methods and AFT}
}

\textcolor{teal}{
We further run both flow models (AFT and NFVI) on all possible target distributions (subject to OOM errors). Results can be found in Table \ref{tab:allres}.}

\begin{table}[!ht]
\caption{\textcolor{teal}{\textbf{$\ln Z$ comparison.} $\ln Z$ values for a different number of steps $K$, experiments and methods. Not all methods could be evaluated on every $K$/experiment combination due to numerical instabilities or out-of-memory (OOM) problems.}}\label{tab:dnf}
    \centering
    \resizebox{\textwidth}{!}{%
    \begin{tabular}{l|l|cccccc}
\toprule
        \textbf{Dataset} & \textbf{Method} & $K=8$ & $K=16$ & $K=32$ & $K=64$ & $K=128$ & $K=256$ \\ \midrule
        funnel & CMCD & -0.3037 $\pm$ 0.1507 & -0.223 $\pm$ 0.1041 & -0.1805 $\pm$ 0.0773 & -0.1085 $\pm$ 0.1143 & -0.0573 $\pm$ 0.0444 & -0.01928 $\pm$ 0.0641 \\ 
        \textcolor{white}{log\_ionosphere} & VI-DNF & -0.3768 $\pm$ 0.2157 & -0.3517 $\pm$ 0.1627 & -0.2919 $\pm$ 0.0999 & -0.6941 $\pm$ 0.6841 & -0.1947 $\pm$ 0.1325 & -0.2124 $\pm$ 0.0637 \\ 
        ~ & VI-NF & -0.206$\pm$ 0.079 & -0.206$\pm$ 0.082 & -0.206$\pm$ 0.087 & -0.194$\pm$ 0.101 & -0.182$\pm$ 0.097 & -0.197$\pm$ 0.099 \\ 
        ~ & AFT & -0.875$\pm$ 0.543 & -0.395$\pm$ 0.351 & -0.348$\pm$ 0.192 & -0.271$\pm$ 0.227 & -0.235$\pm$ 0.139 & -0.196$\pm$ 0.111 \\ 
        \midrule
        gmm & CMCD & -0.1358 $\pm$ 0.0839 & -0.01331 $\pm$ 0.1292 & 0.0095 $\pm$ 0.0495 & 0.00736 $\pm$ 0.0477 & -0.0004 $\pm$ 0.0368 & -0.0081 $\pm$ 0.0520 \\ 
        ~ & VI-DNF & -0.3676 $\pm$ 0.6314 & -0.258 $\pm$ 0.412 & -0.4983 $\pm$ 0.3878 & -0.4449 $\pm$ 0.5379 & -0.4652 $\pm$ 0.3223 & -0.204 $\pm$ 0.6381 \\ 
        ~ & VI-NF & -0.355$\pm$ 0.698 & -0.455$\pm$ 0.258 & -0.064$\pm$ 0.138 & -0.054$\pm$ 0.15 & -0.066$\pm$ 0.188 & -0.045$\pm$ 0.177 \\ 
        ~ & AFT & -0.336$\pm$ 0.372 & -0.006$\pm$ 0.082 & 0.02$\pm$ 0.068 & -0.016$\pm$ 0.042 & -0.003$\pm$ 0.029 & 0.001$\pm$ 0.026 \\ 
        \midrule
        lgcp & CMCD & 491.059 $\pm$ 3.553 & 498.147 $\pm$ 2.624 & 502.705 $\pm$ 2.482 & 506.045 $\pm$ 1.761 & 508.165 $\pm$ 1.553 & 509.43 $\pm$ 1.242 \\ 
        ~ & VI-DNF & 424.733 $\pm$ 5.858 & 424.719 $\pm$ 5.855 & 424.714 $\pm$ 5.861 & 424.719 $\pm$ 5.860 & 424.7 $\pm$ 5.869 & 424.705 $\pm$ 5.896 \\ 
        ~ & AFT & 126.651$\pm$ 5.764 & 344.145$\pm$ 23.95 & 191.613$\pm$ 173.873 & 420.259$\pm$ 91.43 & 480.126$\pm$ 33.059 & 491.028$\pm$ 8.057 \\ 
\bottomrule
    \end{tabular}
    }
\end{table}
\textcolor{teal}{
\begin{table}[!ht]
\caption{\textcolor{teal}{\textbf{ELBO comparison.} ELBO values for a different number of steps $K$, experiments and methods. Not all methods could be evaluated on every $K$/experiment combination due to numerical instabilities or out-of-memory (OOM) problems.}} \label{tab:allres}
    \centering
    \resizebox{\textwidth}{!}{%
    \begin{tabular}{l|l|cccccc}
\toprule
        \textbf{Dataset} & \textbf{Method} & $K=8$ & $K=16$ & $K=32$ & $K=64$ & $K=128$ & $K=256$ \\ 
        \midrule
        seeds & 
        CMCD &  $-74.501 \pm 0.049$ &
$-74.327 \pm 0.065$ &
$-74.142 \pm 0.05$ &
$-73.967 \pm 0.038$ &
$-73.8 \pm 0.032$ &
$-73.684 \pm 0.033$ \\ 
        ~ & 
        VI-NF&  $-73.563 \pm 0.013$  &  $-73.547 \pm 0.012$  &  $-73.574 \pm 0.012$  &  $-73.58 \pm 0.014$  &  $-73.621 \pm 0.014$  &  $-73.675 \pm 0.014$ \\ 
        ~ & AFT &  $-147.457 \pm 24.808$  &  $-116.134 \pm 8.157$  & $-99.032 \pm 6.321$ & $-87.436 \pm 1.53$  &  $-79.847 \pm 0.419$  &  $-76.364 \pm 0.188$ \\ 
        ~ & CRAFT &  $-146.973 \pm 1.531$ &
$-94.2 \pm 0.505$ &
$-80.985 \pm 0.344$ &
$-76.555 \pm 0.175$ &
$-74.979 \pm 0.143$ &
$-74.225 \pm 0.133$ \\
        \midrule
        log\_ionosphere & 
        CMCD &  $-113.211 \pm 0.089$ &
$-112.643 \pm 0.062$ &
$-112.643 \pm 0.062$ &
$-112.22 \pm 0.046$ &
$-111.98 \pm 0.04$ &
$-111.925 \pm 0.046$ \\ 
        ~ & VI-NF  & $-111.903 \pm 0.022$  &  $-111.902 \pm 0.022$  &  $-111.892 \pm 0.017$  &  $-111.881 \pm 0.017$ &  OOM & OOM \\ 
        ~ & AFT &  $-168.174 \pm 21.249$  &  $-138.733 \pm 8.374$  &  $-123.013 \pm 3.771$  &  $-118.644 \pm 0.891$  &  $-116.497 \pm 0.495$  &  $-114.905 \pm 0.781$ \\ 
        \midrule
        log\_sonar & 
        CMCD & $-112.274 \pm 0.124$ &
$-110.904 \pm 0.111$ &
$-110.459 \pm 0.106$ &
$-109.503 \pm 0.075$ &
$-109.608 \pm 0.066$ &
$-109.25 \pm 0.052$  \\ 
        ~ & VI-NF &  $-109.353 \pm 0.035$  &  $-109.346 \pm 0.031$  &  $-109.441 \pm 0.035$  &  $-109.94 \pm 0.044$  &  $-109.711 \pm 0.039$  &  OOM \\ 
        ~ & AFT &  $-203.249 \pm 12.506$  &  $-148.357 \pm 8.096$  &  $-129.772 \pm 3.057$  &  $-121.653 \pm 2.505$  &  $-114.911 \pm 0.331$  &  $-112.021 \pm 0.182$ \\ 
        \midrule
        lgcp & 
        CMCD &  $469.475 \pm 0.259$ &
$479.246 \pm 0.237$ &
$486.739 \pm 0.249$ &
$492.745 \pm 0.239$ &
$497.074 \pm 0.267$ &
$499.708 \pm 0.236$ \\ 
        ~ & AFT & $75.896 \pm 0.863$  &  $265.005 \pm 34.254$  &  $62.898 \pm 200.991$  &  $340.687 \pm 126.853$  & $417.916 \pm 50.35$  & $424.705 \pm 12.416$ \\ 
        \midrule
        lorenz & 
        CMCD &  $-1180.797 \pm 0.184$ &
$-1180.797 \pm 0.184$ &
$-1176.514 \pm 0.154$ &
$-1174.309 \pm 0.148$ &
$-1172.453 \pm 0.153$ &
$-1170.826 \pm 0.15$ \\ 
        ~ & VI-NF & $-1499.102 \pm 0.84$  &  $-1471.798 \pm 0.582$  &  $-1439.648 \pm 0.274$  &  $-1433.536 \pm 0.316$  &  OOM &  OOM \\ 
        \midrule
        brownian & 
        CMCD &  $-0.753 \pm 0.075$ &
$-0.209 \pm 0.059$ &
$0.153 \pm 0.045$ &
$0.376 \pm 0.038$ &
$0.578 \pm 0.046$ &
$0.722 \pm 0.032$ \\ 
        ~ & VI-NF &  $0.733 \pm 0.019$  & $0.797 \pm 0.018$  & $0.816 \pm 0.018$  &  OOM  &  OOM  &  OOM \\ 
        %
        \bottomrule
    \end{tabular}
    }
\end{table}
}

\textcolor{teal}{
\subsection{Wallclock times for $\log Z$ calculation}
}
\textcolor{teal}{
In order to calculate the average wall-clock time for $\log Z$ calculation, we calculate the time it takes to draw $30$ seeds of $2000$ samples each from the methods below, and use these samples to calculate the mean and standard deviation of $\log Z$ across $30$ seeds.}
\begin{table}[]
    \centering
    \begin{tabular}{lccc}
\toprule Method & Average Time (s) & Min Time (s) & Max Time (s) \\
\midrule CMCD (OD) & 9.665 & 5.592 & 21.475 \\
 ULA & 9.204 & 4.673 & 20.721 \\
 UHA & 9.427 & 5.588 & 20.263 \\
 MCD & 9.204 & 4.673 & 20.721  \\
\bottomrule
\end{tabular}
    \caption{\textcolor{teal}{Wallclock times for evaluation in seconds.}}
    \label{tab:my_label}
\end{table}

\begin{table}[]
\centering
\adjustbox{max width=\textwidth}{
\begin{tabular}{@{}lllll@{}}
\toprule
Mode     & ULA                         & MCD                                   & CMCD                                       & DDS / PIS                                         \\ \midrule
Sampling & $\mathcal{O}(K \cdot (d+ G(d)))$ & $\mathcal{O}(K \cdot (d+G(d)))$           & $\mathcal{O}(K \cdot ( d + G(d) + N(d)) )$ & $\mathcal{O}(K \cdot ( N_2 + G(d) + N_1(d)) )$ \\
ELBO     & $\mathcal{O}(K \cdot (d+G(d)))$ & $\mathcal{O}(K \cdot (d +G(d) + N(d)) )$ & $\mathcal{O}(K \cdot ( d + G(d) + N(d)) )$ & $\mathcal{O}(K \cdot ( N_2 + G(d) + N_1(d)) )$ \\ \bottomrule
\end{tabular}}
\caption{\textcolor{orange}{Sampling and loss calculation complexity across SDE based methods,  $K$ represents the number of integration steps, $G(d)$ represents the cost of evaluating the score of the target and $N(d)$ for evaluating the drift/score networks both quantities are dimension dependant. PIS and DDS have an additional grad network cost $N_2$ which is dimension independant. }}
\end{table}

\textcolor{orange}{
\subsection{Training time comparisons to SMC}}

\textcolor{orange}{In this section, we explore a total time comparison between our approach CMCD and SMC. }

\textcolor{orange}{As both methods are quite inherently different it is not immediately obvious how to carry out an insightful comparison. In order to do so we chose the LGCP which is our most numerically intense target and we phrase the following question:}

\begin{center}
\textcolor{orange}{“For how long do we have to train CMCD to outperform the best-run SMC”}
\end{center}

\textcolor{orange}{For this, we look at our Figure \ref{fig:simple} pane c) and we can see that at $K=8$ CMCD already outperforms SMC at $K=256$ with 2000 particles. So we choose these two approaches to compare to.  In Table \ref{tab:cmcdsmc} is a brief comparison of total time calculations, note we have included tuning time for SMC which is akin to our training time as without tuning SMCs hyperparameters ELBOs and ln Z estimations were much worse. We can observe that the total runtime for training and sampling CMCD to reach a better $\ln Z$ value does not exceed the time required to tune SMC.}

\begin{table}[]
\centering
\begin{tabular}{@{}llll@{}}
\toprule
Method & Train + Sample Time (min) & ln Z                & ELBO                  \\ \midrule
CMCD   & $33.12 \pm 0.12$          & $491.059 \pm 3.553$  & $469.475 \pm  0.2589$ \\ 
SMC    & $62.62 \pm 0.10$           & $477.162 \pm 4.998$ & $453.395 \pm 4.4300$    \\ \bottomrule
\end{tabular}
\caption{\textcolor{orange}{Training+ Tuning + Sampling time comparisons for CMCD and SMC at comparable ln Z estimates.}} \centering\label{tab:cmcdsmc}
\end{table}

\section{Regularised IPF-type Experiments}

For the purpose of completeness in this section, we empirically explore the regularised IPF-type objectives proposed in the main text. We explore a series of low-scale generative modelling experiments where the goal is to retain generative modelling performance whilst improving the quality of the bridge itself (i.e. solving the SBP problem better).

Across our experiments, we use $\KL$ and let $\Gamma_0=\Gamma_T=\mathrm{Leb}$, which can be simplified to the forward-backwards KL objective used in DNF \citep{zhang2014applications}, see Appendix \ref{app:dnf}. We use the Adam optimiser \citep{kingma2014adam} trained on 50,000 samples and batches of size 5000 following  \citet{zhang2021diffusion}. For the generative modelling tasks we use 30 time steps and train for 100 epochs whilst for the double well we train all experiments for 17 epochs (early stopping via the validation set) and 60 discretisation steps.  Finally note we typically compare our approach with $\lambda>0$ to DNF ($\lambda=0$),  with DNF initialised at the reference process, which we call DNF (EM Init), see Appendix \ref{app:em_init} for further details.

\subsection{2D toy targets -- generative modelling}

Here we consider the suite of standard 2D toy targets for generative modelling explored in \citet{zhang2021diffusion} In contrast to \citet{zhang2021diffusion} we consider the SDE $\dd \mY_t = -\sigma^2 \mY_t \,\dd t + \sigma \sqrt{2} \,\dd \mW_t$  as the Schr{\"o}dinger prior across methods.
We parametrise DNF and our proposed approach with the same architectures for a fair comparison. Furthermore, we incorporate the drift of the above Schr\"odinger prior into DNF via parameterising the forward drift as in  (\ref{eq:Schr SDE}), partly motivated by Corollary \ref{cor:path space EM}.

In order to assess the quality of the bridge we consider three different error metrics. Firstly we estimate $\KL$ between the Schr{\"o}dinger prior and the learned forward process (i.e. $\mathbb{E}_{\fY \sim \ora{\P}^{\mu,a}} \left[ \tfrac{1}{2 \sigma^2}\int_0^T \Vert a_t - f_t\Vert^2(\fY_t) \, \d t \right]$). Secondly, we evaluate $\KL\!(\!\ora{\P}^{\mu,f + \sigma^2\nabla \!\phi},\! \!\ola{\P}^{\nu,f + \sigma^2\nabla \!\theta})$ to obtain a proxy error between the learned and target marginals. Finally, we estimate the cross entropy between $\ora{\P}^{\mu,a}_T$ and $\nu$ to assess how well the constraint at time $T$ is met.

In Table \ref{tab:toy} we observe that similar values of $\KL$ are attained across both approaches in the tree, sierpinski, and  checkerboard datatsets whilst achieving significantly lower values of the SBP loss across all training sets, and for tree, swirl and checkerboard validation datasets.  At the same time, we can see that the cross-entropy errors are effectively the same across both approaches. Overall we can conclude that on the empirical measures over which we train our approach, we obtain a much better fit for the target Schr\"odinger bridge, and on the validation results we can see that we generalise to 3/5 datasets in improving the bridge quality whilst preserving the marginals to a similar quality.

\subsection{Double well -- rare event}

In this task we consider the double well potential explored in \citep{vargasshro2021,hartmann2013characterization} where the Schr{\"o}dinger prior is specified via the following overdamped Langevin dynamics $\dd \mY_t = -\nabla_{\mY_t} U(\mY_t) \, \dd t + \sigma \,\dd \mW_t$. The potential $U(\vy)$ typically models a landscape for which it is difficult to transport $\mu$ into $\nu$.

This is a notably challenging task as we are trying to sample a rare event and as noted by \citet{vargas2021solving} many runs would result in collapsing into one path rather than bifurcating. In Figure \ref{fig:wells} we can observe how our proposed regularised approach (\ref{fig:pinwell}) is able to successfully transport particles across the well whilst respecting the potential, whilst both variants of DNF using the EM-Init for $\phi$ (\ref{fig:nopinnwell}) and random init (\ref{fig:badnopinwell}) fail to respect the prior as nicely and do not bifurcate, with the random init in particular sampling quite inconsistent trajectories. Finally for reference we train a DNF model with $f_t =0$ and $\phi$ (\ref{fig:noprior}) initialised at random to illustrate  the significance of the initialisation of $\phi$.

\subsubsection{Double well potential}

We used the following potential \citep{vargas2021solving}:
\begin{align}
    U\left({x \choose y}\right) = \frac{5}{2}(x^2-1)^2+y^2 +  \frac{1}{\delta}\exp\left(-\frac{  x^2 +y^2}{\delta}\right),
\end{align}
with $\delta=0.35$, furthermore, we used the boundary distributions:
\begin{align*}
    \mu \sim \calN\left(\begin{pmatrix} -1 \\
    0\\
    \end{pmatrix} ,\begin{pmatrix} 0.0125
 & 0 \\
    0 & 0.15\\
    \end{pmatrix}\right),\;\; \nu \sim 
    \calN\left(\begin{pmatrix} 1 \\
    0\\
    \end{pmatrix} ,\begin{pmatrix} 0.0125
 & 0 \\
    0 & 0.15\\
    \end{pmatrix}\right).
\end{align*}
The Schr\"odinger prior is given by:
\begin{align}
    \d \mY_t = -\nabla_{\mY_t} U(\mY_t) \, \dd t + \sigma \, \dd \mW_t,
\end{align}
with $\sigma=0.4$. The terminal time is $T=1$. Furthermore, we employ the same exponential discretisation scheme as in the generative modelling experiments.

\begin{table*}[h]
\adjustbox{max width=\textwidth}{
    \begin{tabular}{@{}llllllllll@{}}
    \toprule
    \multicolumn{1}{c}{\multirow{2}{*}{Target}} & \multicolumn{1}{c}{\multirow{2}{*}{Method}} & \multicolumn{2}{c}{KL}                              & \multicolumn{2}{c}{SBP Loss}                        & \multicolumn{2}{c}{PINN Loss}                       & \multicolumn{2}{c}{Cross Ent}                       \\ \cmidrule(l){3-10} 
    \multicolumn{1}{c}{}                        & \multicolumn{1}{c}{}                        & \multicolumn{1}{c}{Val} & \multicolumn{1}{c}{Train} & \multicolumn{1}{c}{Val} & \multicolumn{1}{c}{Train} & \multicolumn{1}{c}{Val} & \multicolumn{1}{c}{Train} & \multicolumn{1}{c}{Val} & \multicolumn{1}{c}{Train} \\ \midrule
    \multirow{2}{*}{tree}                       & $\lambda =0.5$                             &         1.67±0.02 &          1.40±0.01 &    \textbf{47.84±1.58} &     \textbf{42.31±1.52} &       \textbf{0.06±0.00} &        \textbf{0.05±0.00} &      2.87±0.01 &       2.80±0.01      \\
                                                & DNF (EM Init)                                      &          1.63±0.02 &          1.39±0.01 &    55.33±1.79 &     49.60±1.68 &      1.74±0.04 &       1.64±0.04 &      2.88±0.01 &       2.80±0.01              \\ \midrule
    \multirow{2}{*}{olympics}                   & $\lambda =0.5$                          &         2.95±0.06 &          0.12±0.01 &    39.30±0.90 &      \textbf{25.24±0.62} &      \textbf{0.26±0.01} &        \textbf{0.10±0.00} &      2.49±0.01 &       2.77±0.02       \\
                                                & DNF (EM Init)                                       &          \textbf{2.70±0.05} &          0.02±0.01 &    40.20±0.77 &     38.30±1.53 &      1.64±0.04 &       2.05±0.08 &       \textbf{2.54±0.01} &       2.77±0.02        \\ \midrule
    \multirow{2}{*}{sierpinski}                 & $\lambda =0.5$                          &         2.31±0.01 &          2.20±0.01 &    28.54±1.49 &     \textbf{26.67±0.90} &       \textbf{0.04±0.00} &        \textbf{0.03±0.00} &      2.82±0.01 &       2.83±0.01       \\
                                                & DNF (EM Init)    &         2.30±0.01 &          2.20±0.00 &    30.87±1.93 &     29.53±1.18 &      7.25±0.14 &       7.22±0.14 &      2.80±0.01 &       2.82±0.02             \\ \midrule
    \multirow{2}{*}{swirl}                      & $\lambda =0.5$                        &        15.67±0.29 &          1.95±0.03 &    \textbf{121.81±1.94} &      \textbf{40.24±1.74} &       \textbf{1.01±0.03} &        \textbf{0.14±0.00} &      2.97±0.01 &        2.69±0.02     \\
                                                & DNF (EM Init)   &         \textbf{13.77±0.38} &          1.92±0.04 &   151.67±3.68 &     67.55±1.86 &      5.89±0.15 &       2.63±0.08 &      2.95±0.02 &       2.74±0.03  \\ \midrule
    \multirow{2}{*}{checkerboard}               & $\lambda =0.5$  &         4.79±0.01 &          4.70±0.01 &    \textbf{34.47±0.80} &     \textbf{33.70±0.91} &      \textbf{0.03±0.00} &      \textbf{ 0.02±0.00} &      2.82±0.00 &       2.81±0.01      \\
                                                & DNF (EM Init)  &         4.78±0.01 &          4.70±0.02 &    39.76±0.83 &     39.20±1.10 &      3.66±0.07 &       3.68±0.06 &      2.81±0.01 &       2.81±0.02   \\ \bottomrule
    \end{tabular}
}\vspace{-0.2cm}
\centering
 \caption{Generative Modelling Results comparing DNF \citep{zhang2021diffusion} ($\lambda=0$) to our PINN regualirsed approach  with $\lambda=0.5$. We observe that PINN regularisation obtains similar KL and Cross entropy losses to DNF whilst achieving lower distances to the prior.}\label{tab:toy}
\end{table*}
\begin{figure*}
    \centering
    \begin{subfigure}[t]{0.24\textwidth}
        \centering
        \includegraphics[width=\linewidth,trim={3cm 0 2cm 2cm},clip]{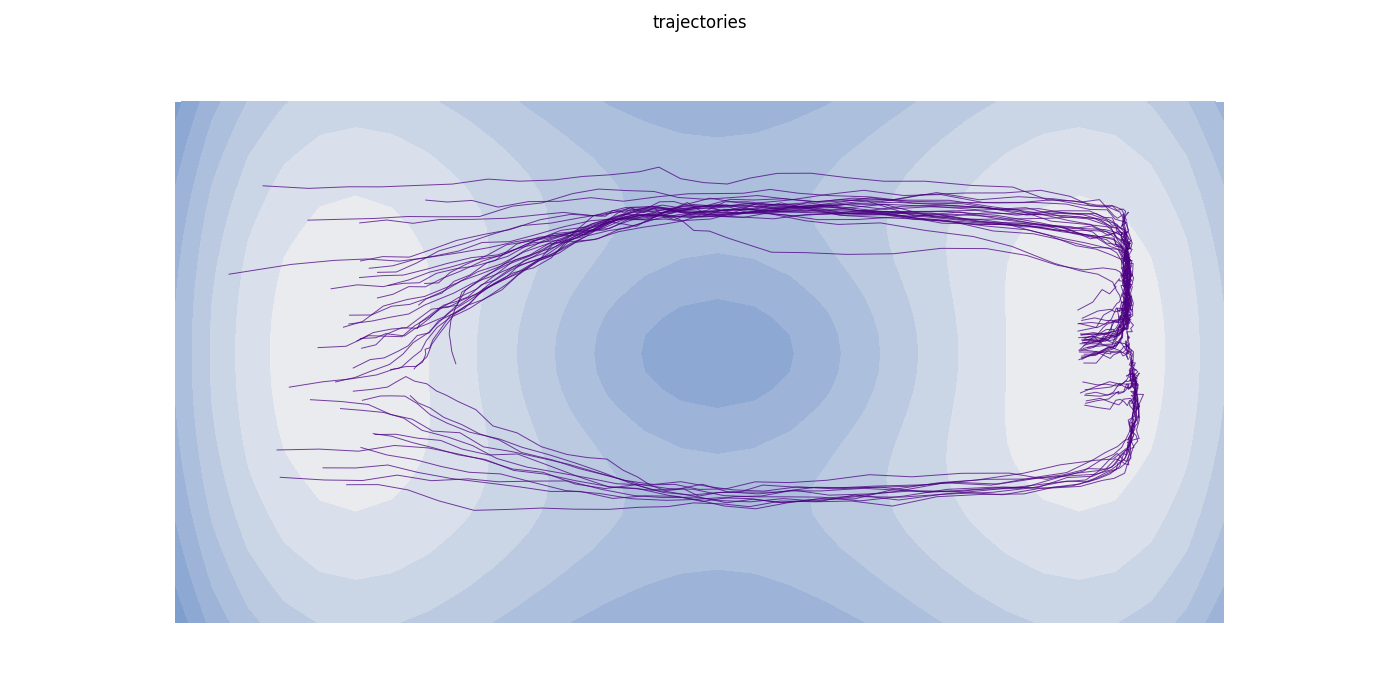} 
        \caption{$\lambda=2$} \label{fig:pinwell}
    \end{subfigure}
    \begin{subfigure}[t]{0.24\textwidth}
        \centering
        \includegraphics[width=\linewidth,trim={3cm 0 2cm 2cm},clip]{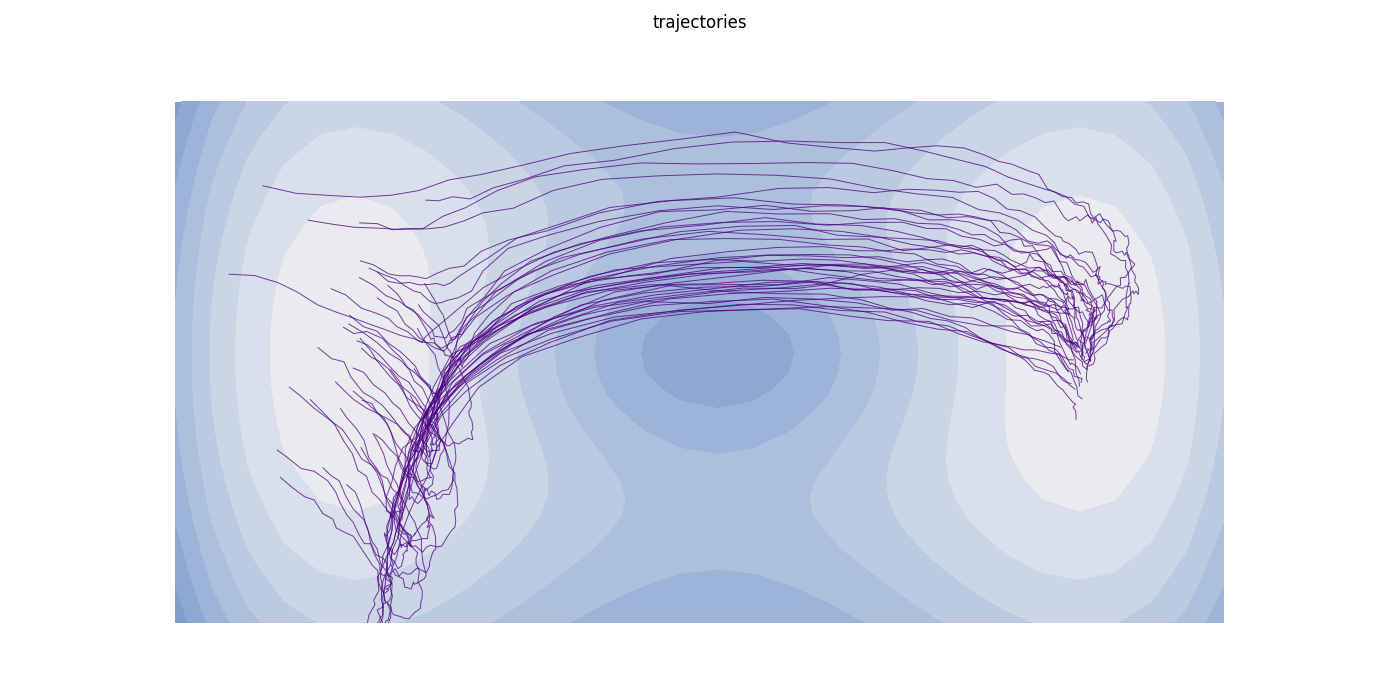} 
        \caption{$\lambda=0$ (EM Init)} \label{fig:nopinnwell}
    \end{subfigure}
    \begin{subfigure}[t]{0.24\textwidth}
    \centering
        \includegraphics[width=\linewidth,trim={3cm 0 2cm 2cm},clip]{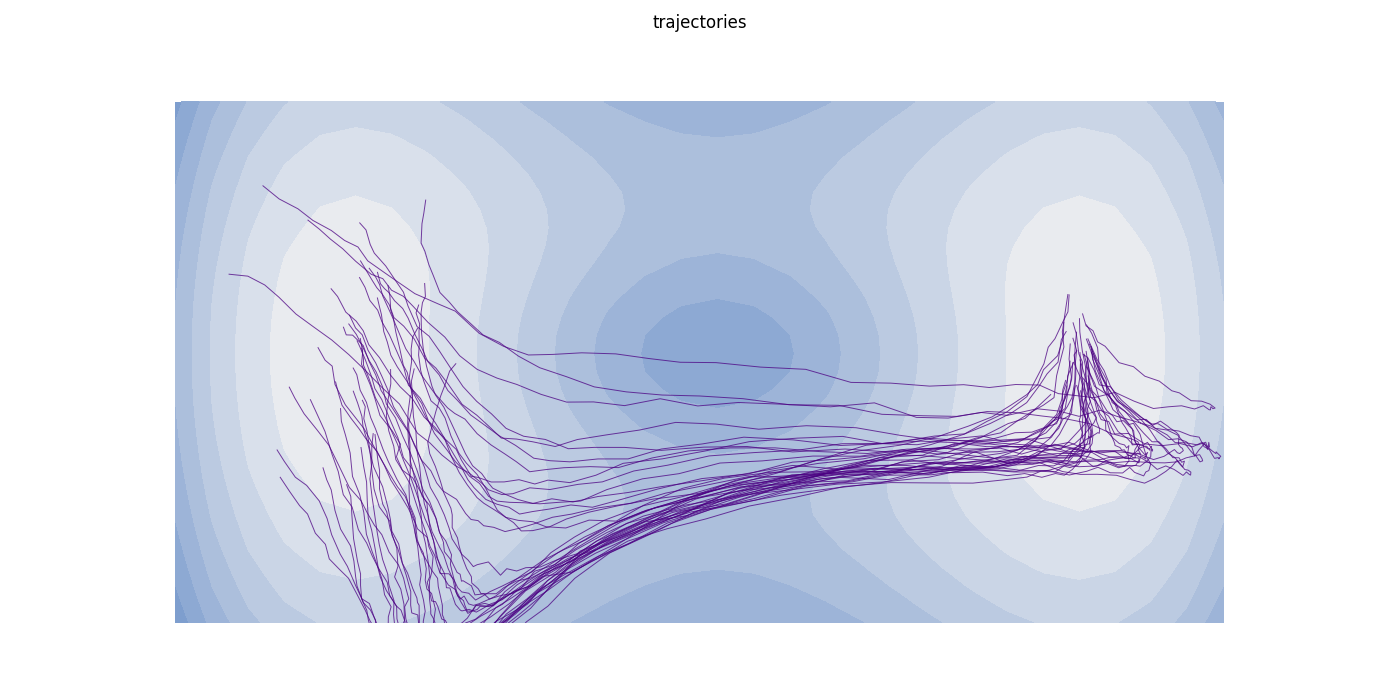} 
        \caption{$\lambda=0$ (Random Init)} \label{fig:badnopinwell}
    \end{subfigure}
    \begin{subfigure}[t]{0.24\textwidth}
    \centering
        \includegraphics[width=\linewidth,trim={3cm 0 2cm 2cm},clip]{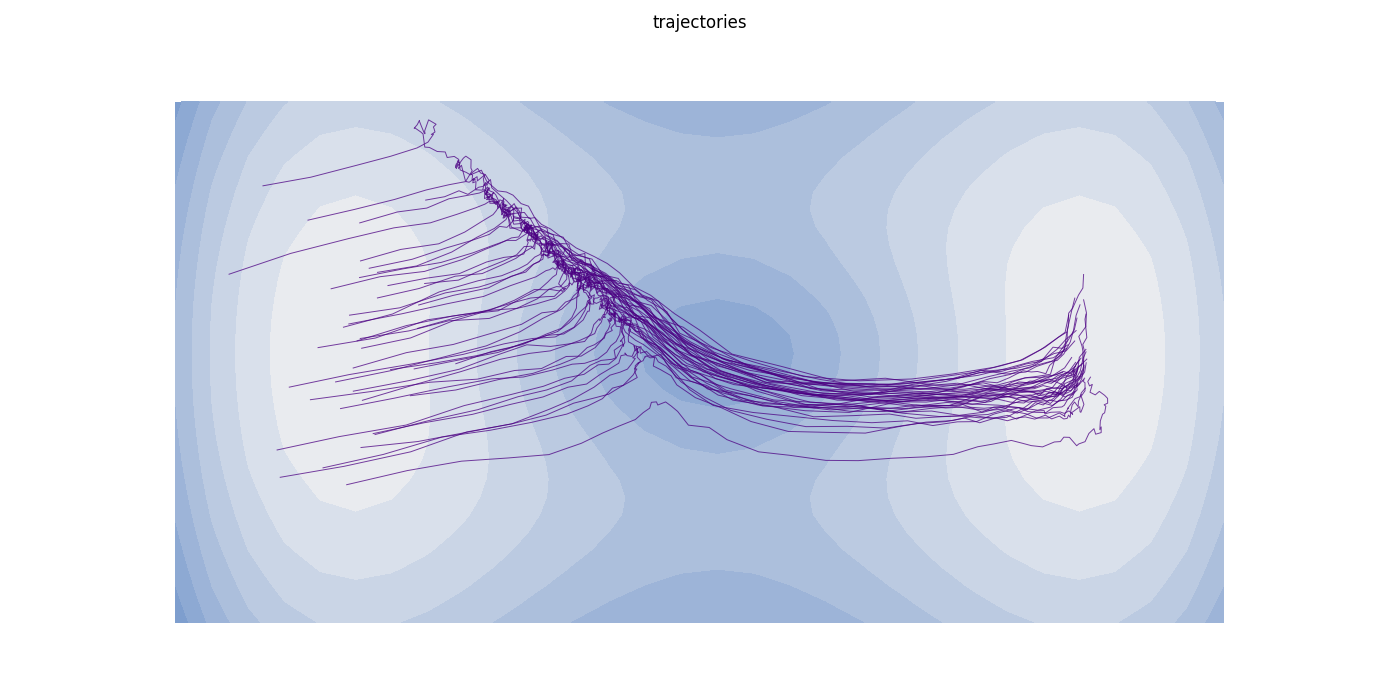} 
        \caption{DNF (No Prior)} \label{fig:noprior}
    \end{subfigure}\vspace{-0.25cm}
    \caption{ (\subref{fig:pinwell}) our proposed regularised objective,  (\subref{fig:nopinnwell}) $\lambda$ set to 0 but using clever EM motivated initialisation, (\subref{fig:badnopinwell}) $\lambda$ set to 0 with random initialisation of the forward drift, (\subref{fig:noprior}) for reference DNF with $f_t = 0$ (uninformative Schr\"odinger prior). \label{fig:wells}}
\end{figure*}

\begin{figure*}
    \centering
    \begin{subfigure}[t]{0.1515\textwidth}
        \centering
        \includegraphics[width=\linewidth,trim={25cm 0cm 2cm 2cm},clip]{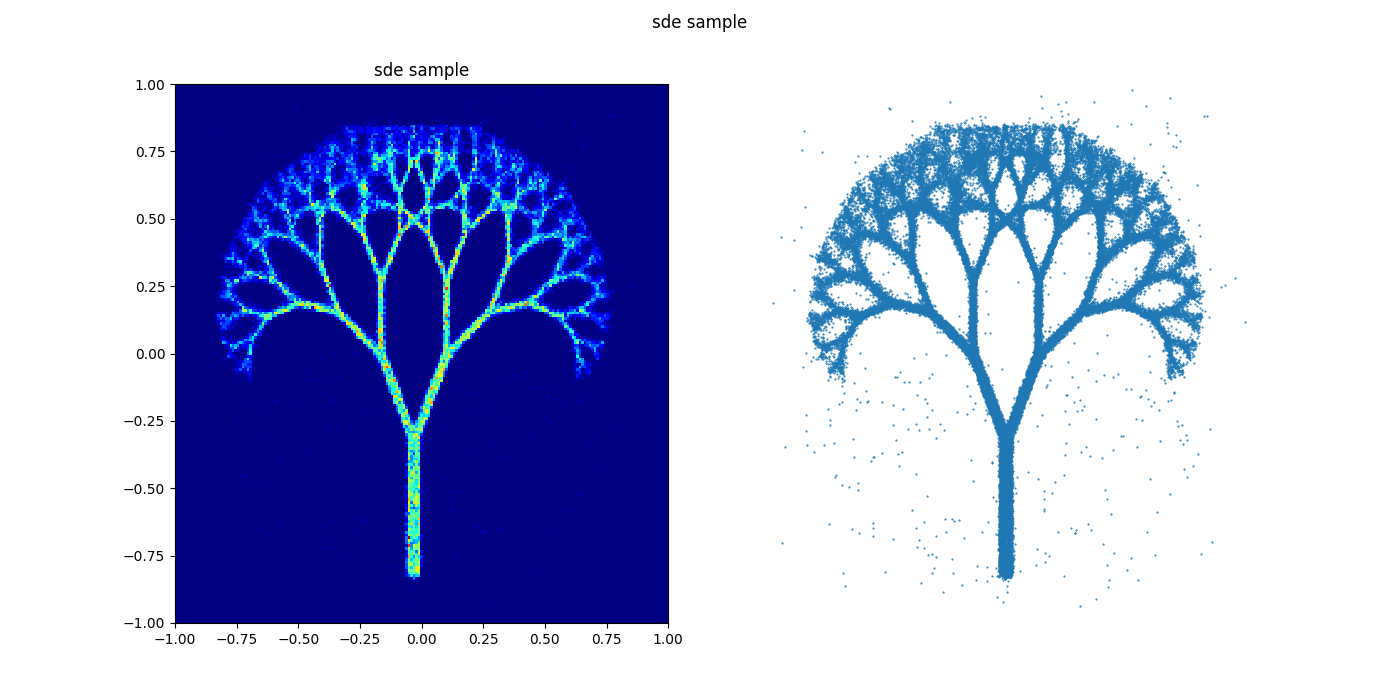} 
    \end{subfigure}\hspace{-0.5cm}
    \begin{subfigure}[t]{0.1515\textwidth}
        \centering
        \includegraphics[width=\linewidth,trim={25cm 0 2cm 2cm},clip]{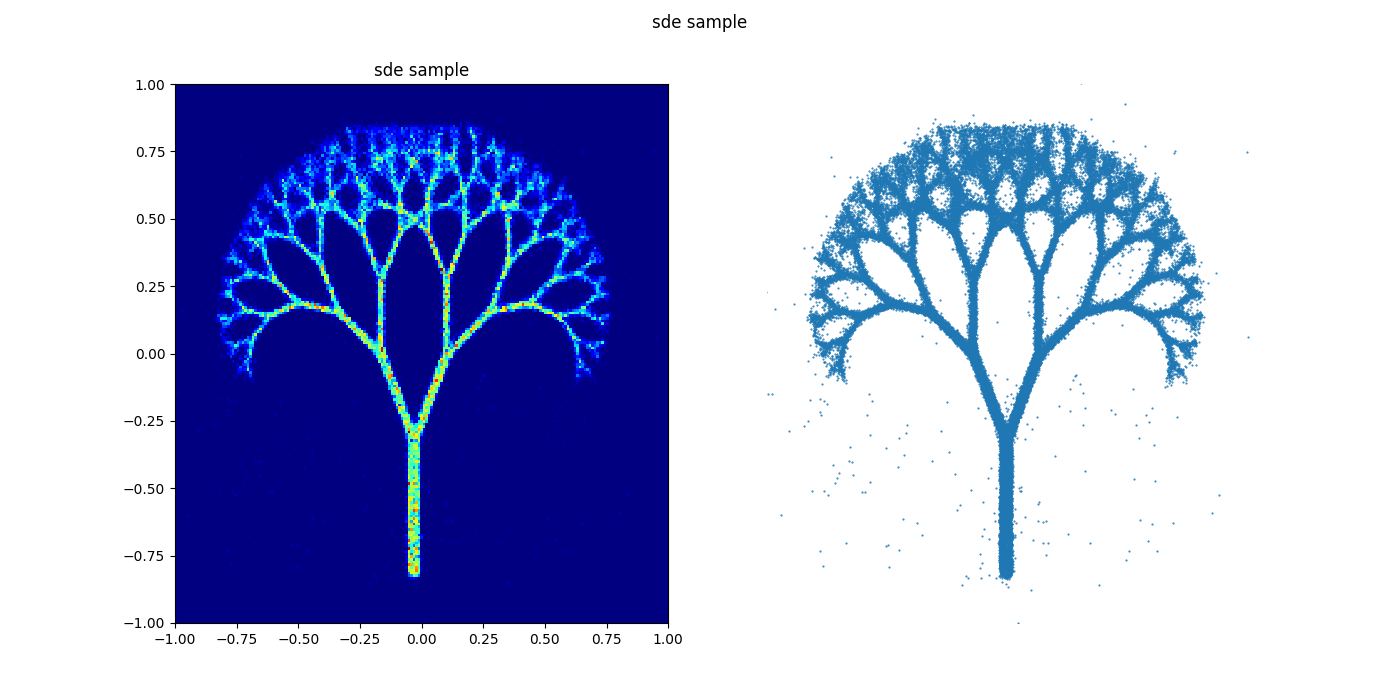} 
    \end{subfigure}
    \begin{subfigure}[t]{0.1515\textwidth}
    \centering
        \includegraphics[width=\linewidth,trim={25cm 0 2cm 2cm},clip]{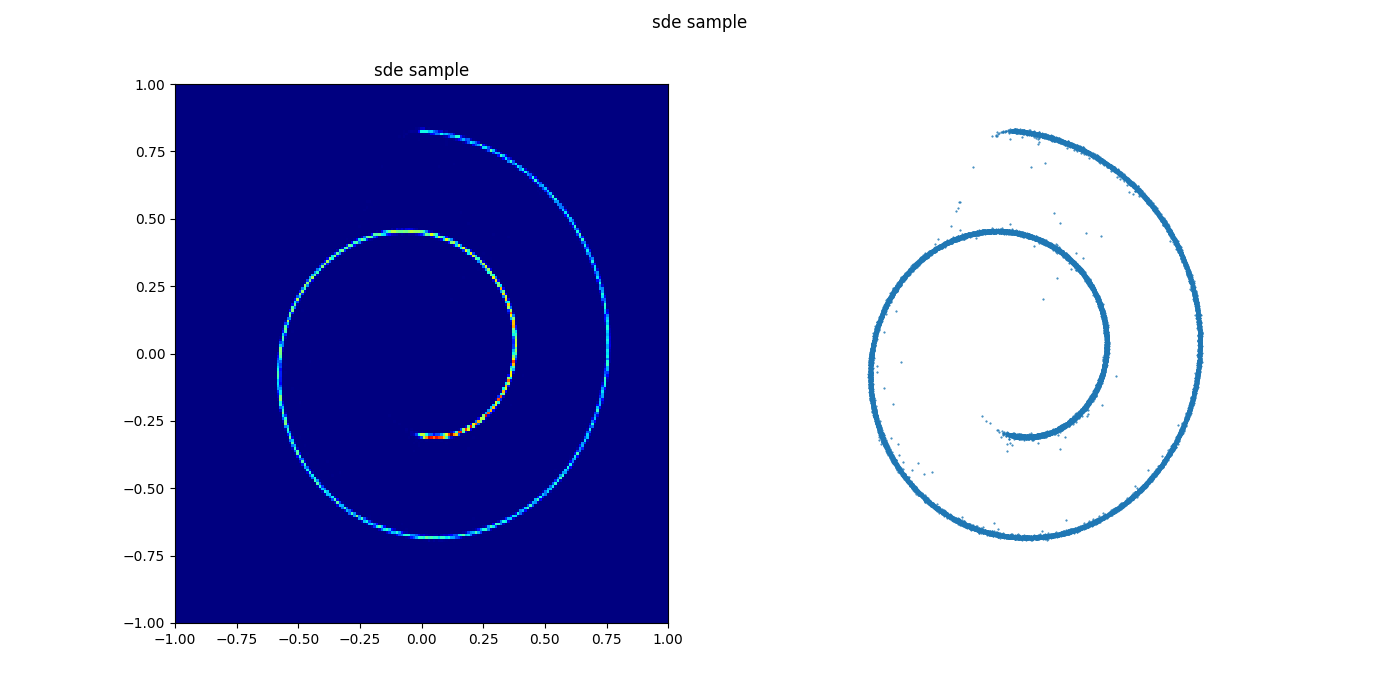} 
    \end{subfigure}\hspace{-0.6cm}
    \begin{subfigure}[t]{0.1515\textwidth}
    \centering
        \includegraphics[width=\linewidth,trim={25cm 0 2cm 2cm},clip]{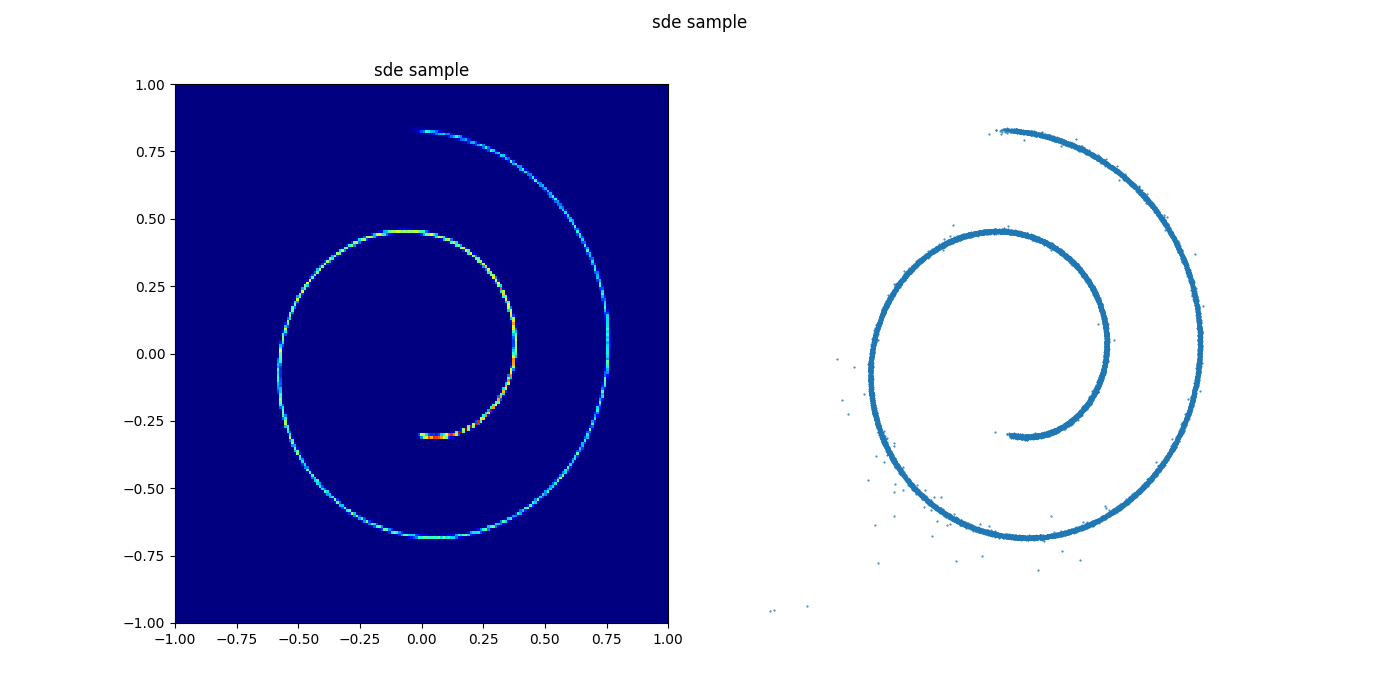} 
    \end{subfigure}
    \begin{subfigure}[t]{0.1515\textwidth}
    \centering
        \includegraphics[width=\linewidth,trim={25cm 0 2cm 2cm},clip]{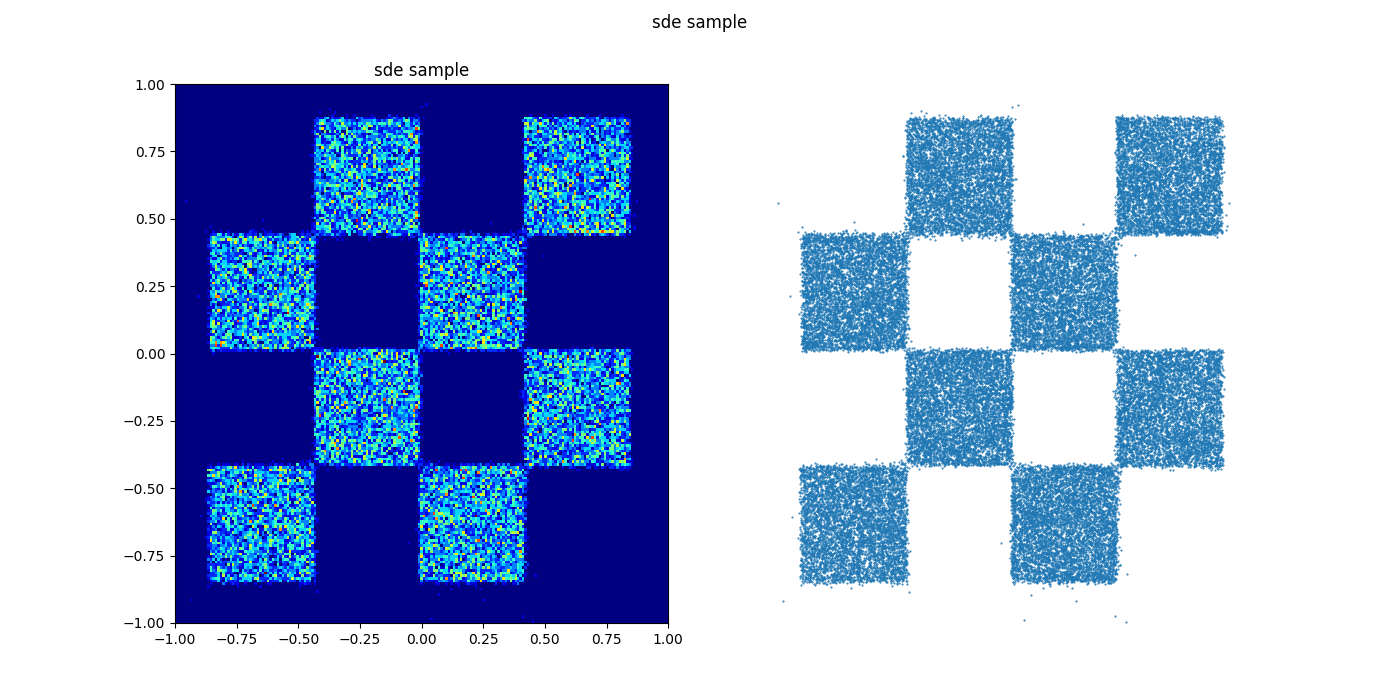} 
    \end{subfigure}\hspace{-0.5cm}
    \begin{subfigure}[t]{0.1515\textwidth}
    \centering
        \includegraphics[width=\linewidth,trim={25cm 0 2cm 2cm},clip]{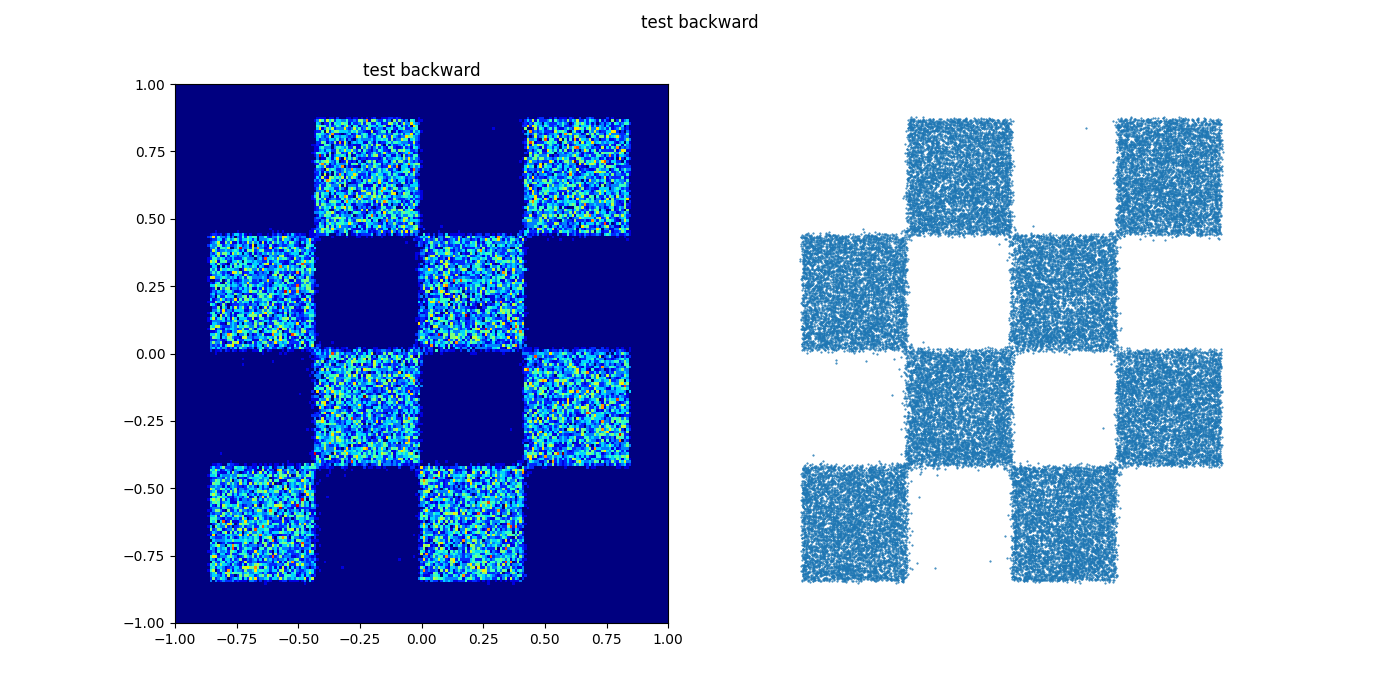} 
    \end{subfigure}
    
    \vspace{0.25cm}

        \begin{subfigure}[t]{0.1515\textwidth}
    \centering
        \includegraphics[width=\linewidth,trim={25cm 0 2cm 2cm},clip]{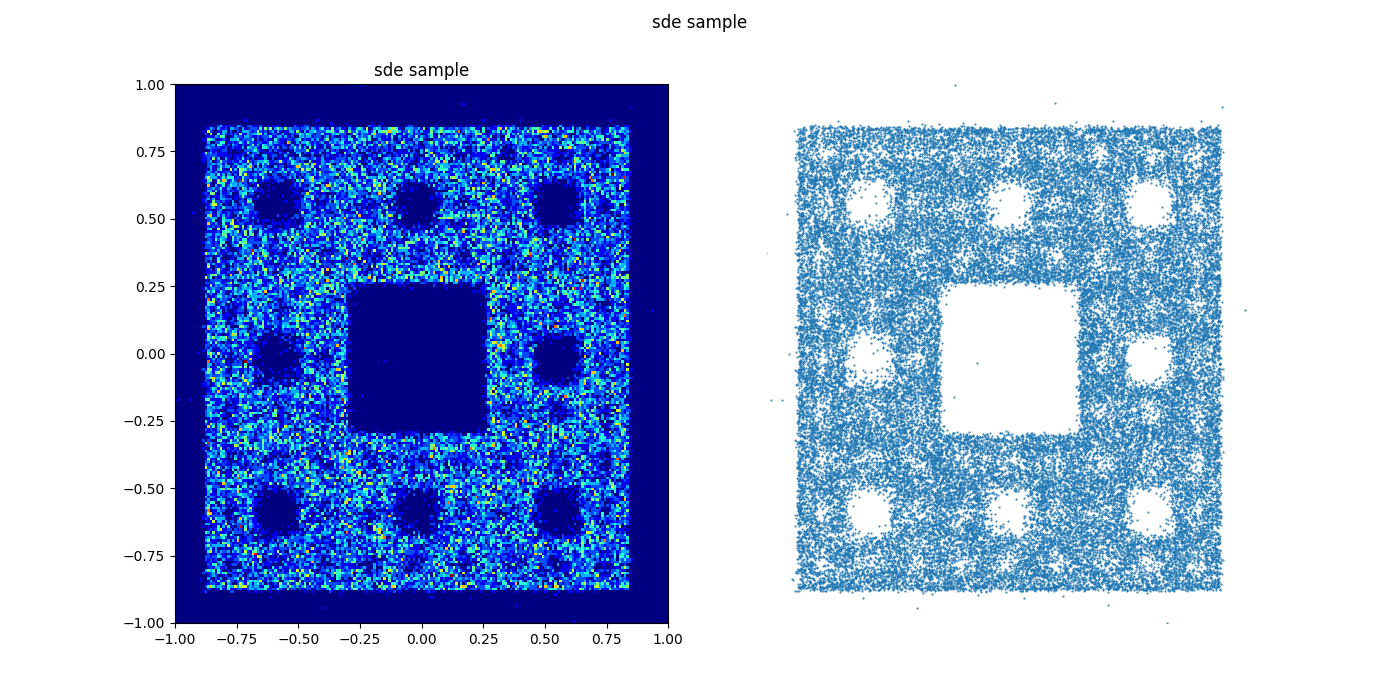} 
    \end{subfigure}\hspace{-0.5cm}
    \begin{subfigure}[t]{0.1515\textwidth}
    \centering
        \includegraphics[width=\linewidth,trim={25cm 0 2cm 2cm},clip]{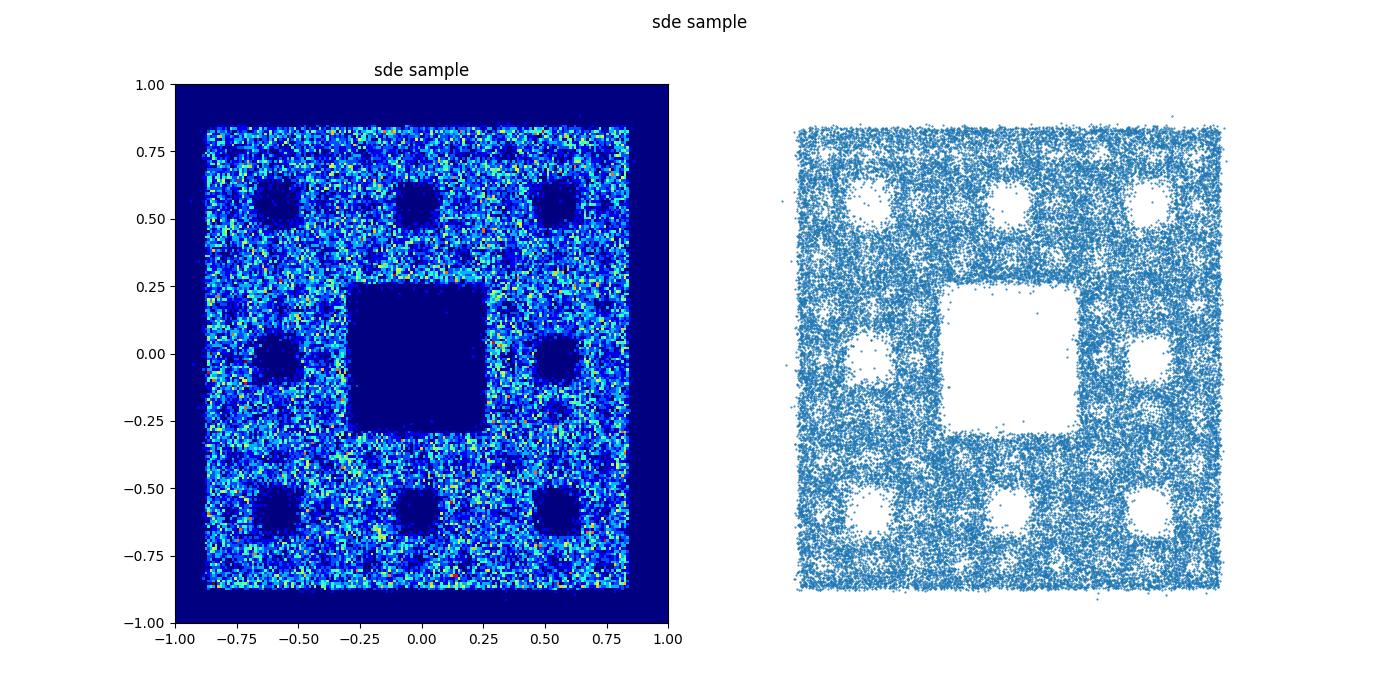} 
    \end{subfigure}
    \begin{subfigure}[t]{0.1515\textwidth}
    \centering
        \includegraphics[width=\linewidth,trim={25cm 0 2cm 2cm},clip]{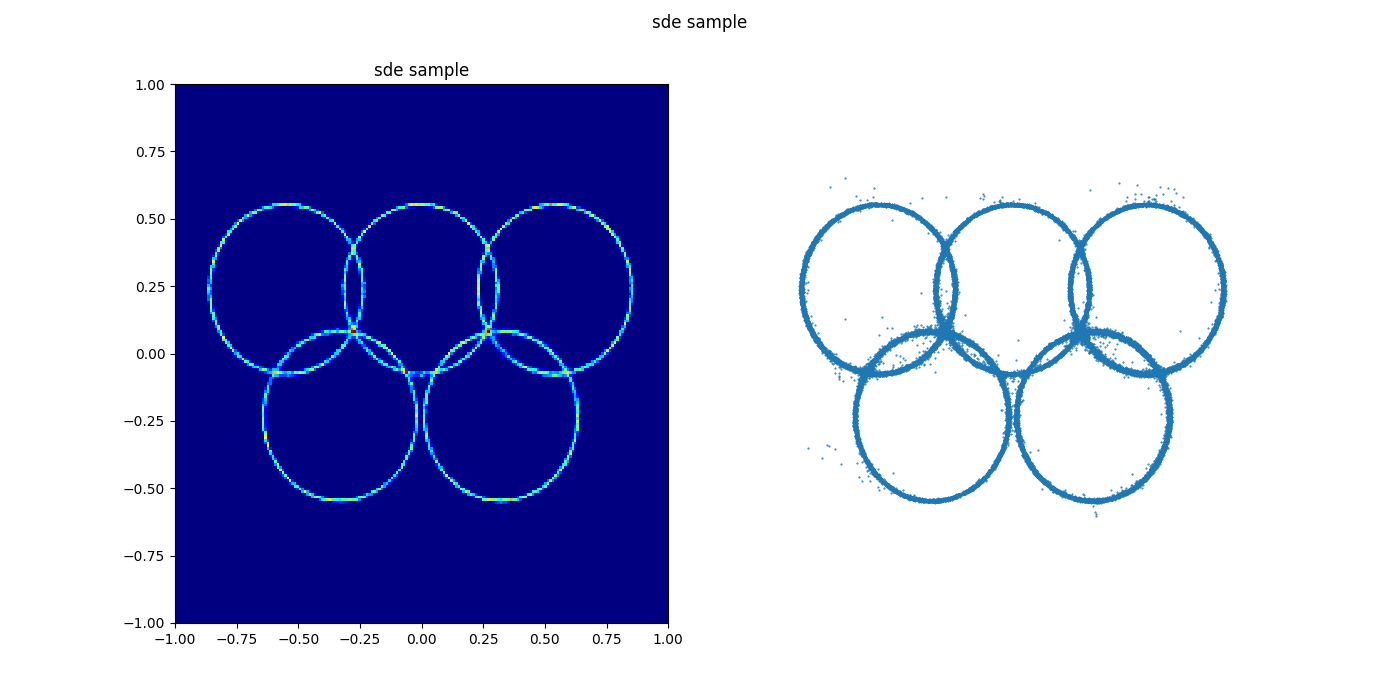} 
    \end{subfigure}\hspace{-0.5cm}
    \begin{subfigure}[t]{0.1515\textwidth}
    \centering
        \includegraphics[width=\linewidth,trim={25cm 0 2cm 2cm},clip]{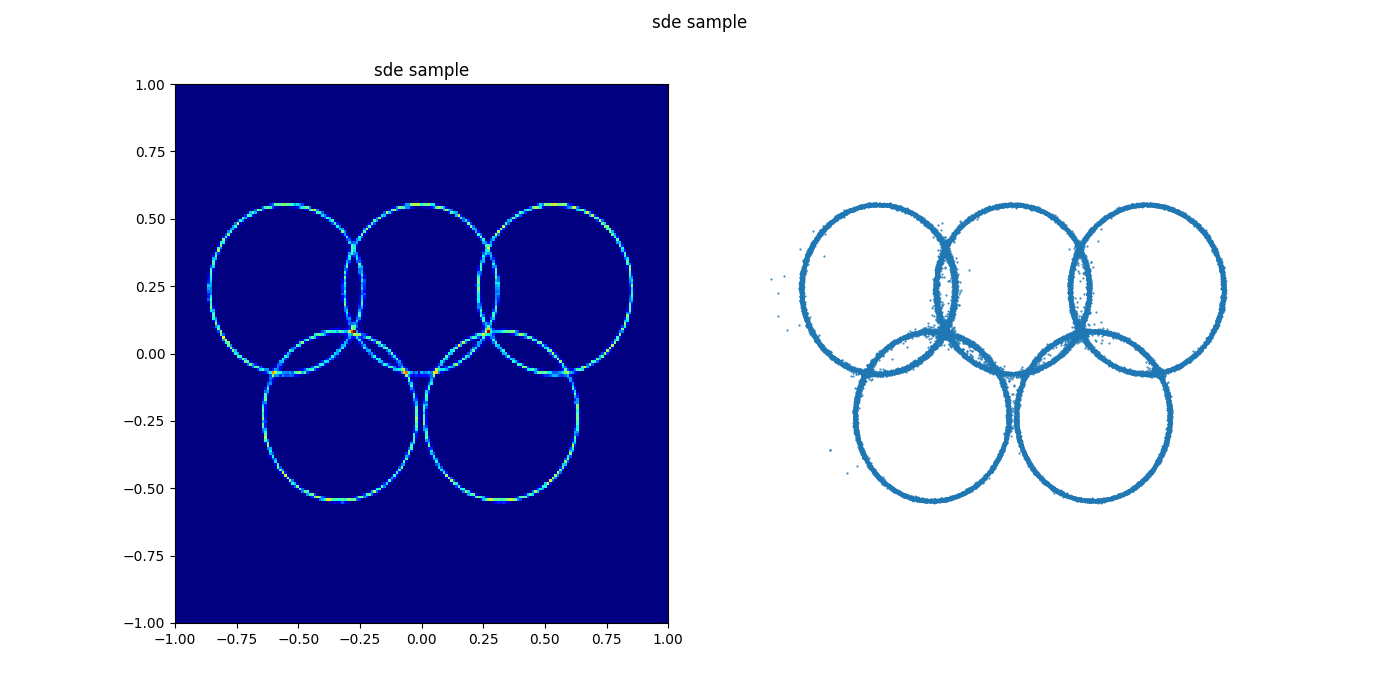} 
    \end{subfigure}
    \caption{Generated samples trained by our approach ($\lambda=0.5$) left and DNF ($\lambda=0$) right. Qualitatively we can observe that both learned models have similarly matched marginals.\label{fig:gen}}
\end{figure*}

\subsection{Implementation Details}
\label{app:impl}

\subsubsection{Neural network parameterisations}

Following \citet{zhang2021diffusion}  and the recent success in score generative modelling we choose the following parameterisations: 
\begin{subequations}
\begin{align}
    a_t(\vx) &= f_t(\vx)+  \sigma^2 \nabla \phi(t,\vx), \\
    b_t(\vx) &= f_t(\vx) + \sigma^2 \nabla \phi(t,\vx) - \sigma^2  s_{\theta}(t,\vx), \label{eq:score_net}
\end{align}
\end{subequations}
where $s_{\theta}$ is a score network \citep{song2020score,de2021diffusion,zhang2021diffusion} and $\phi(t,\vx)$ is a neural network potential. We adapt the architectures proposed in \citet{onken2021ot,koshizuka2022neural} to general activation functions. Note that these architectures allow for fast computation of $\Delta \phi$ comparable to that of Hutchinson's trace estimator \citep{grathwohl2018ffjord, hutchinson1989stochastic}. 

Finally, we remark that the parametrisation in  (\ref{eq:score_net}) allows us to learn the score of the learned SDE and thus seamlessly adapt our approach to using the probability flow ODE \citep{song2020score} at inference time.


\subsubsection{PINN Loss}

For the PINN loss across all tasks, we sample the trajectories from $\mY_{0:T}^{\phi} \sim \ora{\P}^{\mu,\nabla \phi}$ and thus employ the same discretisation as used in the KL loss. However, we detach the trajectories $\mY_{0:T}^{\mathrm{detach}(\phi)}$ before calculating the gradient updates in a similar fashion to \citet{nusken2021solving}.

\end{document}



\end{document}